\documentclass[11pt]{article}

\usepackage[margin=1in]{geometry}
\usepackage{amsmath}

\usepackage[utf8]{inputenc} % allow utf-8 input
\usepackage[T1]{fontenc}    % use 8-bit T1 fonts
\usepackage[hidelinks]{hyperref}       % hyperlinks
\usepackage{url}            % simple URL typesetting
\usepackage{booktabs}       % professional-quality tables
\usepackage{amsfonts}       % blackboard math symbols
\usepackage{nicefrac}       % compact symbols for 1/2, etc.
\usepackage{microtype}      % microtypography
\usepackage{xcolor}         % colors
\usepackage{notation}
\usetikzlibrary{calc}
\usepackage{titletoc}
\makeatletter
\newcommand{\andbreak}{\end{tabular}\hfil\linebreak[4]\hfil\begin{tabular}[t]{c}\ignorespaces}
\makeatother

\begin{document}
\title{Bayes-Consistent Losses for Adaptive Inference in Classification}
\date{}
\author{Shirley Xiaoqi Liu\thanks{Equal contribution. 
Contact emails:  \href{mailto:shirley.liu@stats.ox.ac.uk}{\texttt{shirley.liu@stats.ox.ac.uk}}, \href{mailto:florence.robert-regol@mail.mcgill.ca}{\texttt{florence.robert-regol@mail.mcgill.ca}},  \\ and \href{mailto:antonios.valkanas@mail.mcgill.ca}{\texttt{antonios.valkanas@mail.mcgill.ca}}.  
This work does not relate to A.\ Valkanas’ role at Amazon.} $^\dagger$
\and Florence Regol$^{*\ddagger}$
\and Antonios Valkanas$^{*\ddagger}$
%\\McGill University 
 \andbreak  Patrick Rebeschini\thanks{University of Oxford} 
\and Mark Coates\thanks{McGill University}
}

\maketitle

\begin{abstract}
Adaptive inference systems balance accuracy and inference cost by routing inputs to models with varying capabilities.   Although jointly training  models  and routers is empirically more sample-efficient than  training them separately,  principled  losses for joint training remain underexplored.
%, and  most existing approaches rely on heuristics. 
We address this gap in a classification setting  by  proposing cost-sensitive hinge   losses for jointly training   
classifiers with   single-stage or cascade  routers. Single-stage routers directly select a model per input, whereas cascade routers sequentially decide at each stage whether to predict with the current model or defer to the next stage, which uses a more  accurate but more costly model. For both settings, we prove  linear calibration bounds   establishing Bayes consistency: minimizing the hinge risk recovers the Bayes-optimal risk under the target cost-sensitive zero-one loss. 
Moreover, since routers can require more labeled data to train than  classifiers, we propose a  novel  post-training prediction-powered inference (PPI) procedure that  uses unlabeled data to refine a trained single-stage router. Experiments on synthetic and real datasets demonstrate that   the proposed hinge losses are effective for joint training and  that PPI  can correct routing errors when labeled data is scarce.
\end{abstract}

\section{Introduction}

Adaptive inference systems allocate  computation on a per-sample basis to balance accuracy and inference cost. For each sample, they decide which   one of several predictive models with different cost-accuracy profiles should be used for prediction. This framework covers  a broad range of methods, including  learning to defer~\citep{Mozannar2020}, model cascades~\citep{chen2024frugalgpt, zhang2024treacle, valkanas2025cpo} and dynamic neural networks (DNNs)~\citep{dnn_survey}, including early-exit systems~\citep{xin2021berxit}. We study adaptive inference in a classification setting that involves two components: an array of classifiers with varying cost-accuracy profiles, and a router module that decides, for each sample, which classifier makes the prediction.

A central challenge  is how to train such systems effectively. Many
approaches rely on \emph{separate training}, where the
classifiers are trained independently and routing is determined by a rule chosen separately, 
%with routing decisions based on heuristic rules 
such as   confidence-score thresholding~\citep{schuster2022confident, Jitkrittum2023}. 
This approach is often sample-inefficient: each  classifier is trained
on the full  sample  space, even though  the trained router  
may   send it    samples from only a small subset of the  space.
This mismatch may lead to suboptimal results~\citep{Jitkrittum2023} and hinder the training process~\citep{kubaty2025how}. In contrast, \emph{joint training} aligns   classifier and  router objectives, encouraging each classifier to specialize in  samples routed to it. Empirically, joint training substantially improves prediction accuracy at a fixed sample size, including with losses designed to align training with test-time routing
%even with heuristic losses
~\citep{han2022Weighted, boostedNet, regol2024jointlylearned, kubaty2025how}. These results motivate the design of 
%principled 
  joint-training losses with consistency guarantees, which are currently scarce.

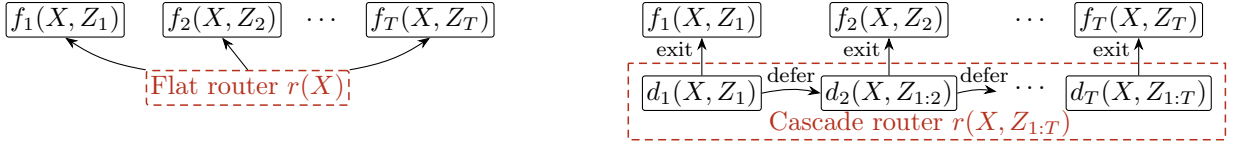
\begin{figure*}[t]
\centering
\begin{minipage}[t]{0.42\linewidth}
\centering
\begin{tikzpicture}[
    >=Stealth,
    node distance=0.5cm,
    inner sep=1.5pt,
    minimum height=0.38cm,
    minimum width=0.9cm,
    font=\small,
    baseline=(F1.north)
]

\node (F1) [draw, rectangle, rounded corners=1pt] {$f_1(X, Z_1)$};
\node (F2) [draw, rectangle, rounded corners=1pt, right=0.5cm of F1] {$f_2(X,Z_2)$};
\node (dots) [right=0.1cm of F2] {$\dots$};
\node (FT) [draw, rectangle, rounded corners=1pt, right=0.1cm of dots] {$f_T(X,Z_T)$};

\node (R) [draw, rectangle, rounded corners=1pt, densely dashed,
           align=center, line width=0.02cm, color=BrickRed]
      at ($(F1)!0.5!(FT) + (0,-0.9)$)
      {Flat router $r(X)$};

\draw[->, bend left=12]  (R.north west) to (F1.south);
\draw[->]                (R.north)      to (F2.south);
\draw[->, bend right=12] (R.north east) to (FT.south);

\end{tikzpicture}
\end{minipage}
\hfill
\begin{minipage}[t]{0.48
\linewidth}
\centering
\begin{tikzpicture}[
    >=Stealth,
    node distance=0.5cm,
    inner sep=1.5pt,
    font=\small,
    minimum height=0.38cm,
    minimum width=0.9cm,
    baseline=(cF1.north)
]

\node (cF1) [draw, rectangle, rounded corners=1pt] {$f_1(X, Z_1)$};
\node (D1)  [draw, rectangle, rounded corners=1pt, below=of cF1] {$d_1(X,Z_1)$};

\node (cF2) [draw, rectangle, rounded corners=1pt, right=0.9cm of cF1] {$f_2(X,Z_2)$};
\node (D2)  [draw, rectangle, rounded corners=1pt, below=of cF2] {$d_2(X,Z_{1:2})$};

\node (cdots) [right=0.65cm of cF2] {$\dots$};
\node (ddots) [below=of cdots] {$\dots$};

\node (cFT) [draw, rectangle, rounded corners=1pt, right=0.1cm of cdots] {$f_T(X,Z_T)$};
\node (DT)  [draw, rectangle, rounded corners=1pt, below=of cFT] {$d_T(X,Z_{1:T})$};

\draw[->] (D1.north) -- node[left, xshift=0.1cm, yshift=0.1cm] {\scriptsize exit} (cF1.south);
\draw[->] (D2.north) -- node[left, xshift=0.1cm, yshift=0.1cm] {\scriptsize exit} (cF2.south);
\draw[->] (DT.north) -- node[left, xshift=0.1cm, yshift=0.1cm] {\scriptsize exit} (cFT.south);

\draw[->, bend left=12] (D1.east) to node[above] {\scriptsize defer} (D2.west);
\draw[->, bend left=12] (D2.east) to node[above, xshift=0.1cm] {\scriptsize defer} (3.9,-0.95);

\draw[densely dashed, line width=0.02cm, color=BrickRed]
      ($(D1.north west)+(-0.18,0.16)$) rectangle
      ($(DT.south east)+(0.18,-0.38)$);

\node[color=BrickRed, align=center]
      at ($(D1)!0.5!(DT) + (0,-0.4)$)
      {\small Cascade router $r(X,Z_{1:T})$};
\end{tikzpicture}
\end{minipage}
\vspace{-0.08cm}
\caption{Illustration of flat (left) and cascade (right) routing models. A flat router selects a single classifier $f_t$ based on $X$, while a cascade router sequentially invokes classifiers $f_1,\dots, f_T$ until the first exit decision is made based on $(X,Z_{1:t})$, or until the last classifier $f_T$ is reached. 
}

\label{fig:architecture}
\end{figure*}

We address this gap in adaptive inference for classification under 
two common routing models, as illustrated in
Fig.~\ref{fig:architecture}, considering $T$ classifiers. Each sample has a base feature $X$, available from the start, and additional features $Z_1, \dots, Z_T$, where $Z_t$ is revealed once the $t$-th classifier is reached. In the single-stage (flat) model,   the router directly selects a
classifier $t$  based on   $X$ alone, and the classifier then predicts from $(X, Z_t)$. This model covers certain configurations
of elastic networks~\citep{devvrit2023matformer}, learning to defer to
multiple experts~\citep{Verma2023}, and width-adaptive
DNNs~\citep{dnn_survey}. 
In the multi-stage (cascade)
  model,  classifiers are invoked sequentially:  at each stage $t$, the router uses the accumulated features $(X, Z_{1:t})$ to either exit with the $t$-th  classifier's prediction or defer to the next stage, which uses  a more
powerful classifier. We assume that each classifier $t$ remains lightweight by predicting from $(X, Z_t)$ alone. 
This model covers
cascades~\citep{chen2024frugalgpt, zhang2024treacle, valkanas2025cpo}, early-exit 
systems~\citep{xin2021berxit},  and, more generally, depth-adaptive DNNs~\citep{dnn_survey}.

In our LLM experiments, for example, $Z_1, \dots, Z_T$ correspond to embeddings of the answers produced by pretrained LLMs of increasing size, which serve as inputs to the classifiers
(see Appendix~\ref{app:discussion_Z} for further justification of this modeling choice). 
Accessing $Z_t$ incurs a known cost  $\kappa_t$, e.g., the cost of running the $t$-th LLM. Such systems are typically designed so that  $\kappa_t$ varies with $t$ and is non-decreasing in $t$.  Due to the heterogeneous costs, 
jointly training the classifiers and router   requires a \emph{cost-sensitive} loss that balances the classification errors, measured by the zero-one  loss, against    the costs. We propose two such losses,~one~for~each~routing~model.

\subsection{Main Contributions}

\begin{itemize}[leftmargin=*]
    \item   
    We
propose  hinge-based    surrogates of the cost-sensitive zero-one loss for jointly training classifiers
and routers in both flat and cascade   models. In both cases, we prove   linear   calibration bounds (Theorems~\ref{thm:calibration} and \ref{thm:calibration_casc}) that control the   excess  target risk by the excess hinge risk,  thereby  establishing  Bayes consistency of the surrogate losses. This linear rate is sharp~\citep{Frongillo2021}. These results place joint training on a principled theoretical foundation. We validate  our theory  and demonstrate its practical utility through  experiments on synthetic and real datasets.

\item A key technical contribution is the calibration bound for the  cascade hinge loss. It is challenging to establish  because the exit stage  is a \emph{stopping time} induced by a sequence of   adaptive deferral decisions. We address this via a  dynamic programming approach~\citep{bertsekas2005dynamic,
  shiryaev2008optimal}, applying  backward induction to \emph{recursively} establish stage-wise calibration through a novel risk decomposition that separates   excess risks into  contributions from   suboptimal deferral rules  and suboptimal classifiers. A matching lower bound (Proposition~\ref{prop:LB}) shows that the resulting calibration constant is tight.

\item Since routers can require more labeled data to train than classifiers, we also propose a novel post-training prediction-powered inference (PPI) procedure that leverages \emph{unlabeled} data to refine a trained flat router.  At each input,   PPI   recovers the oracle routing decision  with high probability, whenever the width of its confidence interval on risk improvement is narrower than the gap between the oracle and the second-best routing decisions (Theorem~\ref{thm:PPI}). This constitutes the first integration of PPI into adaptive inference.

\end{itemize}

\subsection{Related Work}

Learning to defer allows a model to either predict with a learned classifier or   defer to one or more experts with heterogeneous  capabilities and costs \citep{Madras2018, Keswani2021, Wilder2021}. Here, \emph{only} the  classifier and the deferral rule  are learned; the experts are fixed.  
Building on   foundational works on Bayes consistency \citep{Bartlett2006, Tewari2007, steinwart2007,zhang2004statistical}, a line of  works develops Bayes-consistent surrogate losses with a single expert~\citep{Mozannar2020, Verma2022a} or with multiple experts~\citep{Verma2023}. 
For the single-expert setting, later works establish realizable $\mathcal{H}$-consistency, a  guarantee relative to a restricted hypothesis class~\citep{Mozannar2023, mao2024a}, where~\citet{mao2024a} further identify surrogates that are simultaneously realizable $\mathcal{H}$-consistent, Bayes-consistent, and admit $\mathcal{H}$-consistency bounds~\citep{Awasthi2022}. Closer to our multi-model setting, \citet{mao2024b} and \citet{mao2023} prove $\mathcal{H}$-consistency bounds in the   multiple-expert setting. \citet{mao2024b} consider the standard setting in which the classifier and the deferral rule are learned jointly, while~\citet{mao2023} consider a two-stage setting where the classifier is first trained then fixed, and the deferral rule is learned in a second stage. Finally,~\citet{mao2025mastering} establish realizable $\mathcal{H}$-consistency, $\mathcal{H}$-consistency bounds, and Bayes consistency for learning to defer to multiple experts, in both the standard and two-stage settings.

In contrast to learning to defer, which learns a single classifier,~\citet{regol2025acquisition} and~\citet{pugnana2025} recently studied a setting where \emph{all} classifiers involved are  learned:   they study joint training of \emph{two} classifiers  with a deferral rule (or router), 
and their analysis is restricted to this setting. Since adaptive inference systems often involve more than two models~\citep{chen2024frugalgpt, zhang2024treacle, valkanas2025cpo}, we address this limitation by developing principled losses for jointly training \emph{multiple} classifiers with the router under both  flat and cascade models. As our work is primarily theoretical, we defer discussion of the rich literature on practical architectures and empirical performance in adaptive inference to  Appendix~\ref{app:review}.

\subsection{Notation} 
We use $x_i=[x]_i$ to denote the $i$-th entry of vector $x$. We write $[n]\coloneqq \{1,2, \dots, n\}$, and for $i,j\in\mathbb{N}$ with $i\le j$,  we use the shorthand $a_{i:j}\coloneqq \{a_i, a_{i+1},\dots,  a_j\}$ for sequences of variables or functions. 
We denote the $n$-dimensional probability simplex by $\Delta^{n-1}$. 
We use uppercase letters for random variables, lowercase for their realizations (e.g., $X$ and $x$), and calligraphic letters to denote spaces of realizations (e.g., $x\in \mc{X}$). 

\section{Multi-Model Classification} \label{sec:setup}
 We study a multi-model classification problem, where the goal is to learn a tuple $\pi=(f_{1:T},r)$, consisting of  $T\geq 2$ classifiers $f_{1:T}$ and a router $r$.  Each classifier $f_t$ takes the   form:
\begin{align}
f_t : \mathcal X\times\mathcal Z_t \to [K], \quad \text{for}\quad t\in[T]\,,\label{eq:f_ell_def}
\end{align}
where  $\mc{X}\subset \R^p$ denotes the space of the  base features (e.g., a multiple-choice question  prompt),   $[K]$ is the label space for $K$  classes (e.g., the   answer choices), and   $ \mathcal Z_t\subset \R^{p_t}$ denotes the space of the  additional feature  (e.g., LLM embeddings) available at $f_t$.  
Accessing $Z_t$ incurs a  known nonnegative
cost $\kappa_t$, with $\kappa_t\!\ge\! \kappa_{t-1}$, reflecting that later classifiers are more accurate but more costly.\footnote{Our results can be extended to  cost functions depending on $X$,  provided the cost function is known, e.g., per-token cost in an LLM where  prompt length   varies.}  
The routing rule takes the form
\begin{align}
    r: \mc{H}\to [T] , \label{eq:def_r}
\end{align}
which maps the   history information  $\mc{H} \subseteq \mc{X}\times \prod_{t=1}^T \mc{Z}_t $   available to the router to a classifier index.\footnote{The cost of executing $r$, if non-negligible, can   be added  to the costs.} The space $\mc{H}$ varies with the router type, as we specify in Section~\ref{sec:two_routers_def}.  We denote  by $\Pi$ the class of all   multi-model classifiers $\pi=(f_{1:T}, r)$ satisfying \eqref{eq:f_ell_def} and \eqref{eq:def_r}.

For a given  $h\in\mc{H}$, the prediction $\hat{y}$ of  $\pi=(f_{1:T}, r)$ is   the output of the classifier selected by the router:
\begin{align}
    \hat{y} = f_{r(h)} (x, z_{r(h)})\,. 
\end{align}
\paragraph{Training data.} To train $\pi\in \Pi$, we use  $n$  training samples generated i.i.d.\ according to
\begin{align}\label{eq:def_P}
X^{(i)},  Z_{1:T}^{(i)}, Y^{(i)}\stackrel{\text{i.i.d.}}{\sim} P\quad \text{for} \quad i\in[n],
\end{align}
where $X^{(i)}\!\in\!\mc{X}$ and $Z_{t}^{(i)}\!\in\! \mc{Z}_t$ for $t\!\in\![T]$ are     features,   $Y^{(i)}\!\in\![K]$  is the  label, and $P$ is the joint distribution over $(X^{(i)}, Z_{1:T}^{(i)}, Y^{(i)})$.
We drop the superscript $^{(i)}$ for a generic sample.  In some cascade  applications, it is natural to assume the conditional independence \begin{align}\label{eq:conditional_indep}
        Y\perp Z_{1:t-1}\mid (X, Z_t), 
        %\;  Z_{t+1:T}\perp Z_{1:t-1}\mid (X, Z_t),
    \end{align}
    where    $(X,Z_t)$   subsumes all 
    information in  $Z_{1:t-1}$ relevant to predicting $Y$. We now formalize two routing models and their corresponding  multi-model classifiers.

\subsection{Single-Stage and Cascade Routing Models }\label{sec:two_routers_def}
The  single-stage router directly selects one of the $T$ classifiers using \emph{only} the base feature, formalized as:
\begin{definition}[Single-stage (flat) router]\label{def:flat_H}
With history information  $\mc{H}=\mc{X}$, a flat router takes  the form
\begin{align}
    r: \mc{X}\to [T].\label{eq:flat_H}
\end{align}
\end{definition}
 The multi-stage cascade router  sequentially decides, at each stage $t$, whether to exit with classifier $f_t$ or defer to the  more powerful classifier $f_{t+1}$ after collecting additional features, as formalized below: 
 
\begin{definition}[Cascade router]\label{def:cascade_H}
With history information $\mc{H}_t=\mc{X}\times \prod_{\tau=1}^t \mc{Z}_\tau$ available at stage  $t\in[T]$,   
  a cascade router 
consists of stage-wise deferral rules
\begin{align}\label{eq:d_ell_def}
&d_t: \mc{H}_t \to  \{0,1\}\quad \text{for}\quad t\in [T-1],
\end{align}
where  $d_t(h_t)=0$ indicates exit at stage  $t$ and  $d_t(h_t)=1$  indicates deferral to the next stage. We fix $d_T\equiv 0$, so the cascade always terminates at the final stage $T$. Writing $\mc{H}=\mc{H}_T$ for the full history, the \emph{stopping time} induced by the deferral rules is the cascade routing rule
\begin{align}\label{eq:def_r_casc}
    r(h)\in \min\{t\in [T]: d_t(h_t)=0\}.
\end{align}
\end{definition}
Restricting the class $\Pi$ of multi-model classifiers  to flat or cascade routers (Definitions~\ref{def:flat_H} or \ref{def:cascade_H}) yields the hypothesis classes $\Pi_\Flat\!\subset\! \Pi$  and  $\Pi_\casc\!\subset\! \Pi$, respectively, which are the focus of this work.

\subsection{The Target Cost-Sensitive Zero-One Loss}\label{sec:loss_func}
To capture the tradeoff between prediction error and   cost, we follow prior work on cost-sensitive classification and learning to defer~\citep{cao2022generalizing, mao2024b} and augment the standard zero-one loss with the costs. This yields  the \emph{$\mathit{0}$\!-$\mathit{1}$\!-$\mathit{c}$ loss}, denoted by $\ell: \mc{X}\times \prod_{t=1}^T\mc{Z}_t\times [K]\times \Pi\times \R^T_{\ge0}\to \R_{\ge0}$. For a multi-model classifier $\pi=(f_{1:T}, r)\in \Pi$   and a  cost vector $c\in \R^T_{\ge0}$ associated with $f_{1:T}$, we define
\begin{align}
\ell(x,z_{1:T},y;    \pi,c)
=
\Ind\!\big\{f_{r(h)}(x,z_{r(h)}) \neq y\big\} + c_{r(h)} \,.
\label{eq:multi_01c}
\end{align}
That is, an incorrect prediction incurs a unit loss, while exiting  at the $t$-th classifier additionally incurs a cost~$c_t$. In  \eqref{eq:multi_01c}, the  semicolon separates the random data arguments from the fixed multi-model classifier $\pi$ and cost vector $c$. The corresponding  $0$-$1$-$c$ population risk is then
\begin{align}\label{eq:def_R(f;c)}
\mc{R}(\pi)\equiv
\mc{R}(\pi; c)
&:=
\ex\big[\ell(X,Z_{1:T},Y;\pi, c)\big].
\end{align}
The form of $c$ differs between   the two routing models.  For the single-stage router,   $c=[\kappa_1, \kappa_2, \dots, \kappa_T]^\top$, where recall $\kappa_t$ denotes the cost of executing only the $t$-th classifier. In a cascade, since the classifiers are invoked sequentially, costs accumulate along the execution path, yielding $c=\big[\kappa_1, \kappa_1+\kappa_2, \dots, \sum_{t=1}^T\kappa_t\big]^\top$.

\section{Main Results}
\label{sec:main_results}
This section develops surrogate losses for the target $0$-$1$-$c$ loss  \eqref{eq:multi_01c}, one for the flat model and one for the cascade model. Throughout, we use a tilde $\widetilde{\cdot}$ to denote soft (score-based or probabilistic) predictors, reserving untilded notation for their hard, discrete counterparts. 
We introduce two \emph{soft} multi-model-classifier classes $\widetilde{\Pi}_\Flat$ and $\wt{\Pi}_\casc$ in parallel to $\Pi_\Flat$ and $\Pi_\casc$. 
Each element $\tpi$ of these classes  is composed of a sequence of score functions $\tilde{f}_{1:T}$
%, which output real-valued scores, 
and  a soft router $\tilde{r}$ that outputs a probability distribution over the $T$ classifiers. Formally, we define the class $\wt{\Pi}_\Flat$   as: 
\begin{align}
    &\wt{\Pi}_\Flat=\Big\lbrace \tilde{\pi}=(\tilde{f}_{1:T}, \tilde{r}): \tilde{f}_t:\mc{X}\times \mc{Z}_t \to \R^{K}\,\text{with}\,\nonumber\\
    &\sum_{k\in[K]}[\tilde{f}_t(x, z_t)]_k=0\; \forall  (x, z_t),\;\text{and}\;\tilde{r}: \mc{X}  \to \Delta^{T-1}\Big\rbrace. \label{eq:flat_space}
\end{align}
Each $\tilde{\pi}=(\tf_{1:T}, \tr)\in \wt{\Pi}_\Flat$ induces a hard multi-model classifier $\pi=(f_{1:T}, r)=\varphi_\Flat(\tilde{\pi})\in \Pi_\Flat$ via a   deterministic argmax map $\varphi_\Flat: \wt{\Pi}_\Flat\to \Pi_\Flat$, defined by the componentwise relations: 
\begin{align}\label{eq:hard_ft_from_soft_ft}
f_t(x,z_t) &\in  \argmax_{k\in[K]} [\tilde{f}_t(x, z_t)]_k,
\quad \forall t\in[T],
\\
r(x) &\in \argmax_{t\in[T]} \tilde{r}_t(x),\; \;\text{with}\quad  \tilde{r}_t(x)= [\tilde{r}(x)]_t\,.\label{eq:hard_r_from_soft_r}
\end{align}
For the cascade, we define the class  $\wt{\Pi}_\casc$ and the map $ \varphi_{\casc}: \wt{\Pi}_\casc\to  \Pi_\casc$ analogously. The soft router~$\tr$ consists of soft deferral rules   $\td_{t}: \mc{H}_t\to [0,1]$ for $t\in[T-1]$ with $\td_T\equiv0$, and the map $ \varphi_{\casc}$   induces the hard, thresholding router $r$ which replaces \eqref{eq:hard_r_from_soft_r} and mirrors the structure of \eqref{eq:def_r_casc}, taking the form 
\begin{align}\label{eq:hard_r_casc}
     r(h)\in  \min\{t\in [T]: \tilde{d}_t(h_t)<0.5\}. 
\end{align}
Our goal is to  design, for each $\model\in\{\Flat, \casc\} $, a  surrogate loss defined over the soft  class $ \wt\Pi_\model$ such that driving the surrogate risk of $\tpi \in \wt\Pi_\model$ to its minimum drives the target $0$-$1$-$c$ risk~\eqref{eq:def_R(f;c)} of the induced hard multi-model classifier to its minimum over the hard class $\Pi_{\model}$. In particular, any minimizer $\tpi_\model^*$ of the surrogate risk induces a minimizer $ \pi_\model^* \!=\!\varphi_\model(\tpi_\model^*)$ of the target risk. 
This desired property is Bayes consistency. 
In the sequel, we drop  subscript from $\varphi_\model$ when the model~is~clear~from~context.

\paragraph{Cost-sensitive hinge loss.}  Let $[x]_+\coloneqq \max(x,0)$. We propose a cost-sensitive hinge loss  $\ell^\hinge: \R^K\times [K]\times \R_{\ge0}\to \R_{\ge0}$, which  acts on the score vector $v $ output by a given  soft classifier (e.g., $\tf_t$), the true label $y$, and the cost $c $ of the classifier as follows:
\begin{align}\label{eq:stage_wise_surrogate}
     \ell^\hinge(v, y; c)=
\sum_{k\in[K]\setminus y}
\Big[v_k + \tfrac{1}{K-1}\Big]_+
+ \tfrac{K}{K-1}c\,.
\end{align}
This loss is a cost-sensitive extension of the multiclass hinge loss from~\citep{tarigan2008} and is also used for building the surrogate loss in~\citep{regol2025acquisition} for jointly training  two classifiers. We use it in Sections~\ref{subsec:flat_main_results} and \ref{subsec:cascade_main_results} to construct   surrogates of the target loss \eqref{eq:multi_01c} for joint training  in $\wt\Pi_\Flat$ and $\wt\Pi_\casc$.

\subsection{Bayes-Consistent Loss for the Flat Model}\label{subsec:flat_main_results}
 For training $\tpi =(\tf_{1:T}, \tr)\in\wt{\Pi}_\Flat$ in the flat setting, we construct the surrogate loss $\ell^\hinge_\Flat: \mc{X}\times \prod_{t=1}^T\mc{Z}_t\times [K]\times \wt{\Pi}_\Flat\times \R^T_{\ge0}\to \R_{\ge0}$ as a  weighted sum  of  the hinge loss  $\ell^{\hinge}$   across all classifiers: 
\begin{align}\label{eq:surrogate_hinge_flat}
&\ell^\hinge_\Flat\big(x,z_{1:T},y; \tpi, c\big)
=\sum_{t=1}^T
\tilde{r}_t(x) \ell^\hinge\big(\tilde{f}_t(x, z_t), y; c_t\big),
\end{align}
with the  associated risk $\mc{R}^\hinge_\Flat(\tpi)\equiv \mc{R}^{\hinge}_\Flat\big(\tpi; c\big)
:=
\ex\big[
\ell^{\hinge}_\Flat\big(X, Z_{1:T}, Y; \tpi, c\big)
\big].$ The calibration bound below shows  that the     excess $0$-$1$-$c$ risk is linearly bounded by the excess hinge risk in the flat setting.

\begin{restatable}[Linear calibration of $\ell^\hinge_\Flat$
w.r.t.~$\ell$]
{theorem}{excessriskflat}
\label{thm:calibration}
Define $
\mc{R}^*_\Flat
\!:=\!
\min_{\pi\in \Pi_{\Flat}} \mc{R}(\pi)$ and $ \mc{R}^{\hinge*}_\Flat
\!:=\!
\min_{\tilde{\pi}\in \wt{\Pi}_{\Flat}} \mc{R}^\hinge_\Flat(\tilde{\pi})$.
Then, for any $\tilde{\pi}\in \wt{\Pi}_\Flat$\,,
\begin{align}
\mc{R}\big(\varphi(\tilde{\pi})\big)\! -\! \mc{R}^*_\Flat
\le
 T(K\!-\!1)
\Big(
\mc{R}^{\hinge}_\Flat(\tilde{\pi}) \!-\! \mc{R}^{\hinge*}_\Flat
\Big).
\label{eq:flat_calibration}
\end{align}
\end{restatable} 
This  implies   
$\mc{R}^{\hinge}_\Flat(\tilde{\pi})\!\to\! \mc{R}^{\hinge*}_\Flat
\!\Rightarrow\!
\mc{R}\big(\varphi(\tilde{\pi})\big)\!\to \!\mc{R}^*_\Flat$, establishing Bayes consistency of the surrogate loss $\ell_\Flat^\hinge$. We detail the proof in Appendix~\ref{app:flat_all}. 

\subsection{Bayes-Consistent Loss for  the Cascade }\label{subsec:cascade_main_results}
For training $\tpi =(\tf_{1:T}, \tr)\in \wt{\Pi}_\casc$, we  build the cascade surrogate hinge loss $\ell^\hinge_\casc: \mc{X}\times \prod_{t=1}^T\mc{Z}_t\times [K]\times \wt{\Pi}_\casc\times \R^T_{\ge0}\to \R_{\ge0}$ as follows:
\begin{align}\label{eq:surrogate_hinge_casc}
&\ell^\hinge_\casc(x, z_{1:T},y;\tpi, c) =\sum_{t=1}^T\tr_t(x,z_{1:t})\ell^\hinge\big(\tf_t(x,z_{t}),y; c_t\big) ,
\end{align}
where 
\begin{align}\label{eq:r_tilde_t_def}
    \tr_t(x,z_{1:t})&= 
\Bigg(\prod_{\tau=1}^{t-1} \td_\tau(x,z_{1:\tau})\Bigg)
\big(1 - \tilde{d}_t(x,z_{1:t})\big),
\end{align}  
 and $\tilde{d}_T\equiv 0$.
The outer sum in \eqref{eq:surrogate_hinge_casc} runs over exit stages $t\in[T]$.  The $\tr_t(x,z_{1:t})$ term represents the probability of exiting at stage $t$, 
and the  $\ell^\hinge$ term is the  hinge loss incurred upon exiting at that stage. We define the  risk  associated with the   loss function $\ell_\casc^\hinge$ as  $\mc{R}^\hinge_\casc(\tpi)\equiv \mc{R}_\casc^\hinge(\tpi; c):=\E\big[\ell_\casc^\hinge(X,Z_{1:T}, Y;\tpi, c)\big]$.

The cascade setting is technically more challenging to analyze than the flat setting because  adaptive deferral decisions induce the stopping time  in \eqref{eq:hard_r_casc}, which  indexes the  classifier ultimately invoked. As a result, the optimal deferral decision at each stage must account for both the immediate loss and the  expected future loss. To prove that $\ell_\casc^\hinge$ is calibrated in an analogous sense to Theorem~\ref{thm:calibration} of the flat setting, we cast the cascade as a finite-horizon dynamic program and use backward induction with stage-wise risk decomposition to \emph{recursively} propagate single-stage calibration through the cascade, ultimately yielding  
Theorem~\ref{thm:calibration_casc} below. 
\begin{restatable}[Linear calibration of $\ell^\hinge_\casc$~
w.r.t.~$\ell$]{theorem}{excessriskcasc}
\label{thm:calibration_casc} Assume  $Y\!\perp\! Z_{1:t-1}\,|\, (X, Z_t)$ as in \eqref{eq:conditional_indep}.
 Define $
\mc{R}^*_\casc
\!:=\!
\min_{ \pi \in \Pi_{\casc}} \mc{R}(\pi)$ and  $\mc{R}^{\hinge*}_\casc
\!:=\!
\min_{\tpi \in \wt{\Pi}_{\casc}} \mc{R}^\hinge_\casc(\tpi)$.
Then, for any $\tpi \in \wt{\Pi}_\casc$\,,
\begin{align}
\mc{R}\big(\varphi(\tpi)\big) \!- \!\mc{R}^*_\casc
\le
 2^{T-1}(K\!-\!1) 
\Big(
\mc{R}^{\hinge}_\casc(\tpi)\! -\! \mc{R}^{\hinge*}_\casc
\Big).
\label{eq:casc_calibration}
\end{align}

\end{restatable}
 This result establishes the Bayes consistency of $\ell_\casc^\hinge$. We detail its proof    in Appendix~\ref{app:cascade_all}. Importantly, the calibration constant $ 2^{T-1}(K\!-\!1)$ cannot be improved: 
 Proposition~\ref{prop:LB} gives a matching lower bound by constructing bad   instances for which the smallest valid calibration constant approaches $ 2^{T-1}(K-1)$. In these instances,  every  soft deferral decision is a tie between exit and defer, i.e., $\td_t\equiv 0.5$ for all $t\in[T-1]$. When the soft deferral decisions  are more decisive, the constant can be substantially smaller:  Proposition~\ref{prop:tighter_UB}  shows that for $m\!\in\![0,0.5]$, if $|\td_t(H_t)\!-\!0.5|\!\ge\! m$ almost surely for all $t\in [T-1]$, the constant reduces to $(\frac{2}{1+2m})^{T-1}(K\!-\!1)$. 
 
Sections~\ref{sec:synthetic} and~\ref{sec:LLM_experiments} present experiments that validate  Theorems~\ref{thm:calibration}--\ref{thm:calibration_casc},
and demonstrate their practical utility.

\section{Improving Routing with Unlabeled Data and PPI}\label{sec:section_PPI}

While Section~\ref{sec:main_results}
establishes the Bayes consistency of  the proposed hinge losses as the sample size $n\!\to \!\infty$, practical training operates in  the finite-sample regime. 
Empirical evidence from prior work~\citep{shazeer2017, kubaty2025how}, as well as our own experiments (see Fig.~\ref{fig:route_vs_class} in Appendix~\ref{app:training_trajectories}), shows that the classifiers $f_{1:T}$ converge faster than the router  $r$. Consequently, $r$ can be suboptimal if labeled data is limited, while $f_{1:T}$ may be more reliable. This motivates  post-training methods that improve the routing decisions  efficiently without retraining, by leveraging unlabeled data. 

In this regime, the  prediction-powered inference (PPI) framework~\citep{Angelopoulos2023PPI} provides a natural tool. PPI uses predictions from trained models on abundant \emph{unlabeled} data to construct confidence intervals (CI) for quantities otherwise hard to estimate reliably from \emph{labeled} data alone. In our setting, the trained classifiers can provide these predictions, which we use to refine our trained router.

Accordingly, we consider a setting with   access to a large unlabeled dataset $\{X^{(i)}, Z^{(i)}_{1:T}\}_{i=1}^N\stackrel{\text{i.i.d.}}{\sim} P$ in addition to a held-out labeled dataset $\{X^{(i)}, Z^{(i)}_{1:T}, Y^{(i)}\}_{i=1}^l$ where $N\gg l$, both separate from the data used to train the multi-model classifier. We show how  PPI can be used post-training to improve a flat router, and note that the proposed method may be extended to the cascade setting via the recent e-values framework~\citep{ramdas2025hypothesis,grunwald2020safe,vovk2021values, csillag2025predictionpowered, Kilian2025PPI_E}. 
For notational simplicity,  we present our construction and high-probability guarantees pointwise for a fixed input feature $x\in \mc{X}$. In practice, however, test inputs  (e.g., prompts or images) rarely coincide with training samples. We therefore implement the method using \emph{binning}, which groups similar inputs to ensure coverage of previously unseen samples. The pointwise guarantees in the sequel should be interpreted at the bin level.

\paragraph{Oracle router $r^\dagger_\Flat$\,.} 
       % \E[\ell(f^\dagger(x,Z_{1:T}),Y; c)] ,
Let $\pi^\dagger=(f_{1:T}^\dagger, r_\Flat)$ denote the multi-model classifier after training. Given $f^\dagger_{1:T}$, the Bayes-optimal routing decision for feature $x$ is the one that  minimizes the conditional $0$-$1$-$c$ risk  given $X=x$, taking the form
\begin{align}
r_\Flat^\dagger(x)
&\in \argmin_{t\in[T]  } \mc{R}_t(x; f^\dagger_t) \,,\label{eq:r_dagger} 
\end{align}
where, for any classifier $f_t: \mc{X}\times \mc{Z}_t\to [K]$,      \begin{align}
\mc{R}_t(x; f_t)&:=
\E\big[\Ind\{f_t(x, Z_t)\ne Y\} \mid X=x\big]+c_t\,. \label{eq:cond_R_def}
\end{align}
Since $\mc{R}_t(x; f_t^\dagger)$ is an oracle population quantity (an expectation over the true data distribution), 
the  routing decision $r_\Flat^\dagger(x)$   is not  computable, but provides a natural benchmark. Our goal is \emph{not} to estimate the oracle directly, but rather to \emph{certify} whether the trained routing decision can be safely improved using unlabeled data to approach the oracle. 
\paragraph{PPI-corrected router $r^\PP_\Flat$\,.} To this end, we compare each alternative routing decision to the trained baseline by comparing the difference in conditional risks. Let  $\delta_t(x)\coloneqq \mc{R}_{r_\Flat(x)}(x; f_{r_\Flat(x)}^\dagger)-\mc{R}_t(x; f_t^\dagger)$  for  $t\in [T]\setminus \{r_\Flat(x)\}$.  We leverage PPI to
  construct the 95\% CI around $\delta_t(x)$  in the form:
  \begin{align}\label{eq:PPI_CI_result}
\Prb\left(\big|\delta_t(x)- \delta_t^\PP(x)\big|\ge w_{t,N_b,l_b}(x)\right)\le \tfrac{0.05}{T-1}\,,
  \end{align}
  where $\delta_t^\PP(x)$ and $w_{t,N_b,l_b}(x)$  have   explicit expressions which we detail in Appendix~\ref{app:PPI_CI_construction}, $N_b$ and $l_b$ denote the numbers of unlabeled and labeled points in the bin, and the probability is over the randomness of both the labeled and unlabeled data.  
  Based on these CIs,  we identify a  set of routing destinations $\mc{S}(x)$ that are certified, with  probability $1-\frac{0.05}{T-1}$, to be no worse than the baseline decision $r_\Flat(x)$. If $\mc{S}(x)$ is nonempty, we propose the PPI-corrected routing decision $r_\Flat^\PP(x)$  by selecting the alternative with the largest certified improvement; otherwise, we retain the   baseline. Formally, we define
\begin{align}
    &\mc{S}(x)\coloneqq\big\{ t\in [T]\setminus \{r_\Flat(x)\}: \delta_t^\PP(x)-w_{t,N_b,l_b}(x)\ge 0\big\}, \nonumber\\
    &r_\Flat^\PP(x)\coloneqq\begin{cases}
        \displaystyle   \argmax_{t\in \mc{S}(x)}   \delta^\PP_t(x)-w_{t,N_b,l_b}(x) &\text{if}\; \mc{S}(x)\ne \emptyset\,,\\
        r_\Flat (x) &\text{otherwise}\,.
    \end{cases}  \label{eq:r_PPI}
\end{align}
Then  $r_\Flat^\PP$ is guaranteed to be never worse than the baseline $r_\Flat$ in the  sense of \eqref{eq:no_worse_than} below in Theorem~\ref{thm:PPI}. Under an additional mild separation assumption, where  the oracle decision is unique while the gap between the oracle and the next runner-up is strictly positive, i.e., 
 \begin{align}
     \gamma(x)\coloneqq \delta_{r_\Flat^\dagger(x)}(x)-\max_{t'\in [T]\setminus \{r^\dagger_\Flat(x)\}}\delta_{t'}(x)>0\,,\label{eq:gap_PPI}
 \end{align} 
 we can obtain the stronger guarantee in \eqref{eq:PP_equals_oracle}.  
The CI half-width $w_{t,N_b,l_b}(x)$ decreases   with the accuracy  of the classifier  used to predict the labels of the $N$ unlabeled samples, motivating our use of the most   powerful classifier $f^\dagger_T$ among $f^\dagger_{1:T}$ in our experiments in Section~\ref{sec:PPI_experiments}. If the imputing classifier $f^\dagger_T$ is accurate,  $w_{t,N_b,l_b}(x)$  decreases with  $N$ roughly as $1/\sqrt{N}$, so for  large enough $N$,  the CI for $\delta_{r^\dagger_\Flat(x)}(x)$ becomes narrow enough to separate it from the  runner-up, guaranteeing the  recovery of the oracle decision: $r_\Flat^\PP(x)= r^\dagger_\Flat(x)$. 

\begin{theorem}[Certification for PPI-corrected router]\label{thm:PPI}
For a given $x\in \mc{X}$ and a fixed trained multi-model classifier $\pi^\dagger=(f^\dagger_{1:T}, r_\Flat)$, let $r^\dagger_\Flat(x)$  and $r^\PP_\Flat(x)$ be as defined in \eqref{eq:r_dagger}  and \eqref{eq:r_PPI}, then 
\begin{align}
\Prb\Big(\mc{R}_{r_\Flat^\PP(x)}\big(x; f^\dagger_{r_\Flat^\PP(x)}\big)
  &> \mc{R}_{r_\Flat(x)}\big(x; f^\dagger_{r_\Flat(x)}\big) \, \cap\, \mc{S}(x)\ne \emptyset\Big)\le 0.05. \label{eq:no_worse_than}
\end{align}
Moreover, with $\gamma(x)$ as defined in \eqref{eq:gap_PPI}, then 
\begin{align}
\Prb\Big(&r^\PP_\Flat(x)\ne r^\dagger_\Flat(x)\,\cap\, \mc{S}(x)\ne \emptyset  \, \cap\, w_{r^\dagger_\Flat(x), N_b, l_b}(x)< \gamma(x)/2\Big)\le 0.05. \label{eq:PP_equals_oracle}
\end{align}
\end{theorem}
The probabilities in \eqref{eq:no_worse_than} and \eqref{eq:PP_equals_oracle} are over the randomness of both the labeled and unlabeled data. 
The proof is a  combination of \eqref{eq:PPI_CI_result} with a union bound. For sufficiently accurate  classifiers $f^\dagger_{1:T}$ and sufficiently large $N$, the CIs become sufficiently tight that $\mc{S}(x)\ne \emptyset$ with high probability, in which case \eqref{eq:no_worse_than} and \eqref{eq:PP_equals_oracle} effectively reduce to stronger  marginal guarantees. We experimentally illustrate the PPI corrections in Section~\ref{sec:PPI_experiments}. 

\section{Experiments}
\label{sec:experiments}
\subsection{Validating the Bayes Consistency of Surrogate Losses on Synthetic Data}
\label{sec:synthetic}

To validate Theorems~\ref{thm:calibration} and \ref{thm:calibration_casc} in a controlled setting, we study a synthetic classification task with known ground-truth conditional distributions. We  compute the Bayes-optimal solutions (as characterized in Lemmas~\ref{lem:bayes_opt_flat} and \ref{lem:bayes_opt_casc_V}) under both the flat and cascade routing models 
and verify that minimizing our proposed surrogate losses recovers these solutions.

Under both models, we consider $T{=}3$ classifiers for a $K{=}3$ class   problem. The data consists of a base feature $X\,{\in}\, \mathbb{R}$ and two additional features $Z_2, Z_3 \,{\in}\,\mathbb{R}$ all sampled independently from $\mathrm{Uniform}[-1,1]$, and a label $Y\,{\in}\, [K]$ generated  according to: 
\begin{figure*}[t]
    \centering \includegraphics[width=0.495\linewidth]{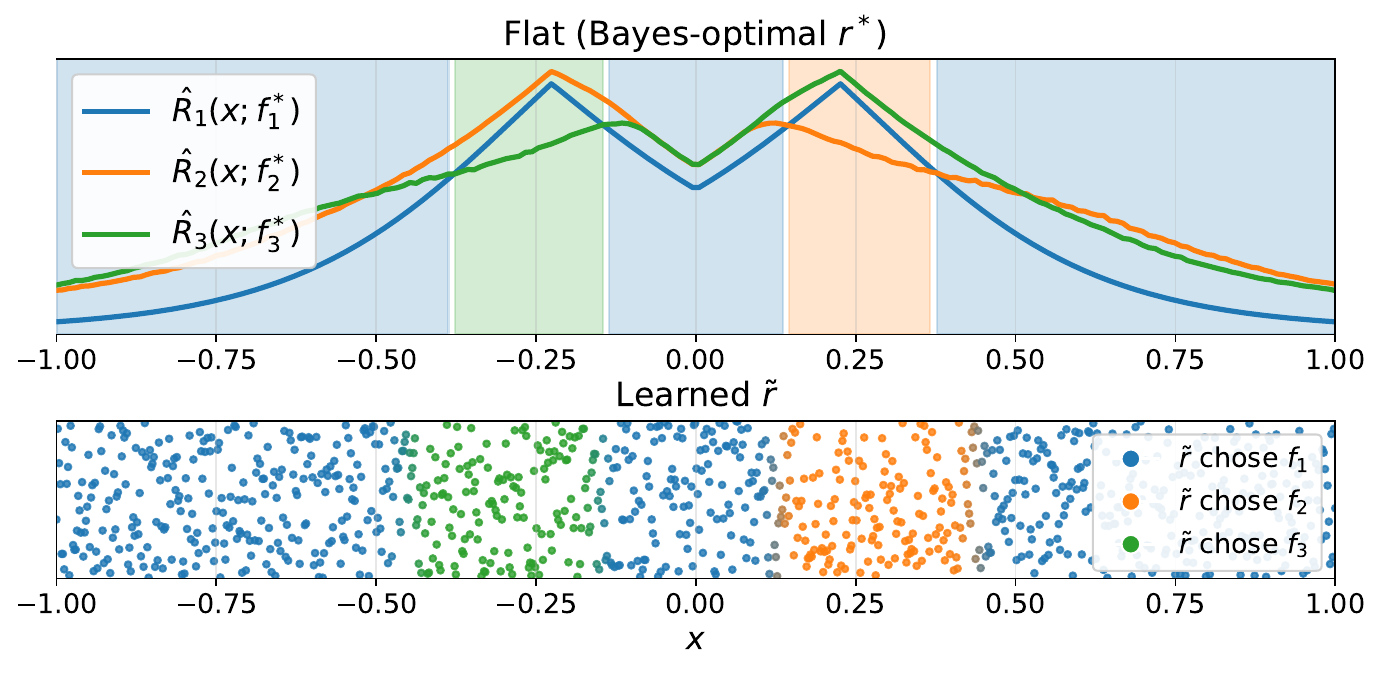}
      \includegraphics[width=0.495\linewidth]{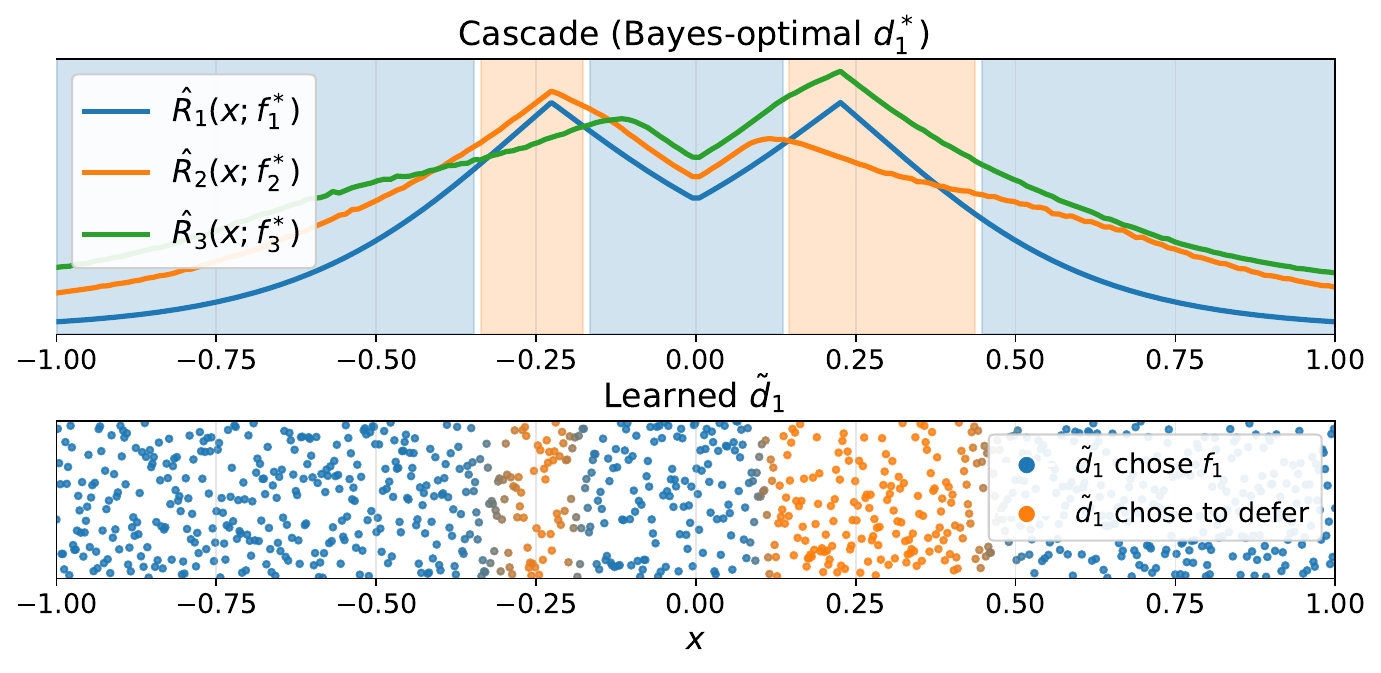}
    \caption{(Top) Monte-Carlo estimate of the Bayes-optimal conditional  risk and Bayes-optimal routing decisions for the flat router $r^*$ (left) and for the first stage of the cascade router $d^*_1$ (right). 
    %with $\kappa=0.05$. 
    (Bottom) Learned routing decisions on the test set, where points are colored according to the routing probabilities learned by the model.
    }
    \label{fig:task} 
\end{figure*}
\begin{align}
&Y \mid (X=x,Z_t=z_t) \sim \mathrm{Categorical}\big(\boldsymbol{\theta}(x,z_t)\big)\,,
\label{eq:synthetic_generating_process}
\end{align}
where the full specification of $\boldsymbol{\theta}$  is provided in Appendix~\ref{app:synthetic_details}. We set the cost of $f_1$ to $\kappa_1=0$ and assign the same cost $\kappa_2=\kappa_3=0.05$ to  $f_2$ and $f_3$, yielding cost vector  $c=[0, 0.05, 0.05]^\top$ in the flat setting and $c=[0, 0.05, 0.1]^\top$ in the cascade setting. 

The Bayes-optimal routing decisions in both settings are governed by the conditional risk $\mc{R}_t(x;f_t^*)$ in  \eqref{eq:cond_R_def}:  see Eq.~\eqref{eq:r_star_multi_app} for the flat and \eqref{eq:d_star_casc_V} for the cascade setting. We choose $\boldsymbol{\theta}$ so that this conditional risk is minimized at different exit stage $t$ over different regions of $\mc{X}$;  the solid lines in the top panels  of Fig.~\ref{fig:task} show its Monte-Carlo estimate denoted by $\hat{\mc{R}}_t(x;f_t^*)$.
This induces nontrivial Bayes-optimal routing boundaries, shown by the shaded regions in the top panels. In the flat setting (top left panel), the Bayes-optimal router $r^*$ selects the classifier with the lowest conditional risk, following~\eqref{eq:r_star_multi_app}. In the cascade setting (top right panel), 
the stage-1 Bayes-optimal deferral rule $d^*_1$ 
does \emph{not} align with the lower envelope of the risk curves, since it must compare  the stage-1 conditional risk $\hat{\mc{R}}_1(x; f_1^*)$ against the expected cost of optimally continuing the cascade, following~\eqref{eq:d_star_casc_V}. Consequently, the left orange-shaded region in the cascade  figure is narrower than the green-shaded  region  in the flat figure, indicating   fewer samples are routed to $f_3$ in the cascade, due to the accumulated cost of deferral. 

Having illustrated the Bayes-optimal solutions, we now train  flat and  cascade multi-model classifiers, and compare the  learned functions with their Bayes-optimal counterparts. Each classifier $f_1(x), f_2(x, z_2), f_3(x, z_3)$ is implemented as a three-layer multilayer perceptron. For simplicity, the synthetic setup takes $\mc{Z}_1 = \emptyset$, so $f_1$ and $d_1$ depend on $x$ alone. Routers are  implemented as three-layer networks: a multiclass routing network $r: \mc{X} \to \{1,2,3\}$ in the flat setting, and two sequential binary routing networks $d_1: \mc{X} \to \{0,1\}$ and $d_2: \mc{X}\times\mc{Z}_2 \to \{0,1\}$ in the cascade setting. Models are trained end-to-end with Adam  using our proposed surrogate losses, i.e.,  \eqref{eq:surrogate_hinge_flat}  for flat routers and  \eqref{eq:surrogate_hinge_casc}   for cascades. Appendix~\ref{app:synthetic_details} provides additional architectural and optimization details. 

The bottom panels of Fig.~\ref{fig:task} visualize the learned routing boundaries for both the flat and cascade settings (via the first deferral function $\tilde{d}_1$), where each test point is colored according to the routing probabilities learned by $\tr$ or $\td_1$. In both cases, the learned regions closely match the Bayes-optimal regions in the top panels, confirming our Bayes consistency results from Theorems~\ref{thm:calibration} and \ref{thm:calibration_casc}.

Complementing the  routing visualizations in Fig.~\ref{fig:task}, Fig.~\ref{fig:emp_risk} in Appendix~\ref{app:empirical_risk_vs_bayes} directly compares the empirical $0$-$1$-$c$ risk over minibatches against the empirical Bayes-optimal $0$-$1$-$c$ risk for both the flat and cascade settings, under two cost configurations. In all cases,  minimizing  the proposed hinge losses attains the Bayes-optimal risk on the test set, further validating our theory.

\subsection{Demonstrating Post-Training PPI-Based Routing Improvement on Synthetic Data}
\label{sec:PPI_experiments}

\begin{figure}[t]
    \centering
\includegraphics[width=0.6\linewidth]{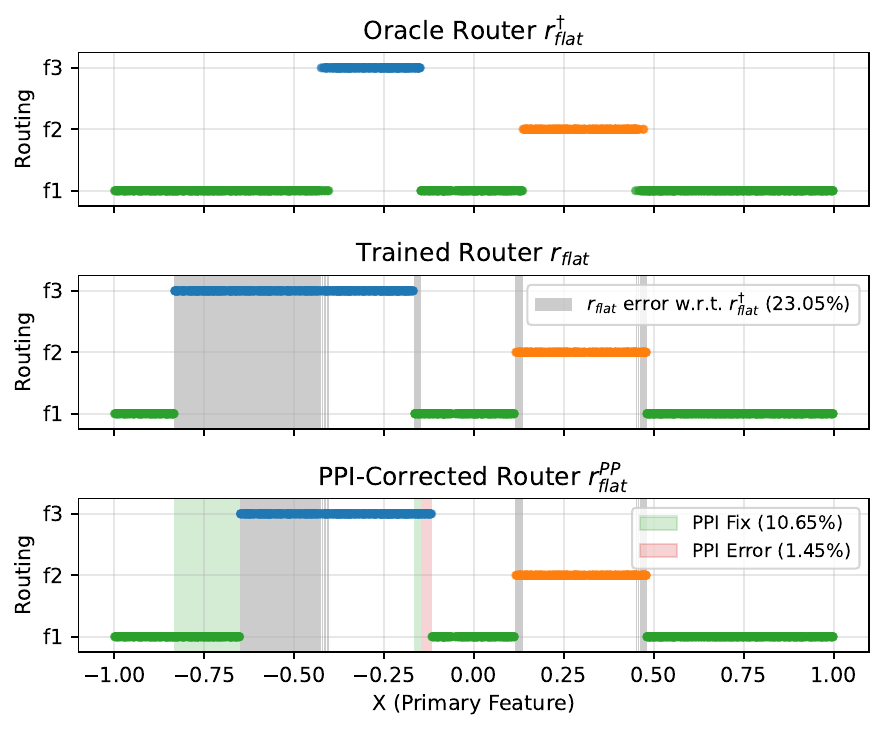}
    \caption{Comparison of the oracle router $r^{\dagger}_\Flat$, the router $r_{\Flat}$ trained on labeled data, and the  PPI-corrected router $r^{\PP}_{\Flat}$ leveraging unlabeled data. The fixes and errors introduced  by PPI are computed with respect to $r^{\dagger}_\Flat$.}
    \label{fig:ppi}
\end{figure}

Fig.~\ref{fig:route_vs_class} in Appendix~\ref{app:training_trajectories}
compares the   training dynamics of  the classifiers and   router by plotting their training accuracies over minibatches.  The classifiers converge quickly and plateau early, while the router requires much more labeled data to converge. When labeled data is scarce, we now demonstrate how the post-training PPI procedure of Section~\ref{sec:section_PPI}  leverages    unlabeled data to improve routing decisions at test time.

We use  the same synthetic data generating process  from Section~\ref{sec:synthetic} with $2000$ labeled samples ($n=1000$ training samples and $l=1000$ PPI labeled samples) and $N = 20,000$ unlabeled samples, and   partition the input space $\mc{X}$ using K-means clustering (binning) with 40 bins to account for unseen data, as described in Section~\ref{sec:section_PPI}.  We consider    three classifiers and a flat router,  set up identically to  those  in Section~\ref{sec:synthetic}. The classifiers and router are trained on the training labeled samples only. Post-training, routing decisions at test time are  corrected via the PPI   procedure~of~Section~\ref{sec:section_PPI}.

Fig.~\ref{fig:ppi} compares three routers: the oracle router $r^{\dagger}_\Flat$, the trained router $r_\Flat$, and the PPI-corrected router $r_\Flat^{\PP}$. While the trained router  deviates noticeably from the oracle  due to   limited   labeled data,  PPI correction  substantially  improves it, approaching  $r^{\dagger}_\Flat$ and    
%by leveraging unlabeled data, 
fixing $10.65\%$ routing errors at the cost of $1.45\%$ mistakes. 
Note that  Theorem~\ref{thm:PPI}  guarantees improvement towards  the oracle router $r^{\dagger}_\Flat$, not the Bayes-optimal router $r_\Flat^*$. However, even if the trained router is poor, the trained classifiers can be closer to their Bayes-optimal counterparts, in which case the oracle router would be near-optimal ($r^{\dagger}_\Flat \approx r_\Flat^*$). We illustrate this in Appendix~\ref{sec:app_ppi}, where we additionally include the Bayes-optimal router and vary the amount of labeled data to show how this affects the behavior of the oracle, the trained router and the PPI correction. Finally, following the theory, the experiment above keeps the training data and the labeled PPI data separate. However, setting labeled data aside from training can be wasteful, so in Appendix~\ref{sec:app_ppi} we also show that reusing the training data as the labeled PPI data can still offer improvements, even if it violates the independence assumption of Theorem~\ref{thm:PPI}.

\subsection{Large Language Models Experiment}
\label{sec:LLM_experiments}

To showcase the practical effectiveness of our surrogate losses, we conduct experiments in the $T=3$ setting with LLMs from two model families LLaMA 3 1B, 70B, and 405B~\citep{dubey2024llama3herdmodels} 
and Qwen 2.5 1.5B, 32B, 72B~\citep{yang2025qwen25technicalreport}. In this section we show flat router results for LLaMA models. For cascade router results, and Qwen model experiments with flat router and cascade router setups see Appendix~\ref{app:additional_experiments}.

{\bf Datasets.} We use two standard reasoning benchmarks from the literature: CommonsenseQA~\citep{talmor2018commonsenseqa} (commonsense logical reasoning) and GSM8K ~\citep{cobbe2021gsm8k} (grade-school math reasoning). See Appendix~\ref{app:dataset} for more details on the datasets.
For all experiments we use a 80-10-10 train-validation-test split.

{\bf Setup.} We form each input feature $x$ by embedding the natural language prompt into a vector. Additional features $z_t$ are obtained by embedding the LLM outputs in a similar fashion and learn light classifier modules $f_t$ on top of them. We provide experimental and implementation details in Appendix~\ref{app:model_training_llm}. Appendix~\ref{app:additional_experiments} reports cascade-router results with the LLaMA models and flat-router results with the Qwen models.

{\bf Results.} Fig.~\ref{fig:cost_comp} presents a flat router experiment with the three LLaMA models. We present three panels, one per cost configuration (in increasing order of costs from left to right with $c\,{=}\,[0, 0.01, 0.02]$, $[0, 0.05, 0.07]$, $[0, 0.1, 0.2]$).
For each configuration, we display statistics through box-and-whisker plots computed via bootstrap sampling over 10 trials and the median (horizontal line). 
\begin{figure}[t]
    \centering
    \includegraphics[width=0.6\linewidth]{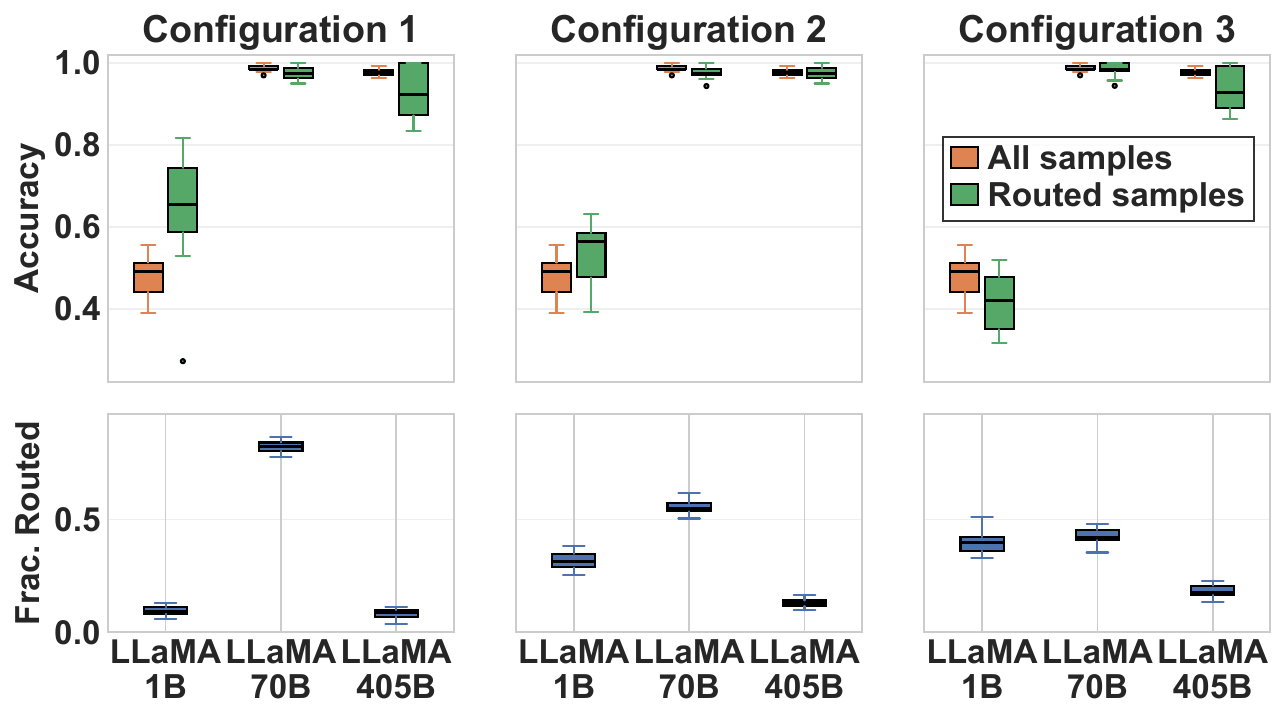}
    \caption{Classifier accuracy and routing fractions on  GSM8K across three cost configurations   (costs increasing from left panel to right panel).
    }
    \label{fig:cost_comp}
\end{figure}

The orange box represents the performance of each model on the entire dataset (independent of the routing decision), the green box represents the performance of each model on the samples that are routed to it, and the blue box (bottom row) represents the fraction of samples that are routed to each model. 
In configuration 1 (left panel), the gap between the orange and green boxes shows that the router successfully identifies easier samples on which LLaMA 1B is more accurate than average.
As costs increase across the three configurations, a larger fraction of questions is routed to the cheapest model, causing this model's accuracy on the routed subset (green) to decrease as the model receives a larger fraction of harder questions. Meanwhile, LLaMA 70B and 405B achieve accuracies in the nineties on routed samples. Thus, joint classifier-router training under our proposed loss enables cost savings by finding samples where the 1B is accurate. 

In Appendices~\ref{app:add_llm_cascade}--\ref{app:ablation}, we provide additional experiments to demonstrate the utility of our proposed training losses in three aspects:   (i) they induce consistent cost-accuracy tradeoff behavior, similar to that in Fig.~\ref{fig:cost_comp}, across   flat and cascade settings, as well as other datasets and model families (\ref{app:add_llm_cascade}); (ii) they attain the lowest empirical $0$-$1$-$c$ risk against popular separate-training and joint-training baselines (\ref{app:casc_baselines} and~\ref{app:adap_baselines}); and (iii) they scale stably with the number of classifiers~$T$~(\ref{app:ablation}).

\section{Conclusion}

We present a principled framework for 
jointly training  classifiers and routers in cost-sensitive multi-model classification. We  propose hinge-based surrogate losses and  establish their Bayes consistency with respect to the   cost-sensitive zero-one loss in both single-stage  and cascade routing models. To address  label scarcity, we  design a   post-training PPI procedure that leverages unlabeled data to refine a trained single-stage router. Experiments on synthetic data and LLM benchmarks validate our theory and   the utility~of~the hinge losses.

Our framework has some limitations. In practice, full joint training can be computationally demanding, and the hinge loss can be harder to optimize than a common log-based version. Moreover, our analysis assumes that the distribution of the additional features $Z_t$ is fixed during training, which excludes settings where they are produced by models that are fine-tuned jointly with the classifiers and router.  Appendix~\ref{app:limitation} discusses these, along with further limitations in more detail.

 \subsubsection*{Acknowledgements and Disclosure of Funding}
 S.\ X.\ Liu and P.\  Rebeschini were supported by UK Research and Innovation (UKRI) under the UK government’s Horizon Europe funding guarantee [grant number EP/Y028333/1]. 
%
%This work was partially funded by 
%UK Research and Innovation (UKRI) under the UK government’s Horizon Europe funding guarantee [grant number EP/Y028333/1], 
F.\ Regol, A.\ Valkanas and M.\ Coates were supported by 
the Natural Sciences and Engineering Research Council of Canada (NSERC)
[reference number 260250], and the Fonds de recherche du Québec (FRQNT) [reference number 324302]. Cette recherche a été partiellement financée par le Conseil de recherches en sciences naturelles et en génie du Canada (CRSNG) [numéro de référence 260250] et le Fonds de recherche du Québec (FRQNT) [numéro de référence 324302]. 

\bibliographystyle{plainnat}
\bibliography{references}

\clearpage 
\appendix
\section*{Appendix Contents}
\startcontents[appendix]
\printcontents[appendix]{}{1}{}
\clearpage

\section{Extended Literature Review}\label{app:review}

We study  multi-model adaptive inference  for classification   by casting it  as a cost-sensitive  classification problem with the classifiers carrying varying costs. We    propose  losses for  jointly training the   classifiers and their associated  routing model. In this section, we survey related work on multi-model classification, both from an empirical and a theoretical perspective.

Section~\ref{ssec:empirical}  reviews empirical approaches, which prioritize  inference cost savings through the design of architectures  and   loss functions that tend to lack statistical guarantees. 
These works can be   further categorized  into general approaches applicable to any multi-classifier architecture,   and domain-specific ones. Since our work is mainly motivated by  LLM-related applications, we review LLM domain-specific approaches.
Section~\ref{ssec:theoretical}  covers the theoretical works, which we organize into three main streams: reject-option classification, learning to defer, and joint classifier-router training.
Figure~\ref{fig:tax} summarizes both empirical and theoretical works discussed here.

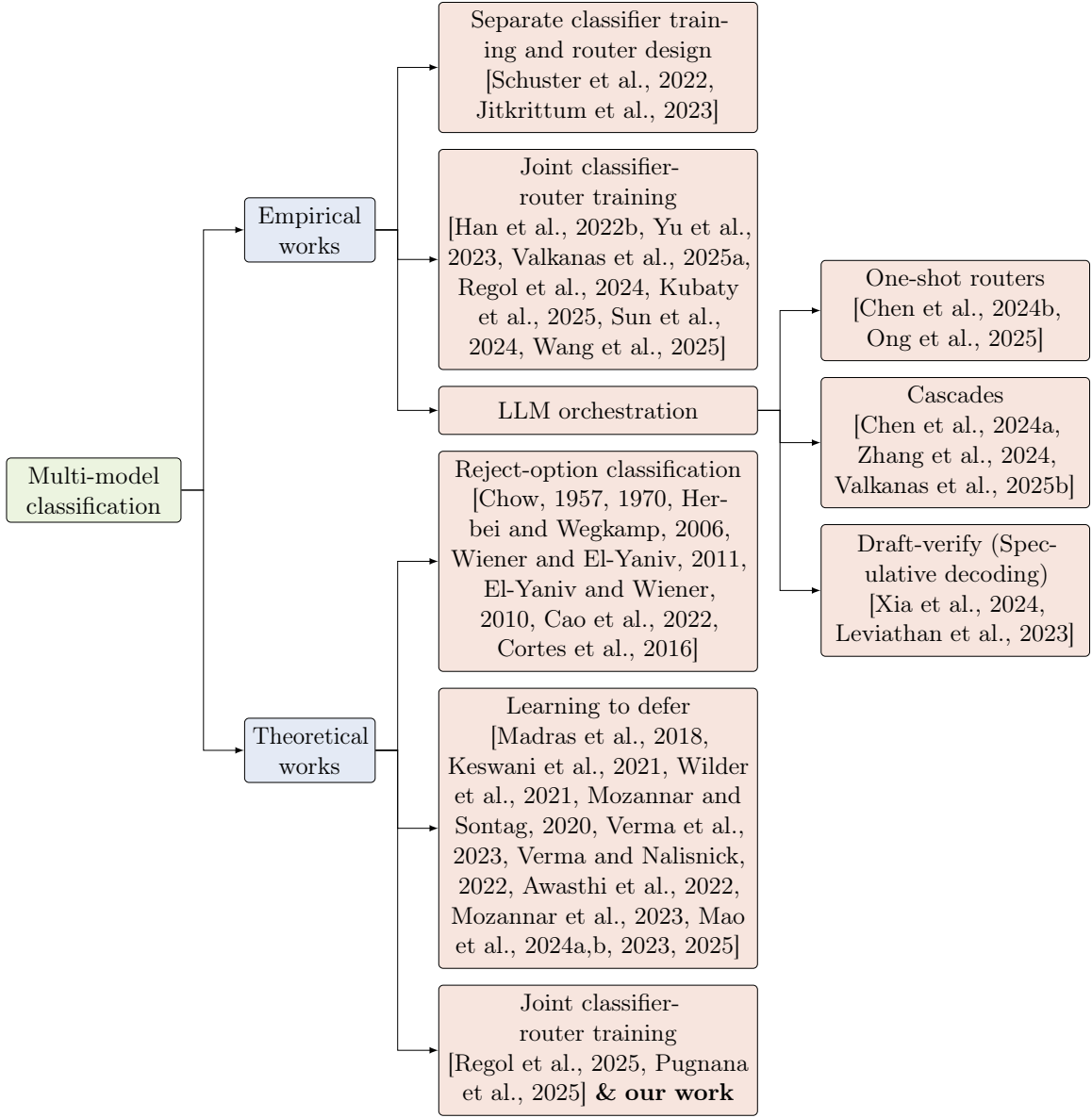
\begin{figure}[h!]
    \centering
    \tikzset{
        basic/.style  = {rectangle, draw, rounded corners=2pt, align=center, text opacity=1,
                         minimum height=2em,
                         inner sep=4pt, align=center, fill opacity=.5,},
        root/.style   = {basic,  
                         fill=LimeGreen!30,  text width=2.5cm},
        tnode/.style  = {basic,  fill=BrickRed!20, text width=4.8cm, minimum width=0pt},
        xnode/.style  = {basic,  fill=NavyBlue!20, text width=1.8cm, minimum width=0pt},
        edge from parent/.style={draw=black, edge from parent fork right}
    }
    \scalebox{0.88}{
    \begin{forest}
      for tree={
        grow=east,reversed,
        growth parent anchor=west,
        parent anchor=east,
        child anchor=west,
        edge path={
          \noexpand\path[\forestoption{edge},->,>=latex]
            (!u.parent anchor) -- +(10pt,0pt) |- (.child anchor)
            \forestoption{edge label};
        }
      },
      [Multi-model classification, root, l sep=10mm
        [Empirical works, xnode, l sep=10mm
            [Separate classifier training and   router design\\\citep{ schuster2022confident, Jitkrittum2023}, tnode]
            [Joint classifier-router training\\\citep{han2022Weighted, boostedNet, valkanas2024modl, regol2024jointlylearned, kubaty2025how, sun2024mixture, wang2025hmoe}, tnode]
          [LLM orchestration, tnode, l sep=10mm
           [One-shot routers\\\citep{chen2024routerdc, ong2025routellm}, tnode, text width=4cm]
           [Cascades\\\citep{chen2024frugalgpt, zhang2024treacle, valkanas2025cpo}, tnode, text width=4cm]
            [Draft-verify (Speculative decoding)\\\citep{spec_decoding_survey, leviathan2023_spec_decoding}, tnode, text width=4cm]
          ]
        ]
        [Theoretical works, xnode, l sep=10mm
          [Reject-option classification\\\citep{Chow1957, chow1970, Herbei2006, Wiener2011, ElYaniv2010foundations, cao2022generalizing, cortes2016}, tnode]
          [Learning to defer\\\citep{Madras2018, Keswani2021, Wilder2021, Mozannar2020, Verma2023,Verma2022a, Awasthi2022, Mozannar2023, mao2024a, mao2024b, mao2023, mao2025mastering}, tnode]
          [Joint classifier-router training\\\citep{regol2025acquisition, pugnana2025} \textbf{\& our work}, tnode]
        ]
      ]
    \end{forest}}
    \caption{Key related works on multi-model classification. 
    }
    \label{fig:tax}
\end{figure}

\subsection{Empirical Works}\label{ssec:empirical}

\paragraph{Domain-Agnostic Approaches.} Adaptive computation and dynamic networks  reduce average test-time inference cost  by conditionally allocating computation, e.g., by exiting early in a deep network and  escalating to more expensive components only when needed \citep{dnn_survey}. One common approach here is to train classifiers   \emph{separately} from the router and make test-time routing decisions  using confidence scores  compared against a calibrated threshold  \citep{schuster2022confident, Jitkrittum2023}. 

This is attractive in practice due to its simplicity, as it allows classifiers to be trained independently from the router ~using~standard~supervised~learning~losses.

In contrast to \emph{separate} training, a growing empirical literature addresses the train-test mismatch caused by this approach, namely that a classifier trained on all samples is applied at test time only to the samples routed to it~\citep{boostedNet, han2022Weighted, regol2024jointlylearned}. In early-exit networks, where an input leaves at the first exit whose confidence passes a threshold, \citet{boostedNet} train all exits with one summed loss in which each exit is boosted on top of the earlier ones, so that deeper exits concentrate on the samples that earlier exits find hard. \citet{han2022Weighted} keep this exit rule and meta-learn per-sample loss weights for each exit, which encourage each exit to train on the samples it would receive at test time. \citet{kubaty2025how} show that the schedule for training the backbone and the exit heads also changes the cost-accuracy trade-off, and they propose training the backbone before the whole multi-exit network. Other works also learn the routing rule, as \citet{regol2024jointlylearned} do by training exit gates together with the classifiers. Mixture-of-experts language models with experts of different sizes train a learned gate together with the experts~\citep{sun2024mixture, wang2025hmoe}, and \citet{wang2025hmoe} add a loss term that penalizes the use of larger experts. 

Two of these procedures alternate between deciding which samples each classifier should handle and retraining the classifiers. ~\citet{regol2024jointlylearned} alternate between fitting the gates to the lowest-cost exit of each sample and retraining the classifiers with each loss weighted by the predicted probability of exiting at that classifier, while ~\citet{han2022Weighted} alternate between assigning samples to exits by simulating confidence-based exiting and updating the classifiers with the resulting weights. ~\citet{han2022Weighted} also report that training each exit only on the samples allocated to it, without weights, causes a drastic drop in accuracy. 
However, these works do not provide theoretical guarantees that their joint training procedures yield models  achieving the  desired cost-accuracy trade-off.

\paragraph{LLM Domain-Specific Approaches.}

Recent interest in efficient LLM inference  has accelerated work on  test-time compute allocation, both within a single model and via multi-model orchestration. Within a single model, works on inference efficiency  include  early-exit networks and depth-adaptive mechanisms \citep{depth_adaptive_transformer_elbayad, consistent_ee_zeng, deja_vu_liu, chen2023eellm}. Works on  multi-model orchestration include  draft-verify strategies (speculative decoding) \citep{spec_decoding_survey, leviathan2023_spec_decoding}, one-shot routers that select a model per input \citep{chen2024routerdc, ong2025routellm}, and cascades that escalate through models of increasing cost \citep{chen2024frugalgpt, zhang2024treacle, valkanas2025cpo}. 
The closest  works to the  setting we consider are the latter two categories: LLM routers and cascades.  

An LLM router takes as input a query embedding produced by   a  small, low-cost language model, and uses it to select which LLM to invoke~\citep{chen2024routerdc, ong2025routellm}. 
Cascade approaches~\citep{chen2024frugalgpt, zhang2024treacle, valkanas2025cpo}
prompt LLMs sequentially from cheapest to most expensive,   deciding at each step whether to return the response of the current LLM or  escalate to a more expensive one. Both the routers and cascades are generally trained on top of fixed, pretrained LLMs. Unlike domain-agnostic approaches, LLM routers and   cascades are tailored  to LLM-specific signals such as self-consistency scores~\citep{mundler2023self} or human feedback~\citep{ong2025routellm}, which  serve as input to these components.

\subsection{Theoretical Works}\label{ssec:theoretical}

While empirical approaches  achieve inference cost savings in practice, they often lack rigorous  analysis. Alongside the empirical work, a theoretical literature has developed surrogate losses for various cost-sensitive classification problems, providing statistical guarantees such as excess-risk bounds and different notions of surrogate consistency. This literature can be separated into three main threads, each corresponding to a different problem formulation. 
The first two, reject-option classification and learning to defer, are closely related to the multi-model classification problem and are the most intensively studied. The third thread, more recently developed, addresses the multi-model classification problem itself, so far only with $T=2$ classifiers. Since our work sits in the third thread and the main text develops it in full, we focus on  the first two here.

\paragraph{Reject-Option Classification.} 
A reject-option classifier, i.e., a classifier with a reject option, may abstain  from predicting at a  cost lower than that of a misclassification  \citep{Chow1957, chow1970}. The Bayes-optimal reject decision  was characterized by~\citet{Chow1957, chow1970}. Subsequent work developed the consistency theory for it, establishing  excess-risk convergence rates for plug-in and empirical risk minimization (ERM) estimators in the binary classification case \citep{Herbei2006} and estimation error bounds that accommodate  arbitrary classification-calibrated losses \citep{cao2022generalizing}.  The same abstention mechanism is also studied under the name selective classification \citep{ElYaniv2010foundations, Wiener2011} and  rejection learning~\citep{cortes2016, Hendrickx2024}. 

\paragraph{Learning to Defer (L2D).}
Learning to defer generalizes abstention: instead of declining prediction outright, the classifier may defer to one or more \emph{experts} of varying  capabilities  and costs \citep{Madras2018, Keswani2021, Wilder2021}. This literature is close in spirit to multi-model classification problem, since a deferral decision routes the input to an alternative predictor.
\citet{Mozannar2020} proposed a consistent surrogate for multiclass L2D with a single fixed expert, obtained by a reduction to cost-sensitive learning, which yields a softmax cross-entropy loss over the label space augmented with a deferral option. \citet{Verma2022a} proposed an alternative one-vs-all surrogate for the same single-expert setting and proved that it is a consistent surrogate for the deferral loss.\citet{Verma2023} extended both the softmax and the one-vs-all surrogates to multiple experts.

All of these works established Bayes consistency, a guarantee about minimizers over the class of all measurable functions. This guarantee can be less informative when, in practice, the surrogate is minimized over a restricted hypothesis class $\mathcal{H}$ (for instance, linear functions). \citet{Long2013} illustrated this with linearly separable distributions, on which a Bayes-consistent surrogate minimized over linear functions keeps a positive classification error, while a surrogate that is not Bayes-consistent drives the error to zero. This motivated them to study consistency relative to a hypothesis class: $\mathcal{H}$-consistency and realizable $\mathcal{H}$-consistency.
A surrogate is $\mathcal{H}$-consistent if, for every distribution, minimizing the surrogate risk over $\mathcal{H}$ yields an error approaching the smallest error attainable within $\mathcal{H}$. This requirement also includes settings in which no function in $\mathcal{H}$ attains zero error, either because the underlying distribution is noisy or because $\mathcal{H}$ is too restricted. \citet{Long2013} therefore introduced the weaker notion of realizable $\mathcal{H}$-consistency, which imposes the same conditions, but only on distributions for which some function in $\mathcal{H}$ attains zero error, so that minimizing the surrogate over $\mathcal{H}$ must drive the classification error to zero. \citet{Awasthi2022} later introduced $\mathcal{H}$-consistency bounds, which upper-bound the target estimation error (i.e., the gap between the target risk and the smallest target risk attainable within $\mathcal{H}$) by a function of the surrogate estimation error for a given hypothesis set, and derived explicit bounds for binary classification. These guarantees were stated for
classification with the $0$-$1$ loss; the following works on learning to defer adapt them by replacing the classification error
with a deferral loss.

In learning to defer with a single expert, \citet{Mozannar2023} showed that prior Bayes-consistent surrogates can fail to find a good linear classifier and deferral function even when a pair with zero error exists, and they proposed a surrogate that is realizable $\mathcal{H}$-consistent for model classes closed under scaling.  \citet{mao2024a} showed that a family of surrogates that includes the one of \citet{Mozannar2023} up to a constant is realizable $\mathcal{H}$-consistent under mild conditions and that suitable members admit $\mathcal{H}$-consistency bounds.  This also shows that the surrogate of\citet{Mozannar2023} is Bayes-consistent when the expert cost is its classification error, a question that was left open.
For multiple pre-defined experts, \citet{mao2024b} proved $\mathcal{H}$-consistency bounds when the predictor and the deferral function are learned simultaneously, and \citet{mao2023} proved such bounds for a two-stage variant that keeps a pre-trained predictor fixed while the deferral function is learned. In both works the cost of deferring to an expert can add an inference cost to that expert's error. \citet{mao2024b} also instantiated their surrogate with a constrained hinge loss whose scores sum to zero and obtained an $\mathcal{H}$-consistency bound that is linear in the surrogate estimation error for distributions on which the best surrogate risk within $\mathcal{H}$ equals the Bayes surrogate risk. 
Finally,~\citet{mao2025mastering} proposed surrogates for learning to defer to multiple experts that are realizable $\mathcal{H}$-consistent, admit $\mathcal{H}$-consistency bounds, and are Bayes-consistent, in both the standard and the two-stage settings. To our knowledge, this is the most general consistency analysis of learning to defer.

\paragraph{Joint classifier-router training.} 

L2D   assumes   the expert(s) are fixed and learns a \emph{single}   classifier together with a deferral rule. Closest to our setting, \citet{regol2025acquisition} and \citet{pugnana2025} have studied joint training of \emph{two} classifiers with a cost-sensitive routing rule, but neither  extends to the general multi-model setting.\citet{regol2025acquisition} prove Bayes consistency, while~\citet{pugnana2025} prove realizable $\mathcal{H}$-consistency. In contrast, our framework allows an arbitrary number of  classifiers with both single-stage  and cascade routing, and like~\citet{regol2025acquisition}, we establish Bayes consistency.

\clearpage
\section{  Proofs for the Flat Router Model}\label{app:flat_all}
In this section, we prove the calibration result in Theorem~\ref{thm:calibration} for the flat-router hinge loss $\ell^\hinge_\Flat$. 
For this,  we first prove some auxiliary results  in Sections~\ref{app:bayes_opt_flat}--\ref{app:connection_hinge_01c_conditional_bayes_risk}, and based on these, we   prove calibration  in Section~\ref{app:proof_calibration}. 

\paragraph{Notation.} We use LHS (resp.\ RHS) as short for left (resp.\ right) hand side.
 We  write $p(\cdot \mid \cdot )$   for  the conditional  probability mass function of $Y$, e.g.,  $p(y \mid x)\equiv \mathrm{Pr}(Y=y\mid X=x)$ and $p(y\mid x, z_t)\equiv \mathrm{Pr} (Y=y\mid X=x,Z_t=z_t)$.

\subsection{Bayes-Optimal Solutions}
\label{app:bayes_opt_flat}
In this section, we characterize  the  Bayes-optimal solution   $\pi^*_\Flat\in \Pi_\Flat$ under the $0$-$1$-$c$ loss (Lemma~\ref{lem:bayes_opt_flat}) and the Bayes-optimal solution $\tpi^*_\Flat\in \wt{\Pi}_\Flat$ under the hinge loss (Lemma~\ref{lem:BO_solution_flat_hinge}).  
We first characterize the Bayes-optimal solution within the hypothesis class $\Pi_\Flat$ under the $0$-$1$-$c$ loss associated with the population risk $\mc{R}(\pi)=\E\big[\ell(X,Z_{1:T}, Y;\pi, c)\big]$. 
\begin{restatable}[Bayes-optimal solution in $\Pi_\Flat$ for $0$-$1$-$c$ loss]{lemma}{bayesoptimalflat}
\label{lem:bayes_opt_flat}
For  the $0$-$1$-$c$ risk $\mc{R}$  defined in \eqref{eq:def_R(f;c)} with cost vector $c$, there exists a Bayes-optimal solution    
\begin{align}\label{eq:def_f_star_flat}
\pi^*_\Flat=(f^*_{1:T}, r^*_\Flat) \in \argmin_{\pi \in \Pi_\Flat} \mc{R}(\pi)    
\end{align}
whose inducing components $(f_{1:T}^*, r^*_\Flat)$ satisfy:   
    \begin{align}
    f_t^*(x,z_t)
    &\in \argmax_{y\in[K]} p(y\mid x,z_t)\quad \forall \; t\in[T],
    \label{eq:f_star_multi_app}\\
    r^*_\Flat(x)&
    %\equiv r^*(x, z_{1:M})
    \in \argmin_{t\in[T]}  \mc{R}_t(x;f^*_t),
    \label{eq:r_star_multi_app}
\end{align}  
where $\mc{R}_t(x;f^*_t)$ denotes the flat stage-wise Bayes risk of $f_t^*$ conditioned on $X=x$:
\begin{align}
   \mc{R}_t(x;f^*_t)
    =
    1 - \mathbb{E}
        \left[ \max_{y\in[K]} p(y\mid x,Z_{t})\mid X=x\right]
      + c_{t}\quad \forall t\in[T].
    \label{eq:G_l_app}
\end{align}
\end{restatable}

\begin{proof}

We first prove the Bayes-optimal solution for the classifiers $f_{1:T}^*$ then the routing~function~$r^*_\Flat$. 
\paragraph{Part 1: Proof for the Bayes-optimal  $f^*_{1:T}$.} We first minimize $\mc{R}(\pi)\equiv \mc{R}(\pi;c)$ with respect to $f_{1:T}$ for an arbitrary routing rule  $r_\Flat : \mc{X}\to [T]$. The $0$-$1$-$c$ risk can be written as  
\begin{align}\label{eq:01c_risk_expanded_for_flat}
    \mc{R}(\pi)&=\E\Big[\ell(X,Z_{1:T}, Y; \pi, c)\Big] \nonumber\\
    &=\E\Big[\Ind\Big\lbrace f_{r_\Flat(X)}(X, Z_{r_\Flat(X)})\ne Y\Big\rbrace +c_{r_\Flat(X)}\Big]\nonumber\\
    &= \sum_{t=1}^T\E\Big[\Ind\{r_\Flat(X)=t\}\cdot \Big(\Ind\big\lbrace f_{t}(X,Z_t)\ne Y \big\rbrace +c_t\Big)\Big]\nonumber\\
    &= \sum_{t=1}^T\E\Bigg[\Ind\{r_\Flat(X)=t\}\cdot \underbrace{\Big(\E\Big[\Ind\big\lbrace f_{t}(X,Z_t)\ne Y \big\rbrace \mid X\Big]+c_t\Big)}_{=:\mc{R}_t(X; f_t)}\Bigg],
\end{align}
where the last step applies the tower property of conditional expectation, and  $\mc{R}_t(X;f_t)$ denotes the conditional risk of the $t$-th classifier given $X$; the semicolon here emphasizes that $X$ is a variable and $f_t$ is treated as a parameter. 

Since $c_t$ is a constant that does not depend on $f_{1:T}$, minimizing $\mc{R}(\pi)$ with respect to $f_t$ reduces to minimizing $\mc{R}_t(X; f_t)$ pointwise in $X$. Fixing $x\in\mc{X}$, by the tower property, we have 
\begin{align}
   \mc{R}_t(x; f_t)&:= \E\Big[ \Ind\big\lbrace f_{t}(X,Z_t)\ne Y \big\rbrace\mid X=x\Big]+c_t\nonumber\\
    &=\E\Big[\E\big[ \Ind\big\lbrace f_{t}(X,Z_t)\ne Y \big\rbrace\mid X=x, Z_t\big]\mid X=x\Big]+c_t\nonumber\\
    &=\E\Big[\mathbb{P}\big( Y\ne f_t(x, Z_t)\mid X=x, Z_t\big)\mid X=x\Big]+c_t\nonumber\\
    &=1-\E\Big[p(f_t(x,Z_t)\mid x,Z_t)\mid X=x\Big]+c_t,\label{eq:conditional_risk_flat_ft_given_x}
\end{align}
where $p(\cdot \mid x,Z_t)$ denotes the conditional probability of $Y$ given $X=x$ and $Z_t$.

Therefore, minimizing $\mc{R}_t(x;f_t)$ with respect to $f_t$ reduces to maximizing $p(f_t(x,z_t)\mid x,z_t)$ for each  $(x,z_t)\in\mc{X}\times \mc{Z}_t$, so pointwise in $(x,z_t)\in\mc{X}\times \mc{Z}_t$, we have 
\begin{align}
f_t^*(x,z_t) \in 
%\argmin_{y\in[K]}1-p(y\mid x, z_t)+c_t= 
\argmax_{y\in[K]} p(y\mid x,z_t),\label{eq:ft_star_flat_proof}
\end{align}
which proves \eqref{eq:f_star_multi_app}. 
\paragraph{Part 2: Proof for the Bayes-optimal   $r^*_\Flat$.} We now minimize $\mc{R}(\pi)\equiv \mc{R}(\pi;c)$ with respect to  $r_\Flat$ holding  $f_{1:T} =f^*_{1:T}$   fixed. Substituting \eqref{eq:ft_star_flat_proof} into \eqref{eq:conditional_risk_flat_ft_given_x}, the conditional Bayes  risk  of the $t$-th classifier is
\begin{align}
\mc{R}_t(x;f^*_t)&:=\E\Big[\Ind\big\lbrace f^*_{t}(X,Z_t)\ne Y \big\rbrace \mid X=x\Big]+c_t\nonumber \\
    &= 1 -\mathbb{E}
       \left[
          \max_{y\in[K]} p(y\mid x,Z_t)
       \mid X=x\right]
       + c_t,
\end{align}
yielding \eqref{eq:G_l_app}.
Since $\pi=(f^*_{1:T}, r_\Flat)$ here, we write $\mc{R}(f^*_{1:T}, r_\Flat)$ to mean $\mc{R}(\pi)$. From \eqref{eq:01c_risk_expanded_for_flat},   the unconditional risk is  $\mc{R}(f^*_{1:T}, r_\Flat)=\sum_{t=1}^T\E[\Ind\{r_\Flat(X)=t\}\cdot \mc{R}_t(X; f_t^*)]$, so minimizing $\mc{R}(f^*_{1:T}, r_\Flat)$ with respect to $r_\Flat$ reduces to a pointwise problem in $X$.  For each  $x\in\mc{X}$, the optimal routing function selects the classifier with the minimum conditional risk, that is,
\begin{align}
r^*_\Flat(x) \in \argmin_{t\in[T]} \mc{R}_t(x;f^*_t),
\end{align}
which proves \eqref{eq:r_star_multi_app}.

\end{proof}

We now characterize the Bayes-optimal solution within the hypothesis class $\wt\Pi_\Flat$ under the   surrogate hinge loss $\ell_\Flat^\hinge$ associated with the population risk $\mc{R}_\Flat^\hinge(\tpi)=\E\big[\ell_\Flat^\hinge(X, Z_{1:T}, Y;\tpi, c)\big]$.

\begin{lemma}[Bayes-optimal solution in $\wt{\Pi}_\Flat$ for hinge loss]\label{lem:BO_solution_flat_hinge}There exists a Bayes-optimal solution 
\begin{align*}
    \tpi_{\Flat}^*=(\tf_{1:T}^*, \tr^*_\Flat)\in \argmin_{\tpi\in \wt{\Pi}_\Flat} \mc{R}_\Flat^\hinge (\tpi),
\end{align*}
where, with $f^*_t$ defined as in \eqref{eq:f_star_multi_app}, 
\begin{align}\label{eq:f_star_tilde_flat}
\left[\tf^*_t(x,z_t)\right]_k=\begin{cases}
    1& \text{for}\;\; k =f^*_t(x, z_t)\\
    %=\argmax_{y\in[K]}p(y\mid x, z_t)\\
    -\tfrac{1}{K-1} & \text{otherwise}
\end{cases}\quad  \forall   t\in[T]\,,
\end{align}
and \begin{align}\label{eq:r_star_tilde_flat}
&\tr^*_\Flat(x)
\in\left\lbrace v\in \Delta^{T-1}: v_k=0 \quad \forall k\notin \argmin_{t\in[T]} \mc{R}_t(x; f_t^*) \right\rbrace    \,.
\end{align}
    
\end{lemma}
\begin{proof}
    We prove \eqref{eq:f_star_tilde_flat} and \eqref{eq:r_star_tilde_flat} one at a time. 

\paragraph{Part 1:  Proof for the Bayes-optimal  $\tf^*_{1:T}$.}The result in   \eqref{eq:f_star_tilde_flat} directly follows from   \cite[Lemma~1]{tarigan2008}.

\paragraph{Part 2: Proof for the Bayes-optimal $\tr_\Flat^*$.} We now minimize $\mc{R}_\Flat^\hinge(\tpi)$ w.r.t.\ $\tr_\Flat$ holding $\tf_{1:T}=\tf^*_{1:T}$  fixed.  
By definition the overall flat-router hinge risk  is
\begin{align*}
\mc{R}_\Flat^\hinge(\tpi)=\E\Bigg[\sum_{t=1}^T [\tr_\Flat(X)]_t\cdot \ell^\hinge(\tf_t(X,Z_t), Y; c_t)\Bigg]
\end{align*}
Using $\tf_{1:T} = \tf_{1:T}^*$ and  writing $\tilde{\pi}=(\tf^*_{1:T}, \tr_\Flat)$, we have via  Lemma~\ref{lem:auxiliary_ft_tilde_star} that 
\begin{align}
\mc{R}_t^\hinge(x, z_t; \tf_t^*) &: = \E\Big[\ell^\hinge(\tf_t^*(X,Z_t), Y; c_t)\mid X=x, Z_t=z_t\Big]\nonumber\\
&=\tfrac{K}{K-1}\mc{R}_t(x,z_t; f_t^*) \,,
%\left(1-\max_{y\in[K]} p(y\mid x,z_t) + \kappa_t\right),
\end{align}
which implies that the overall hinge risk conditioned on $X=x$ 
can be written as  
\begin{align}
    \mc{R}^\hinge(x; \tpi)&:=\sum_{t=1}^T[\tr_\Flat(x)]_t\cdot \E \left[\mc{R}_t^\hinge(x, Z_t; \tf_t^*)\mid X=x\right]\nonumber\\
    &= \tfrac{K}{K-1}
    \sum_{t=1}^T
[\tr_\Flat(x)]_t\cdot \E\left[\mc{R}_t(x,Z_t; f_t^*)\mid X=x\right]\nonumber\\
%     \tfrac{K}{K-1}
%     \sum_{t=1}^T
% [\tr_\Flat(x)]_t\cdot \E\left[1-\max_{y\in[K]} p(y\mid x,Z_t) + \kappa_t\mid X=x\right]\nonumber\\
& = \tfrac{K}{K-1}\sum_{t=1}^T [\tr_\Flat(x)]_t\cdot \mc{R}_t(x; f_t^*)\,,
\end{align}
where the last step holds by definition of $\mc{R}_t(x;f_t^*)$ in \eqref{eq:G_l_app}. 
Therefore  any minimizer  $\tr_\Flat^*(x)\in \Delta^{T-1}$ of  $\mc{R}^\hinge( x; \tpi)$ must be supported on the set  $\argmin_{t\in[T]} \mc{R}_t(x; f_t^*)$, establishing \eqref{eq:r_star_tilde_flat}. In particular, any one-hot vector concentrated on any element of this set is a valid choice.
\end{proof}

\subsection{Relationship Between $0$-$1$-$c$ and Hinge Conditional Risks}\label{app:connection_hinge_01c_conditional_bayes_risk}
Using the characterization of the Bayes-optimal solutions from Lemmas~\ref{lem:bayes_opt_flat}--\ref{lem:BO_solution_flat_hinge}, the next two lemmas connect the hinge and $0$-$1$-$c$ conditional Bayes risks, and the hinge and $0$-$1$-$c$ excess risks, respectively, which will  be handy for our calibration  proof in the next section. 
\begin{lemma}[Hinge and $0$-$1$-$c$ conditional Bayes risks   differ by a factor of $\frac{K}{K-1}$]\label{lem:auxiliary_ft_tilde_star} 
For $t\in[T]$, let $ f_t^*$
%$: \mc{X}\times \mc{Z}_t\to [K]$
and $ \tf^*_t$
%:\mc{X}\times\mc{Z}_t\to \R^K$ 
denote the  Bayes-optimal  classifiers from Lemmas~\ref{lem:bayes_opt_flat}--\ref{lem:BO_solution_flat_hinge} 
%in $\Pi_\Flat$ and $\wt{\Pi}_\Flat$ 
under the $0$-$1$-$c$ loss and the hinge loss, respectively.  
Define the conditional Bayes risks \begin{align}\label{eq:def_Rt_x_zt}
   \mc{R}_t (x,z_t;  f_t^*)&:= \ex\Big[\Ind\big\{f_t^*(X, Z_t)\ne Y\big\}+ c_t \mid X=x, Z_t=z_t\Big]
   %&=\tfrac{K}{K-1} \left(1-\max_{y\in[K]} p(y\mid x,z_t) + c_t\right) 
   ,\\
\mc{R}_t^\hinge(x,z_t; \tf_t^*)&:= \ex\Big[\ell^\hinge(\tf^*_t(X,Z_t), Y;c_t)\mid X=x, Z_t=z_t\Big],
\label{eq:def_Rt_hinge_x_zt}
   %&=\tfrac{K}{K-1} \left(1-\max_{y\in[K]} p(y\mid x,z_t) + c_t\right) 
 \end{align}
then  
 % \begin{align}
 % &\mc{R}_t(x,z_t; f^*_t)=
 %  1-\max_{y\in[K]} p(y\mid x,z_t) + c_t .
 % \end{align}
 \begin{align}
     \mc{R}_t^\hinge(x,z_t; \tf_t^*)=\tfrac{K}{K-1} \mc{R}_t(x,z_t; f^*_t).\label{eq:stage_t_hinge_risk_BO}
 \end{align}
\end{lemma}

\begin{proof}
%Let $y_\MAP (x,z_t)\in \argmax_{y\in[K]} p(y\mid x,z_t)$.  Then 
We can write
    \begin{align}
 \mc{R}_t^\hinge(x,z_t; \tf_t^*) &  := \ex \left[\ell^\hinge(\tf^*_t(X,Z_t), Y;c_t)\mid X=x, Z_t=z_t\right]\nonumber\\
  &   =\sum_{y\in[K]}p(y\mid x, z_t)  \ell^\hinge(\tf^*_t(x,z_t), y;c_t) \nonumber\\
  &=p\big(f_t^* (x,z_t)\mid x, z_t\big) \tfrac{K}{K-1}c_t+  \Big(1-p\big(f_t^* (x,z_t)\mid x, z_t\big)\Big)\tfrac{K}{K-1}(1+c_t)\nonumber\\
&=\tfrac{K}{K-1}\Big(1-  p(f^*_t(x,z_t)\mid x,z_t) + c_t\Big) \nonumber\\
% &=\tfrac{K}{K-1} \left(1-\max_{y\in[K]} p(y\mid x,z_t) + c_t\right) .
&=\tfrac{K}{K-1} \Big(1-\max_{y\in[K]} p(y\mid x,z_t) + c_t\Big)=\tfrac{K}{K-1}\mc{R}_t(x, z_t; f_t^*). \label{eq:Rt_ft_star_Rt_tft_star}
\end{align} 
\end{proof}

\begin{lemma}[Calibration of the stage-wise cost-sensitive hinge loss $\ell^\hinge$]\label{lem:K_mistake}

Let $\tilde{f}_t$ be a soft classifier as defined in 
%as in $\widetilde{\Pi}_{\Flat}$ (see Eq.~
\eqref{eq:flat_space} and $f_t$  its induced hard classifier in \eqref{eq:hard_ft_from_soft_ft}. Define the  conditional risks  $\mc{R}_t(x, z_t; f_t)$ and $\mc{R}^{\hinge}_t(x,z_t; \tilde{f}_t)$ analogously to \eqref{eq:def_Rt_x_zt}--\eqref{eq:def_Rt_hinge_x_zt}. 
%as in Lemma~\ref{lem:auxiliary_ft_tilde_star}, with $f^*_t$ and $\tilde{f}_t^*$ replaced by $f_t$ and $\tilde{f}_t$.
Then, for all $t \in [T]$ and $(x, z_t) \in \mathcal{X} \times \mathcal{Z}_t$, 
\begin{align}
\mc{R}_t(x,z_t; f_t)- \mc{R}_{t}( x, z_t; f_{t}^*)\leq (K-1) \left( \mc{R}^{\hinge}_t(x, z_t; \tilde{f}_t)- \mc{R}^{\hinge}_{t}( x, z_t; \tilde{f}_{t}^*) \right)  .\label{eq:stagewise_calibration_hinge}
\end{align}
\end{lemma}
As an immediate remark, since  \eqref{eq:stagewise_calibration_hinge} holds for every $z_t$, we can trivially obtain the same bound for the risk conditioned on $X=x$ only. Likewise, in our proofs for the cascade model in Appendix~\ref{app:cascade_all}, \eqref{eq:stagewise_calibration_hinge} also holds with $(x, z_t)$ replaced by $h_t$, since   $p(y\mid h_t)=p(y\mid x, z_t)$ under the conditional independence assumption~\eqref{eq:conditional_indep} and both $f_t$ and $\tilde{f}_t$ depend on $h_t$ only through $(x, z_t)$.
\begin{proof} We begin by setting up some notation. Fix $(x,z_t)$ and let $p_y := p(y \mid x, z_t)$. Let  $v:= \tilde{f}_t(x, z_t)\in \mathbb{R}^K$ be the   score vector that $\tf_t$ outputs and $j:= f_t(x, z_t)\in [K]$ the class index that $f_t$ outputs. By the argmax map in \eqref{eq:hard_ft_from_soft_ft}, we have $  j\in \argmax_{k \in [K]} v_k$. 
%be the predicted class of $f_t, \tilde{f}_t$. 
Using the definition of the hinge loss $\ell^\hinge$ in \eqref{eq:stage_wise_surrogate}, we define  its cost-free variant  $\bar{\ell}^\hinge$   as  
\begin{align}
    \bar{\ell}^{\hinge}(v,y) := \ell^\hinge(v, y; 0) =\sum_{k \in [K] \setminus y} \left[ v_k+ \tfrac{1}{K-1} \right]_+ \,. 
\end{align}
We further define two useful shorthand: 
$L(v)$ as the associated  conditional risk, 
%expectation over $Y$, 
and $U$ as the total penalty:
\begin{align}
    L(v) &:=\E\left[\bar{\ell}^\hinge (v, Y)\mid X=x, Z_t=z_t\right]=\sum_{y \in [K]} p_y \bar{\ell}^{\hinge}(v,y), \label{eq:def_L(v)_single_stage_proof} \\
    U &:= \sum_{k \in [K]}  \left[ v_k+ \tfrac{1}{K-1} \right]_+.\label{eq:def_U}
\end{align}
We first express both excess risks in~\eqref{eq:stagewise_calibration_hinge} in terms of $p_y$ and $L(v)$, then derive a lower bound for the conditional expectation $L(v)$ to obtain the final inequality.
\paragraph{Part 1. Rewriting the stage-wise excess risks:}
By definition, we can write
\begin{align}
     \mc{R}^{\hinge}_t(x, z_t; \tilde{f}_t) &:= \ex\left[ \ell^{\hinge}(\tilde{f}_t(X, Z_t), Y; c_t)\mid X=x, Z_t=z_t\right] \nonumber\\
     & = \sum_{y \in [K]} p_y \ell^{\hinge}(v,y;c_t) = L(v) +\tfrac{K}{K-1}c_t.
\end{align}
Combined with $ \mc{R}^{\hinge}_{t}( x, z_t; \tilde{f}_{t}^*) = \tfrac{K}{K-1}\left(1 - \max_{y \in [K]}p_y\right) + \tfrac{K}{K-1}c_t$ from \eqref{eq:Rt_ft_star_Rt_tft_star}, the cost terms cancel and the excess stage-wise hinge risk can be written into 
\begin{align}
    \mc{R}^{\hinge}_t(x, z_t; \tilde{f}_t)- \mc{R}^{\hinge}_{t}( x, z_t; \tilde{f}_{t}^*) = L(v)  - \tfrac{K}{K-1}\Big(1 - \max_{y \in [K]}p_y\Big).\label{eq:stagewise_hinge_cal_excess}
\end{align}
Similarly, for the stage-wise $0$-$1$-$c$ risk, 
since $f_t$ predicts label $j$, we can write
\begin{align}
\mc{R}_t(x, z_t; f_t) = \ex\left[ \Ind\{f_t(X, Z_t) \neq Y\} +c_t \mid X=x, Z_t=z_t\right] = 1-p_j +c_t  . 
\end{align}
Combined with  
%$f^*_t$ predicts the most likely label,  we have 
$\mc{R}_t(x, z_t; f^*_t) = 1 - \max_{y \in [K]} p_y +c_t$ since $f^*_t$ predicts the most likely label, we can write the excess $0$-$1$-$c$ stage-wise risk into
\begin{align}
    \mc{R}_t(x, z_t; f_t)  - \mc{R}_t(x, z_t; f^*_t)  = \max_{y \in [K]} p_y - p_j .\label{eq:stagewise_01c_cal_excess}
\end{align}
In the rest of the proof, we provide a lower bound for $L(v)$, which will be integrated into \eqref{eq:stagewise_hinge_cal_excess} and combined with \eqref{eq:stagewise_01c_cal_excess} to prove the final calibration inequality \eqref{eq:stagewise_calibration_hinge}. 

\paragraph{Part 2. Lower bounding $L(v)$:}
%We can now derive a lower bound for $L(v)$ in terms of $p_y$ and $v_j$.
Using the definition of $L(v)$ in \eqref{eq:def_L(v)_single_stage_proof},  we can rewrite it as follows
\begin{align}
    L(v) &= \sum_{y \in [K]} p_y \sum_{k \in [K] \setminus y} \left[ v_k+ \tfrac{1}{K-1} \right]_+  \stackrel{\text{(i)}} {=} \sum_{k \in [K] } (1-p_k) \left[ v_k+ \tfrac{1}{K-1} \right]_+ \stackrel{\text{(ii)}}{=}   U - \sum_{k \in [K] } p_k \left[ v_k+ \tfrac{1}{K-1} \right]_+ \nonumber\\
    &\stackrel{\text{(iii)}} {=} U -  p_j \left[ v_j+ \tfrac{1}{K-1} \right]_+  - \sum_{k \in [K] \setminus j } p_k \left[ v_k+ \tfrac{1}{K-1} \right]_+ \,,
 \end{align}
 where step (i) follows from exchanging the two summations, step (ii) applies the definition of $U$ from \eqref{eq:def_U}, and step (iii) 
 simply splits  the $j$-th term from the summation. Using this, we now~lower~bound~$L(v)$:
   \begin{align}
  L(v)  &= U -  p_j \left[ v_j+ \tfrac{1}{K-1} \right]_+  - \sum_{k \in [K] \setminus j } p_k \left[ v_k+ \tfrac{1}{K-1} \right]_+ \nonumber\\ 
  & \geq U -  p_j \left[ v_j+ \tfrac{1}{K-1} \right]_+  - \max_{y\in [K]} p_y  \left(U - \left[ v_j+ \tfrac{1}{K-1} \right]_+\right) \nonumber\\
   & = U\left(1 - \max_{y\in [K]} p_y \right)  + \left[ v_j+ \tfrac{1}{K-1} \right]_+\left( \max_{y\in [K]} p_y  - p_j\right) \,,\label{eq:L(v)_LB_stage_wise}
\end{align}
where the inequality   holds by applying  $p_k \leq \max_{y \in [K]} p_y$ and the definition of $U$ from \eqref{eq:def_U}.

\paragraph{Part 3. Combining all components:}
%Two auxiliary inequalities: }
To combine all components of the proof, we   begin with  two auxiliary inequalities about $U$ and the hinge terms. First, 
  $U$ is lower bounded by $\frac{K}{K-1}$ because
\begin{align}
    U &=  \sum_{k \in [K]}  \left[ v_k+ \tfrac{1}{K-1} \right]_+  \stackrel{\text{(i)}}{\ge} \sum_{k \in [K]}  \left( v_k+ \tfrac{1}{K-1}\right) \stackrel{\text{(ii)}}{=}  \tfrac{K}{K-1} \,,
    \label{eq:u_ineq}
\end{align}
where step (i) uses $[x]_+ \geq x$, and step (ii) uses $\sum_{k \in [K]} v_k = 0$.
Second, since the predicted label $j \in \argmax_{k \in [K]} v_k$, we have $v_j \ge v_k$ for all $k\in[K]$, and so 
\begin{align}
    \left[ v_j+ \tfrac{1}{K-1} \right]_+ = \tfrac{1}{K}\sum_{k \in [K]} \left[ v_j+ \tfrac{1}{K-1} \right]_+  \ge \tfrac{1}{K}\sum_{k \in [K]} \left[ v_k+ \tfrac{1}{K-1} \right]_+  = \tfrac{1}{K}\cdot U \ge \tfrac{1}{K-1} \label{eq:vj_ineq}\,,
\end{align}
where 
the last inequality applies the previous inequality~\eqref{eq:u_ineq}.

Equipped with the lower bound in \eqref{eq:L(v)_LB_stage_wise} for $L(v)$ and the auxiliary inequalities \eqref{eq:u_ineq}--\eqref{eq:vj_ineq}, we can now lower bound the excess stage-wise hinge risk from \eqref{eq:stagewise_hinge_cal_excess} as follows:
\begin{align}
  &\mc{R}^{\hinge}_t(x, z_t; \tilde{f}_t)- \mc{R}^{\hinge}_{t}( x, z_t; \tilde{f}_{t}^*)\nonumber\\
  &= L(v)  - \tfrac{K}{K-1}\left(1 - \max_{y \in [K]}p_y\right) \nonumber\\ &\stackrel{\text{(i)}}{\geq} U\left(1 - \max_{y\in [K]} p_y \right)  + \left[ v_j+ \tfrac{1}{K-1} \right]_+\left( \max_{y\in [K]} p_y  - p_j\right)  - \tfrac{K}{K-1}\left(1 - \max_{y \in [K]}p_y\right)\nonumber\\
   &= \underbrace{\left(U - \tfrac{K}{K-1}\right)}_{\geq 0 \text{ by \eqref{eq:u_ineq}}} \underbrace{\left(1 - \max_{y\in [K]} p_y \right)}_{\geq 0}  + \left[ v_j+ \tfrac{1}{K-1} \right]_+\left( \max_{y\in [K]} p_y  - p_j\right)\nonumber \\
   &\stackrel{\text{(ii)}}{\geq}  \left[ v_j+ \tfrac{1}{K-1} \right]_+ \left( \max_{y\in [K]} p_y  - p_j\right)\nonumber \\
   &\stackrel{\text{(iii)}}{\geq}  \tfrac{1}{K-1}\left( \max_{y\in [K]} p_y  - p_j\right) \nonumber \\
   & = \tfrac{1}{K-1}\left(\mc{R}_t(x,z_t; f_t)- \mc{R}_{t}( x, z_t; f_{t}^*)\right)
\end{align}
where   (i) applies \eqref{eq:L(v)_LB_stage_wise},   (ii)  holds by dropping the non-negative term $\left(U \!-\! \tfrac{K}{K-1}\right)  \! \left(1 \!-\! \max_{y\in [K]} p_y \right)$,  
   (iii) uses  \eqref{eq:vj_ineq},  and the last step   uses  the expression of the excess $0$-$1$-$c$ risk from
\eqref{eq:stagewise_01c_cal_excess}. 
Multiplying both sides by $K-1$ proves our final calibration inequality  \eqref{eq:stagewise_calibration_hinge}.
\end{proof}

\subsection{Proof of Theorem~\ref{thm:calibration}}\label{app:proof_calibration}

Using the results in Sections~\ref{app:bayes_opt_flat}--\ref{app:connection_hinge_01c_conditional_bayes_risk},   
we prove our  main calibration result. We first restate the theorem. 

\excessriskflat*

\begin{proof}
First we set up some notation. Let 
\begin{align*}
    \tpi_\Flat=(\tf_{1:T}, \tr_\Flat)\in \wt{\Pi}_\Flat \qquad \text{and}\qquad \pi_\Flat=(f_{1:T}, r_\Flat)=\varphi(\tpi_\Flat)\in \Pi_\Flat
\end{align*} denote  an arbitrary  soft predictor and  its hard counterpart. Let
\begin{align*}
\tpi^*_\Flat=(\tf^*_{1:T}, \tr^*_\Flat)\in \argmin_{\tpi \in  \wt{\Pi}_\Flat}\mc{R}_\Flat^\hinge (\tpi)\qquad \text{and}\qquad     
\pi^*_\Flat=(f_{1:T}^*, r_\Flat^*)\in \argmin_{ \pi \in \Pi_\Flat} \mc{R}(\pi )
\end{align*}
  be the hinge risk minimizer  of Lemma~\ref{lem:BO_solution_flat_hinge}, and    the $0$-$1$-$c$ risk minimizer of Lemma~\ref{lem:bayes_opt_flat}. 
  %as in \eqref{eq:def_f_star_flat}.

 For a given input $x\in\mc{X}$, we use 
 \begin{align}
 t:=r_\Flat(x)\in \argmax_{\tau\in[T]}[\tr_\Flat(x)]_\tau \qquad \text{and} \qquad t^*:=r^*_\Flat(x)\in  \argmin_{\tau\in[T]}\mc{R}_\tau(x; f^*_\tau)\label{eq:def_t_t_star_flat} 
 \end{align}
  to denote the index of the  classifier selected by the hard predictor $\pi_\Flat$ and the Bayes-optimal predictor $\pi^*_\Flat$, respectively.
 
 Hence, to prove \eqref{eq:flat_calibration}, it suffices to prove pointwise in $x\in\mc{X}$ that 
\begin{align}\label{eq:calibration_x}
    \mc{R}_t(x; f_t) - \mc{R}_{t^*}(x;f^*_{t^*}) \le T(K-1)\left(\mc{R}^\hinge_\Flat( x;\tpi) - \mc{R}_\Flat^\hinge (x; \tpi^*)  \right),
\end{align}
where each term here denotes  conditional risks, conditioned on $X=x$, since taking total expectation on both sides yields the desired result. 
We prove \eqref{eq:calibration_x} in the rest of this proof.

\paragraph{Part 1: Decomposition of the conditional  excess  $0$-$1$-$c$ risk.}
We now provide general definitions that we used to define  the terms on the LHS of \eqref{eq:calibration_x}. For $\tau\in[T]$, define in line with \eqref{eq:cond_R_def}, 
\begin{align}
\mc{R}_\tau( x; f_\tau) &\coloneqq \ex\left[\Ind\{f_\tau(x,Z_\tau)\neq Y\}\mid X=x\right]+\kappa_\tau,\\
\mc{R}_\tau( x; f^*_{\tau}) &\coloneqq \ex\left[\Ind\{f_{\tau}^*(x,Z_{\tau})\neq Y\}\mid X=x\right]+\kappa_{\tau}\nonumber\\
&\stackrel{\text{(i)}}{=} 1-\ex\left[\max_{y\in[K]} p(y\mid x,Z_{\tau})\mid X=x\right]+\kappa_{\tau}
%\stackrel{\text{(ii)}}{=}
% \mc{R}_{\tau}(x; f_\tau^*)
,\label{eq:R_tau_cond_x}
\end{align}
where (i) applies the expression for $f_{\tau}^*$ from  \eqref{eq:f_star_multi_app}. 
Combining these, we can write 
\begin{align}
    &\mc{R}_t( x; f_t)-\mc{R}_{t^*}(x; f^*_{t^*})
    = \underbrace{\mc{R}_t( x; f_t) - \mc{R}_t( x; f^*_t)}_{=: \Delta_\internal(x)}+ \underbrace{\mc{R}_t( x; f^*_t)-\mc{R}_{t^*}( x; f^*_{t^*})}_{=:\Delta_\select(x)}
     =\Delta_\internal(x)+ \Delta_\select(x)
   \,.\label{eq:decompose_risk_into_internal_select}
\end{align}
Intuitively, $\Delta_\internal(x)$  denotes the excess risk caused by using a suboptimal classifier as the $t$-th classifier, so it is an excess risk \emph{internal} to the $t$-th stage. Note  $\Delta_\internal(x)\ge0$ because $f^*_t$ is the Bayes-optimal $t$-th classifier; see the explanation around \eqref{eq:conditional_risk_flat_ft_given_x}--\eqref{eq:ft_star_flat_proof}. The $\Delta_{\select}(x)$ term is the excess risk caused by \emph{selecting} a suboptimal exit stage  $t$, and $\Delta_\select(x)\ge 0 $ by the definition of $t^*$ in \eqref{eq:def_t_t_star_flat}. 

\paragraph{Part 2: Decomposition of the conditional  excess hinge risk.}
We can similarly decompose the conditional   excess hinge  risk. To begin, define   the stage-$\tau$ conditional  hinge risk given $X=x$ as
\begin{align}\label{eq:R_tau_hinge_shorthand}
\mc{R}^{\hinge}_\tau(x; \tf_\tau)\coloneqq \E\left[\ell^\hinge(\tf_\tau(x,Z_\tau), Y;\kappa_\tau) \mid X=x\right] \quad \text{for} \quad \tau\in[T]\,.    
\end{align}
Evaluated at the Bayes-optimal solution, this is linked to  the  stage-$\tau$ conditional $0$-$1$-$c$ risk via
\begin{align}
\mc{R}_\tau^\hinge( x; \tf_\tau^*)
&\stackrel{\text{(i)}}{=}\E\big[\mc{R}_\tau^{\hinge}(x, Z_{\tau}; \tf_\tau^*)\mid X=x\big] \nonumber\\
& \stackrel{\text{(ii)}}{=}\tfrac{K}{K-1} \E\big[\mc{R}_\tau (x,Z_\tau; f_\tau^*)\mid X=x\big]
\stackrel{\text{(iii)}}{=}\tfrac{K}{K-1}\mc{R}_\tau(x; f_\tau^*)\quad \forall \tau\in[T]\,,\label{eq:R_hinge_ftau}
\end{align}
where  (i) and (iii) apply the tower property  with respect to $P_{Z_\tau\mid X}(\cdot \mid x)$, and (ii)  applies  \eqref{eq:stage_t_hinge_risk_BO}. This result will become useful shortly.

 Note that  Lemma~\ref{lem:BO_solution_flat_hinge} proves that the support of any Bayes-optimal $\tr_\Flat^*(x)$ is contained in $\argmin_{t\in[T]}\mc{R}_t(x; f_t^*)$, so the discrete router $r_\Flat^*(x)$ from Lemma~\ref{lem:bayes_opt_flat} satisfies 
 \begin{align}\label{eq:r_soft_flat_link_t_star}
  t^*:=r^*_\Flat(x)\in \argmin_{t\in[T]} \mc{R}_t(x; f_t^*) \supseteq
  \argmax_{t\in[T]} \big[\tr_\Flat^*(x)\big]_t\,  .
 \end{align}
 Hence, using \eqref{eq:R_tau_hinge_shorthand} and \eqref{eq:r_soft_flat_link_t_star} and   letting $\mc{R}_\Flat^\hinge( x; \tpi^*)$ denote  the conditional flat-router hinge risk, conditional on $X=x$, 
 we can write  
\begin{align}
\mc{R}_\Flat^\hinge( x; \tpi^*)&:=  \sum_{\tau=1}^T\big[\tr_\Flat^*(x)\big]_{\tau} \mc{R}_\tau^\hinge(x; \tf_\tau^*)\nonumber\\
&\stackrel{\text{(i)}}{=}\sum_{\tau=1}^T\big[\tr_\Flat^*(x)\big]_{\tau}\tfrac{K}{K-1} \mc{R}_\tau(x; f^*_\tau)
\stackrel{\text{(ii)}}{=} \tfrac{K}{K-1}\mc{R}_{t^*}(x; f_{t^*}^*) \stackrel{\text{(iii)}}{=}\mc{R}^\hinge_{t^*}( x; \tf_{t^*}^*) \label{eq:Rf_star_link},
\end{align}
where steps (i) and (iii) hold by  \eqref{eq:R_hinge_ftau}, and step (ii)   uses the fact that $\mc{R}_\tau (x; f_\tau^*)=\mc{R}_{t^*} (x; f_{t^*}^*)$ for all $\tau\in \argmin_{t\in[T]} \mc{R}_t(x; f_t^*)$. 
Using \eqref{eq:R_tau_hinge_shorthand} and  \eqref{eq:Rf_star_link}, we can therefore write 
\begin{align}
&\mc{R}_\Flat^\hinge( x; \tpi)-
\mc{R}_\Flat^\hinge( x; \tpi^*)\nonumber\\
&=
\sum_{\tau=1}^T [\tr_\Flat(x)]_{\tau}\mc{R}_\tau^\hinge(x;\tf_\tau)-  \mc{R}^\hinge_{t^*}(x; \tf_{t^*}^*)\nonumber\\
&=\sum_{\tau=1}^T [\tilde r_\Flat(x)]_\tau\left( \mc{R}^\hinge_\tau( x; \tf_\tau) -\mc{R}^\hinge_\tau(x; \tf_\tau^*)+\mc{R}^\hinge_\tau( x; \tf_\tau^*)
-  \mc{R}^\hinge_{t^*}(x; \tf_{t^*}^*) \right)\nonumber\\
&= \sum_{\tau=1}^T  [\tr_\Flat(x)]_\tau\underbrace{\left( \mc{R}^\hinge_\tau(x;\tf_\tau) -\mc{R}^\hinge_\tau( x; \tf_\tau^*)\right)}_{=:\Delta_\internal^\hinge(x, \tau)}+\sum_{\tau=1}^T [\tr_\Flat(x)]_\tau\underbrace{\left(\mc{R}^\hinge_\tau( x; \tf_\tau^*)
-  \mc{R}^\hinge_{t^*}(x; \tf_{t^*}^*)\right)}_{=:\Delta_\select^\hinge(x, \tau)}\nonumber\\
& = \sum_{\tau=1}^T[\tr_\Flat(x)]_\tau\Delta_\internal^\hinge(x,\tau)+ \sum_{\tau=1}^T[\tr_\Flat(x)]_\tau\Delta_\select^\hinge(x,\tau)\,,
\end{align}
where   we have $\Delta^\hinge_\internal(x, \tau)\ge 0 $ by  applying \cite[Lemma 1]{tarigan2008} to the  Bayes-optimal soft classifier  $\tf_\tau^*$, and $\Delta_\select^\hinge(x, \tau) = \frac{K}{K-1}\left( \mc{R}_\tau (x, f_\tau^*) -  \mc{R}_{t^*} (x, f_{t^*}^*)\right)\ge 0 $ by \eqref{eq:R_hinge_ftau} and the definition of $t^*$ in \eqref{eq:def_t_t_star_flat}.

\paragraph{Part 3: Upper bounding $\Delta_\internal(x)$ and  $\Delta_\select(x)$ by  $\Delta_\internal^\hinge(x,\tau)$ and $\Delta_\select^\hinge(x,\tau)$.}
  
By Lemma~\ref{lem:K_mistake} (given only $X=x$), which links $\Delta_\internal$ and $\Delta^\hinge_\internal$ up to a factor of $K-1$, we can write
\begin{align}
    \Delta_\internal(x)&\le (K-1)\Delta_{\internal}^\hinge(x, t) \nonumber\\
    &\le (K-1)\frac{1}{[\tr_\Flat(x)]_t}\sum_{\tau=1}^T [\tr_\Flat(x)]_\tau\Delta_\internal^\hinge(x, \tau) \nonumber\\
    &{\le} T(K-1) \sum_{\tau=1}^T [\tr_\Flat(x)]_\tau\Delta_\internal^\hinge(x, \tau),\label{eq:Delta_int_UB}
\end{align}
where the last inequality holds because $  [\tr_\Flat(x)]_\tau\ge 0, \Delta_\internal^\hinge(x, \tau)\ge0$ for all $\tau\in[T]$ and,  by the definition of $t$ in \eqref{eq:def_t_t_star_flat}, $[\tr_\Flat(x)]_t\ge \frac{1}{T}$ at the chosen classifier index $t$.  
Similarly, we have 
\begin{align}
\sum_{\tau=1}^T [\tr_\Flat(x)]_\tau\Delta_\select^\hinge(x,\tau)
&=\sum_{\tau=1}^T [\tr_\Flat(x)]_\tau\left(\mc{R}_\tau^\hinge( x; \tf_\tau^*)
-  \mc{R}_{t^*}^\hinge( x; \tf_{t^*}^*)\right)\nonumber\\
& \stackrel{\text{(i)}}{=} \sum_{\tau=1}^T [\tr_\Flat(x)]_\tau \cdot \tfrac{K}{K-1}\left(\mc{R}_\tau(x; f_\tau^*) -  \mc{R}_{t^*}(x; f_{t^*}^*) \right)\nonumber\\
& \stackrel{\text{(ii)}}{\ge}  [\tr_\Flat(x)]_t \cdot \tfrac{K}{K-1}\left(\mc{R}_t(x; f_t^*) -  \mc{R}_{t^*}(x; f_{t^*}^*) \right)\nonumber \\
&\stackrel{\text{(iii)}}{\ge} 
\tfrac{K}{T(K-1)} \Delta_\select(x)\,,\nonumber\\
\implies \Delta_\select(x)& \le \tfrac{T(K-1)}{K}\sum_{\tau=1}^T [\tr_\Flat(x)]_\tau\Delta_\select^\hinge(x,\tau).\label{eq:Delta_select_UB}
\end{align}
where step  (i) applies \eqref{eq:R_hinge_ftau}, step (ii) holds because $[\tr_\Flat(x)]_\tau \ge0$ for all $\tau\in[T]$ 
%(since $\tr_\Flat(x)\in \Delta^{T-1}$) 
and, by the definition of $t^*$ in \eqref{eq:def_t_t_star_flat},  $\mc{R}_\tau(x; f_\tau^*)-\mc{R}_{t^*}(x; f_{t^*}^*)\ge 0$  for all $\tau\in[T]$, and step (iii) holds because $[\tr_\Flat(x)]_t\ge \frac{1}{T}$ and $ \Delta_\select(x)=\mc{R}_t(x; f^*_t)-\mc{R}_{t^*}(x; f^*_{t^*})\ge0$.

\paragraph{Part 4: Combining all components.}
We have so far bounded the selection and internal   excess $0$-$1$-$c$ risks by their  excess hinge risk counterparts, respectively. The final step is to combine these components (Eq.~\eqref{eq:Delta_int_UB} and \eqref{eq:Delta_select_UB}) to obtain
\begin{align}
 \mc{R}_t(x; f_t)- \mc{R}_{t^*}( x; f_{t^*}^*)&=\Delta_\internal(x)+\Delta_\select(x)\nonumber\\
 & \le  T(K-1) \sum_{\tau=1}^T[\tr_\Flat(x)]_\tau \left(\Delta_{\internal}^\hinge(x, \tau)+ \Delta_\select^\hinge(x, \tau)\right) \nonumber\\
 &=  T(K-1) \left(\mc{R}_\Flat^\hinge(x; \tpi)-
\mc{R}_\Flat^\hinge(x; \tpi^*)\right).
\end{align}
This concludes the proof of \eqref{eq:calibration_x}.

\end{proof}

\clearpage

\section{ Proofs for the Cascade Model}\label{app:cascade_all}
In this section, we prove the calibration result  in Theorem~\ref{thm:calibration_casc} for the cascade hinge loss $\ell^\hinge_\casc$. 
Similar to \eqref{eq:def_Rt_x_zt}--\eqref{eq:def_Rt_hinge_x_zt} of the flat case, we first define conditional 0-1-$c$ and hinge risks as
\begin{align}
    \mc{R}_t(h_t; f_t)&: = \ex\Big[\Ind\big\{f_t(X, Z_t)\ne Y\big\}+ c_t \mid H_t=h_t\Big]\,, \label{eq:cond_Rt_casc_1}\\
    \mc{R}_t^\hinge(h_t; \tf_t)&:= \ex\Big[\ell^\hinge(\tf_t(X,Z_t), Y;c_t)\mid  H_t=h_t\Big]\,. \label{eq:cond_Rt_casc_2}
\end{align}
As elsewhere in the paper, we suppress the `0-1-$c$' superscript and use the  `$\hinge$' superscript only to distinguish hinge quantities. Thus, $\mc{R}_t$ means $\mc{R}_t^{0\text{-}1\text{-}c}$, and the same naming convention applies to other quantities throughout.  The calibration proof relies on   \emph{value functions}, a concept standard in  the dynamic programming literature which we restate below for completeness. 
\begin{definition}[Value function]\label{def:value_function}
For a cascade classifier $\pi=(f_{1:T}, d_{1:T})$, a   $\loss\in\{0\text{-}1\text{-}c, \hinge\}$, and the stage index $t\in[T]$, we define the stage-$t$ value function $V_t^{\pi, \loss}(h_t)$ as the expected cost-to-go from stage $t$ onward, given the current history $h_t$. Formally,  define  
\begin{align}\label{eq:def_Vt_new}
    V_t^{\pi,\loss}(h_t):= (1-d_t(h_t))  \mc{R}_t^{\loss}(h_t; f_t) +  d_t(h_t)  \E\left[V_{t+1}^{\pi, \loss}(H_{t+1})\mid H_t=h_t\right],
    \quad \text{for }t<T
\end{align}
and 
\begin{align}\label{eq:def_VT}
    V_T^{\pi,\loss}(h_T):=\mc{R}_T^\loss(h_T; f_T).
\end{align} 
\end{definition}
The recursion in \eqref{eq:def_Vt_new} reflects the two possible actions at stage $t$   given current history $h_t$: if the deferral function exits ($d_t(h_t)=0$), it incurs the immediate conditional risk $\mc{R}_t^\loss(h_t; f_t)$; if it defers  ($d_t(h_t)=1$), it incurs   the   expected future cost  starting from stage $t+1$. At the terminal stage, there is  no further stage to defer to, and the value function reduces to the immediate risk, yielding \eqref{eq:def_VT}. 
Thus, $V_t^{\pi, \loss}(h_t)$   represents the total  expected risk  incurred by following $\pi = (f_{1:T}, d_{1:T})$ from  stage $t$ onward, conditional on $h_t$.

\paragraph{Notation:}
Since   the $0$-$1$-$c$ loss is only  used with hard policies (e.g., $\pi$ and $\pi^*_\casc$) and the  hinge loss only  with soft policies (e.g., $\tpi$ and $\tpi^*_\casc$), we   drop the superscript `$\loss$' from $V_t^{\pi,\loss}(h_t)$ for brevity, as the associated loss  will be  clear from context.  As in Appendix~\ref{app:flat_all}, we use LHS (resp.\ RHS) as short for left (resp.\ right) hand side;  
 we  write $p(\cdot \mid \cdot )$   for  the conditional  probability mass function of $Y$, e.g.,  $p(y \mid x)\equiv \mathrm{Pr}(Y=y\mid X=x)$ and $p(y\mid x, z_t)\equiv \mathrm{Pr} (Y=y\mid X=x,Z_t=z_t)$.

\paragraph{Proof sketch:} This section is organized as follows. We first show that the   $0$-$1$-$c$ and hinge risks for cascades can be expressed in terms of the stage-1 value functions  (Section~\ref{app:connect_risk_value}). Building on this connection,   we   characterize the Bayes-optimal solutions under both the $0$-$1$-$c$ loss and cascade hinge loss (Section~\ref{app:proof_BO_casc_V}). Next, we establish two useful  results linking hinge and $0$-$1$-$c$ quantities (Section~\ref{app:connections_01c_hinge_cascade}). Using this, we complete the proof of the calibration result in Theorem~\ref{thm:calibration_casc} (Section~\ref{app:proof_calibration_casc}).

\subsection{Connection of Risks to Expected Value Functions}\label{app:connect_risk_value}
In this section, we establish an essential connection between the global risk of a cascade classifier $\pi=(f_{1:T}, d_{1:T})$ and its associated value functions. Specifically, Lemma~\ref{lem:global_risk_to_V1} shows that both the  $0$-$1$-$c$ risk $\mc{R}(\pi)$ and hinge risk $\mc{R}_\casc^\hinge(\tpi)$ can be expressed as   expected value functions at the initial stage for arbitrary $\pi\in\Pi_\casc$ and $\tpi\in\wt\Pi_\casc$,   and Corollary~\ref{cor:bayes_to_value} specializes this result to  their Bayes-optimal  counterparts $\pi_\casc^*\in\Pi_\casc$ and $\tpi_\casc^*\in\wt\Pi_\casc$. These connections allow us to analyze the   risks through the recursive structure of the cascade, and  serves as the foundation for characterizing Bayes-optimal solutions in the next section and,  ultimately, for  proving  calibration.

\begin{lemma}[Connection of   risks to value functions for arbitrary $\pi$ and $ \tpi$]\label{lem:global_risk_to_V1}   For any hard cascade classifier $\pi\in \Pi_\casc$,  
%with $r_\casc$ defined through $d_{1:T}$, 
we have
     \begin{align}
     \mc{R}(\pi)  = \E\left[V_1^{\pi}(H_1)\right] \,,\label{eq:result2_new}
\end{align}
    and similarly, 
for any soft cascade classifier  $\tpi\in \wt{\Pi}_\casc$, we have  
\begin{align}
    \mc{R}^\hinge_\casc(\tpi) = \E\left[V^{\tpi}_1(H_1)\right] \,.\label{eq:result1_new}
    \end{align}
\end{lemma}
\begin{proof}
Recall we write $H_t=(X, Z_{1:t})$ and $H\equiv H_T=(X,Z_{1:T})$.  Let     $\pi=(f_{1:T}, r_\casc) $, where the stopping rule is defined as $r_\casc(h)=\min\{t\in[T]: d_t(h_t)=0\}$. By the definition of the $0$-$1$-$c$ risk $\mc{R}(\pi)$ in \eqref{eq:multi_01c}--\eqref{eq:def_R(f;c)}  and the tower property,  we can write 
\begin{align*}
    \mc{R}(\pi) &= \E\Big[\Ind\{f_{r_\casc(H)}(X, Z_{r_\casc(H)})\ne Y\}+ c_{r_\casc(H)}\Big]\\
    &=\sum_{t=1}^T\E\Big[\Ind\big\lbrace r_\casc(H)=t\}\cdot \left(\Ind\{f_{t}(X, Z_{t})\ne Y\big\rbrace+c_t\right)\Big]\\
     &=\sum_{t=1}^T\E\Big[\E\left[\Ind\{r_\casc(H)=t\}\cdot \left(\Ind\{f_{t}(X, Z_{t})\ne Y\}+c_t\right)\mid H_t\right]\Big]\\
      &=\sum_{t=1}^T\E\left[\Bigg(\prod_{\tau=1}^{t-1} d_\tau(H_{\tau})\Bigg) 
\big(1 - d_t(H_t)\big)  \cdot \E\left[ \Ind\{f_{t}(X, Z_{t})\ne Y\}+c_t \mid H_t\right]\right]\\
     &=\sum_{t=1}^T\E\left[ \Bigg(\prod_{\tau=1}^{t-1} d_\tau(H_{\tau})\Bigg) 
\big(1 - d_t(H_t)\big) \cdot \mc{R}_t(H_t; f_t) \right]\\
&=\E\left[\sum_{t=1}^T \Bigg(\prod_{\tau=1}^{t-1} d_\tau(H_{\tau})\Bigg) 
\big(1 - d_t(H_t)\big) \cdot \mc{R}_t(H_t; f_t) \right]
%\stackrel{?}{=}\E[V_1^f(H_1)]
\end{align*}
Therefore, to prove  \eqref{eq:result2_new},
 it suffices to show that 
 \begin{align}\label{eq:intermediate}
     \E\left[\sum_{t=1}^T \Bigg(\prod_{\tau=1}^{t-1} d_\tau(H_{\tau})\Bigg) 
\big(1 - d_t(H_t)\big) \cdot \mc{R}_t(H_t; f_t) \right]=\E\left[V_1^{\pi}(H_1)\right]\,,
 \end{align}
 which in turn implies that, it suffices to show that for all $i\in [T]$,
\begin{align}\label{eq:hypothesis_new}
   \E\left[\sum_{t=i}^T \Bigg(\prod_{\tau=i}^{t-1} d_\tau(H_{\tau})\Bigg) 
\big(1 - d_{t}(H_{t})\big) \cdot \mc{R}_{t}(H_{t}; f_t) \;\Big|\; H_i\right]= V_i^{\pi}(H_i)\,,%\quad \forall  i\in [T]\,.
\end{align}
We prove \eqref{eq:hypothesis_new} 
 using backward induction. \\
 \textbf{Base  case   $i=T$:} To show \eqref{eq:hypothesis_new} holds for $i=T$, note that by the definition of the stage-$T$ value function in \eqref{eq:def_VT}, 
 \begin{align*}
     V_T^{\pi}(H_T)=\mc{R}_T(H_T; f_T)= \E\left[\sum_{t=T}^T \Bigg(\prod_{\tau=T}^{t-1} d_\tau(H_{\tau})\Bigg) 
\big(1 - d_t(H_t)\big) \cdot \mc{R}_t(H_t; f_t) \right]\,,
 \end{align*}
 where we used   the fact that for $a>b$, $\prod_{\tau=a}^{b} d_\tau(H_\tau) = 1$. \\
 \textbf{Backward induction:} Now, suppose \eqref{eq:hypothesis_new} holds for  $i=s+1$, we next prove it holds for $i=s$. Using the recursive Bellman definition of the value function  \eqref{eq:def_Vt_new}, we can write
\begin{align*}
    V_{s}^{\pi}(H_s)
   &=(1-d_s(H_s))\mc{R}_s(H_s; f_s) + d_s(H_s)\E\left[V_{s+1}^{\pi}(H_{s+1})\mid H_s\right]\\
    &=(1-d_s(H_s))\mc{R}_s(H_s; f_s)\, + \\
    &\quad d_s(H_s)\E\left[\E\Bigg[\sum_{t=s+1}^T \Bigg(\prod_{\tau=s+1}^{t-1} d_\tau(H_{\tau})\Bigg) 
\big(1 - d_{t}(H_{t})\big) \cdot \mc{R}_{t}(H_{t}; f_{t})\mid H_{s+1}\Bigg]\mid H_{s}\right]\\
&=(1-d_{s}(H_{s}))\mc{R}_{s}(H_{s}; f_{s}) \,+ \\
    &\quad d_s(H_s)\E\left[ \sum_{t=s+1}^T \Bigg(\prod_{\tau=s+1}^{t-1} d_\tau(H_{\tau})\Bigg) 
\big(1 - d_{t}(H_{t})\big) \cdot \mc{R}_{t}(H_{t}; f_{t}) \mid H_{s}\right]\\
&=\E\left[\sum_{t=s}^T \Bigg(\prod_{\tau=s}^{t-1} d_\tau(H_{\tau})\Bigg) 
\big(1 - d_{t}(H_{t})\big) \cdot \mc{R}_{t}(H_{t}; f_{t})\mid H_s\right],
\end{align*}
where the last step absorbs the $t=s$ term into the summation, yielding \eqref{eq:hypothesis_new} for $i=s$.\\
\textbf{Conclusion:} Applying this backward induction argument   down from $i=T-1$ to $i=1$ yields  \eqref{eq:hypothesis_new} for all $i\in[T]$. Setting $i=1$ in \eqref{eq:hypothesis_new} and taking total expectation over $H_1$ proves \eqref{eq:intermediate}, which~gives~\eqref{eq:result2_new}.

In the remainder of the proof, we establish \eqref{eq:result1_new}. Let $\tpi=(\tf_{1:T}, \td_{1:T})$.   By definition of $\mc{R}^\hinge_\casc(\tpi)$ in \eqref{eq:surrogate_hinge_casc} and tower property,
\begin{align*}
    \mc{R}_\casc^\hinge(\tpi) &= \E\left[\ell_\casc^\hinge(X,Z_{1:T}, Y; \tpi, c)\right]\\
    &=\sum_{t=1}^T \E\left[\Bigg(\prod_{\tau=1}^{t-1} \td_\tau(H_{\tau})\Bigg) 
\big(1 - \tilde{d}_t(H_t)\big)\E\left[\ell^\hinge(\tf_t(X,Z_t), Y; c_t)\mid H_t\right]\right]\\
&=\E\left[\sum_{t=1}^T \Bigg(\prod_{\tau=1}^{t-1} \td_\tau(H_{\tau})\Bigg) 
\big(1 - \tilde{d}_t(H_t)\big)\cdot \mc{R}^\hinge_t (H_t; \tf_t)\right],
%\stackrel{?}{=}\E[V_1^{\tf}(H_1)]
\end{align*}
so to prove \eqref{eq:result1_new}, it suffices to prove 
\begin{align*}
     \E\left[\sum_{t=1}^T \Bigg(\prod_{\tau=1}^{t-1} \td_\tau(H_{\tau})\Bigg) 
\big(1 - \tilde{d}_t(H_t)\big)\cdot \mc{R}^\hinge_t (H_t; \tf_t)\right]{=}\E\left[V_1^{\tpi}(H_1)\right],
\end{align*}
which in turn implies it suffices to prove for all $i\in[T]$,
\begin{align}
     \E\left[\sum_{t=i}^T \Bigg(\prod_{\tau=i}^{t-1} \td_\tau(H_{\tau})\Bigg) 
\big(1 - \tilde{d}_t(H_t)\big)\cdot \mc{R}^\hinge_t (H_t; \tf_t)\mid H_i\right]{=} V_i^{\tpi}(H_i) \,.
\end{align}
This is   analogous to \eqref{eq:hypothesis_new}.  The proof works verbatim 
because the Bellman recursion and tower rule do not depend on whether $\td_t$ lies in $\{0,1\}$ or $[0,1]$.
\end{proof}

Lemma~\ref{lem:global_risk_to_V1} connects the risks to the value functions for arbitrary cascade classifiers   $\pi$ and $\tpi$. Specializing this result to the \emph{optimal} cascade classifiers $\pi^*_\casc$ and $\tpi^*_\casc$ yields a corresponding connection between the  Bayes risks and the value functions of the optimal cascade classifiers. This is stated in   Corollary \ref{cor:bayes_to_value}.

\begin{corollary}[Connection of  risks to value functions for Bayes-optimal  $\pi^*_\casc$ and $\tpi^*_\casc$]\label{cor:bayes_to_value}Let $\pi^*_\casc=(f_{1:T}^*, d_{1:T}^*)\in \argmin_{\pi \in \Pi_\casc}\mc{R}(\pi)$ be as defined in Lemma~\ref{lem:bayes_opt_casc_V}, then  
\begin{align*}
    \mc{R}_\casc^*=\E\left[V_1^{\pi^*_\casc}(H_1)\right].
\end{align*}
  Similarly, let $\tpi^*_\casc=(\tf^*_{1:T}, \td^*_{1:T})\in \argmin_{\tpi \in\wt{\Pi}_\casc} \mc{R}^\hinge_\casc(\tpi)$ be as defined in Lemma~\ref{lem:fstar_tilde_def}, then  
  \begin{align*}
      \mc{R}^{\hinge*}_\casc=\E\left[V_1^{\tpi^*_\casc}(H_1)\right]\,.
  \end{align*}
\end{corollary}
\begin{proof}
By \eqref{eq:result2_new}, which relates $\mc{R}$ to $V_1^\pi$ for an arbitrary hard cascade classifier $\pi$, we can write 
\begin{align*}
    \mc{R}_\casc^*:=\min_{\pi \in \Pi_\casc } \mc{R}(\pi) =  \mc{R}(\pi_\casc^*)=  \E\left[V_1^{\pi^*_\casc}(H_1)\right].
\end{align*}

By \eqref{eq:result1_new}, which relates $\mc{R}^\hinge_\casc$ to $V_1^{\tpi}$ for an arbitrary soft  cascade classifier $\tpi$, we can write
\begin{align*}
    \mc{R}_\casc^{\hinge*}:=\min_{\tpi \in \wt{\Pi}_\casc } \mc{R}^\hinge_\casc(\tpi) =  \mc{R}^\hinge_\casc(\tpi_\casc^*)=  \E\left[V_1^{\tpi^*_\casc}(H_1)\right].
\end{align*}
\end{proof}

\subsection{Bayes-Optimal Solutions }\label{app:proof_BO_casc_V}
Using the connections of risks to expected value functions proven in the previous section, 
we now characterize the Bayes-optimal solutions under the $0$-$1$-$c$ loss and cascade hinge loss, respectively. 
\begin{restatable}[Bayes-optimal solution in $\Pi_\casc$ for $0$-$1$-$c$ loss]{lemma}{lemmaBOcascV}\label{lem:bayes_opt_casc_V}
Assume $Y\perp Z_{1:t-1}\mid (X,Z_t)$ as in  \eqref{eq:conditional_indep}. For the $0$-$1$-$c$ risk $\mc{R}$ with cost vector $c=[\kappa_1, \kappa_1+\kappa_2, \dots, \sum_{t=1}^T \kappa_t]^\top$, there exists a Bayes-optimal cascade classifier 
\begin{align}\label{eq:def_f_casc_star}
\pi^*_\casc=(f^*_{1:T}, r_{\casc}^*)\in \argmin_{\pi \in \Pi_\casc} \mc{R}(\pi),
\end{align}
where $f^*_{1:T}$ satisfies \eqref{eq:f_star_multi_app},  and $r^*_\casc(h):=\min\{t\in[T]: d_t^*(h_t)=0\}$ 
%as in \eqref{eq:def_r_casc} 
is induced by     the deferral rules $d^*_{1:T}$   given by 
\begin{align}
d_t^*(h_t) 
= \Ind\!\Big\{ 
\mc{R}_t(h_t; f^*_t)
>   \E\Big[ V^{\pi_{\casc}^*}_{t+1}(H_{t+1}) \mid H_t=h_t \Big] 
  \Big\}\quad \forall t<T ,\quad \text{and}\quad  d_T^*(h_T)=0,\label{eq:d_star_casc_V}
\end{align}
where 
\begin{align}\label{eq:lem_Rt_ht_ft_star_def}
    \mc{R}_t(h_t; f^*_t) = 1 - \max_{y\in[K]} p(y\mid x,z_t) + c_t,\quad \forall t\in[T]
\end{align}
is the Bayes risk of exiting at stage $t$ conditional on $H_t=h_t$. 
\end{restatable}
\begin{proof}Throughout, we write  $H_t=(X, Z_{1:t})$ and  $ H\equiv H_T=(X,Z_{1:T})$. 
We first optimize $f_{1:T}$ with $d_{1:T}$ fixed and then optimize $d_{1:T}$ with $f_{1:T}$ fixed.

\paragraph{Part 1: 
Proof for the Bayes-optimal   $f^*_{1:T}$.} 
We first minimize $\mc{R}(\pi)$ with respect to $f_{1:T}$ for an arbitrary  routing function $r_{\casc}$, or equivalently, an arbitrary $d_{1:T}$.  Fix $d_{1:T}$, we can expand the $0$-$1$-$c$ risk   as follows:
 \begin{align*}
        \mc{R}(\pi) &= \E[\ell(X,Z_{1:T}, Y; \pi, c)] \\
        &=\E\Big[\Ind\big\{f_{r_\casc(H)} (X, Z_{r_\casc(H)}) \ne Y\big\} + c_{r_\casc(H)}\Big]\\
        & =\sum_{t=1}^T \E\Big[\Ind\{r_\casc(H)=t\}\cdot \Big(  \Ind\big\{f_{t} (X, Z_{t}) \ne Y\big\}+ c_{t}\Big)  \Big] \\
        &\stackrel{\text{(i)}}{=} \sum_{t=1}^T \E\Bigg[\Ind\{r_\casc(H)=t\}\cdot \underbrace{\Big(\E\Big[  \Ind\big\{f_{t} (X, Z_{t}) \ne Y\big\} \mid H_t\Big] + c_{t}\Big)}_{=\mc{R}_t(H_t; f_t)}\Bigg] ,
    \end{align*}
    where $\mc{R}_t(H_t; f_t)$ is the conditional risk of exiting at stage $t$ given $H_t$, as defined in \eqref{eq:cond_Rt_casc_1}. 
    %; the semicolon here emphasizes that $H_t$ is a variable and $f_t$ is treated as a parameter.
    Note step (i)   holds  because $\Ind\{r_\casc(H)=t\}$ is measurable with respect to $H_t$.   %Since $c_t$ is a constant that does not depend on $f_{1:T}$, 
    Minimizing $\mc{R}(\pi)$ with respect to $f_t$ reduces to minimizing $\mc{R}_t(H_t; f_t)$ pointwise in $H_t$. Fix $h_t\in\mc{H}_t$, using $Y\perp  Z_{1:t-1}\mid (X, Z_t)$~from~\eqref{eq:conditional_indep}, we have 
    \begin{align}
        \mc{R}_t(h_t; f_t)&=\E\Big[  \Ind\big\{f_{t} (X, Z_{t}) \ne Y\big\} \mid H_t=h_t\Big] + c_{t}\nonumber\\
        &\stackrel{\text{(i)}}{=}\E\Big[  \Ind\big\{Y\ne f_{t} (x, z_{t}) \big\} \mid X=x, Z_t=z_t\Big] + c_{t} \nonumber\\
        &= 1-p\big(f_t(x,z_t)\mid x, z_t\big) +c_t\label{eq:Rt_ht_cascade_general_ft}
    \end{align}
    %where $c_t$ is a constant that does not depend on $f_{1:T}$. 
    where recall $p(\cdot \mid x,z_t)$ denotes the conditional probability of $Y$ given $(X,Z_t)=(x,z_t)$, and~step~(i) here holds because  $Y\perp Z_{1:t-1}\mid (X, Z_t)$ by the assumption in  \eqref{eq:conditional_indep}.
Hence, minimizing $\mc{R}_t(h_t; f_t)$ pointwise in $h_t$ reduces to minimizing \eqref{eq:Rt_ht_cascade_general_ft} pointwise in $(x,z_t)$, that is,
\begin{align}
    f_t^*(x,z_t)\in \argmin_{y\in[K]} 1-p\big(y\mid x, z_t\big) +c_t=\argmax_{y\in[K]}p\big(y\mid x, z_t\big) .\label{eq:ft_star_casc_proof_concludion}
\end{align}
This proves that $f^*_{1:T}$ satisfy \eqref{eq:f_star_multi_app}.  Substituting \eqref{eq:ft_star_casc_proof_concludion} into \eqref{eq:Rt_ht_cascade_general_ft} yields 
\begin{align*}
    \mc{R}_t(h_t; f_t^*)=1-\max_{y\in[K]}p\big(y\mid x, z_t\big)+c_t,
\end{align*} 
as stated in \eqref{eq:lem_Rt_ht_ft_star_def}. 

Note the result in \eqref{eq:ft_star_casc_proof_concludion}  holds \emph{uniformly} for any  routing policy $r_\casc$, or equivalently,  any  $d_{1:T} $. In other words, it characterizes the Bayes-optimal classifiers $f_{1:T}^*$ independently of how $d_{1:T}$ is chosen.  Therefore, it   remains  only to characterize the Bayes-optimal deferral rules $d_{1:T}^*$ with the classifiers fixed at $f_{1:T}=f^*_{1:T} $.

    \paragraph{Part 2: Proof for the Bayes-optimal   $d_{1:T}^*$.}
 We now  minimize $\mc{R}(\pi)$ with respect to   $d_{1:T}$  with  $f_{1:T}=f^*_{1:T}$ held fixed. 
Since $\pi=(f_{1:T}, d_{1:T})$, we   write $\mc{R}(f_{1:T}, d_{1:T})$ to mean $\mc{R}(\pi)$ and $V_t^{f_{1:T}, d_{1:T}}(H_t)$ to mean $V_t^\pi(H_t)$. 
By \eqref{eq:result2_new}  which links $\mc{R}$ to $V_1$, we have
\begin{align}\label{eq:d_star_formulation}
d^*_{1:T}\in\argmin_{d_{1:T}}
\mc{R}(f_{1:T}^*, d_{1:T}) = \argmin_{d_{1:T}} \E\left[V_1^{f^*_{1:T}, d_{1:T}}(H_1)\right],
\end{align}
where the minimum is over $d_t:\mc{H}_t\to \{0,1\}$ for $t\in[T]$. By the recursive definition \eqref{eq:def_Vt_new} of the value function, we can unroll \eqref{eq:d_star_formulation} to optimize $d_t$
 one at a time, treating the earlier functions $d_{1:t-1} $ as arbitrary. Consequently, the optimal  $d_t^*$ is  pointwise optimal for every trajectory $d_{1:t-1}$, \emph{not} just for the optimal trajectory $d_{1:t-1}^*$, and composing the pointwise minimizers $d_t^*$ across $t\in[T]$  yields $d_{1:T}^*$. 
 
In the rest of the proof,  we characterize $d_t^*$ using  backward induction   starting with   $t=T$. \\
 \textbf{Base case $t=T$:}
 Since $T$ is the terminal stage, $d_T^*(H_T)$ is fixed to $0$ (exit), as stated in  \eqref{eq:d_star_casc_V}.  By definition \eqref{eq:def_VT} of the terminal-stage value function, $V_T^{f^*_{1:T}, d_{1:T-1},d^*_T}(H_T)=\mc{R}_T(H_T; f_T^*)$ for any  arbitrary $d_{1:T-1}$, so we can write 
 \begin{align}\label{eq:VT_f_casc_star}
     V_T^{f^*_{1:T}, d_{1:T-1}, d^*_T}(H_T)=V_T^{\pi^*_\casc}(H_T)\quad \text{where}\quad \pi^*_\casc=(f^*_{1:T}, d_{1:T}^*).
 \end{align} 
\textbf{Backward induction:} Using \eqref{eq:VT_f_casc_star}, we now prove the induction step. Suppose that    for some $t\in\{1, \dots T-1\}$,  
\begin{align}\label{eq:V_t_star_hypothesis}
    V_{t+1}^{f^*_{1:T}, d_{1:t}, d^*_{t+1:T}}(H_{t+1})=V_{t+1}^{\pi^*_\casc}(H_{t+1})
\end{align}for any arbitrary $d_{1:t}$, then 
\begin{align}
  d_t^*(h_t) 
= \Ind\!\Big\{ 
\mc{R}_t(h_t; f^*_t)
>   \E\Big[ V^{\pi_{\casc}^*}_{t+1}(H_{t+1}) \mid H_t=h_t \Big] 
  \Big\}\,.  %V_{t}^{f^*_{1:T}, d_{1:t-1}, d^*_{t:T}}(H_t)=V_t^{f^*_\casc}(H_t)\,.
    \label{eq:d_star_target}
\end{align}
%we fix  $d_T=d_T^*$ and  unroll the value function $V_1$ in the  objective   \eqref{eq:d_star_formulation} to optimize $d_{T-1}$. This gives
We fix  $d_{t+1:T}=d_{t+1:T}^*$ and  unroll the value function $V_1$ in the  objective   \eqref{eq:d_star_formulation} to optimize over $d_{t}$. This gives
 \begin{align}\label{eq:E_V1_induction_new}
    \E\Big[V_1^{f^*_{1:T}, d_{1:t},d_{t+1:T}^*}(H_1)\Big]&=\underbrace{\E\Bigg[ (1-d_1(H_1))\mc{R}_1  (H_1; f^*_1)+  d_1(H_1)\E\Big[\dots+ \E\Big[\dots + d_{t-1}(H_{t-1})}_{\text{Term I}}  \nonumber\\
    &\quad\; \E\Big[ V_{t}^{f^*_{1:T}, d_{1:t}, d_{t+1:T}^*}(H_t)\underbrace{\mid H_{t-1}\Big]\mid H_{t-2}\Big]\mid  \dots \mid H_1\Big]\Bigg]}_{\text{Term II}}.
\end{align}

To clarify, 
 Term I collects the nested expectations taken successively over $ H_{t-1}| H_{t-2}$, $H_{t-2}| H_{t-3}, \dots$ down to $ H_{2}| H_1$, and finally over $H_1$ through the outermost expectation, along with terms depending on  $ d_1, \dots, d_{t-1}$.   Term II collects the corresponding chain of conditionings from $H_{t-1}$ down to $H_1$. Hence, the part of \eqref{eq:E_V1_induction_new} that depends on $d_{t}$ is   
 \begin{align}
&\E\Bigg[d_1(H_1)\E\Big[d_2(H_2)\dots \E\Big[ d_{t-1}(H_{t-1})\E\Big[V_t^{f^*_{1:T}, d_{1:t}, d^*_{t+1:T}}(H_{t})\mid H_{t-1}\Big]\mid H_{t-2}\Big]\mid \dots \mid H_1\Big]\Bigg]\nonumber\\
=&\E\Bigg[\prod_{\tau=1}^{t-1} d_\tau (H_\tau) V_t^{f^*_{1:T}, d_{1:t}, d^*_{t+1:T}} (H_{t})\Bigg],\label{eq:V1_simplified}
\end{align}
where the equality uses the tower rule of conditional expectation, 
and  $V_t^{f^*_{1:T}, d_{1:t}, d^*_{t+1:T}}(H_{t})$ is the innermost term of \eqref{eq:E_V1_induction_new}, which by definition, is
%depending on $H_{T-1}$ and $d_{T-1}$:
\begin{align}
V_t^{f^*_{1:T}, d_{1:t}, d^*_{t+1:T}}(H_{t})
&= \big(1-d_{t}(H_{t})\big)\mc{R}_{t}(H_{t}; f^*_{t}) + d_{t}(H_{t})\E\Big[V_{t+1}^{f^*_{1:T}, d_{1:t}, d^*_{t+1:T}}(H_{t+1})\mid H_{t}\Big]\nonumber\\
&\stackrel{\text{(i)}}{=} \big(1-d_{t}(H_{t})\big)\mc{R}_{t}(H_{t}; f^*_{t}) + d_{t}(H_{t})\E\Big[V_{t+1}^{\pi^*_\casc}(H_{t+1})\mid H_{t}\Big],\label{eq:Vt_star}
\end{align}
where (i) holds by  the induction hypothesis \eqref{eq:V_t_star_hypothesis}. 
Hence, \eqref{eq:E_V1_induction_new}--\eqref{eq:V1_simplified} imply that the Bayes-optimal deferral function $d^*_{t}$ should satisfy
\begin{align}\label{eq:d_star_T-1_as_function_of_A}
    d^*_{t}&\in\argmin_{d_{t}}  \E\Big[V_1^{f_{1:T}^*, d_{1:t}, d_{t+1:T}^*}(H_1)\Big]=\argmin_{d_{t}} \E\Bigg[\prod_{\tau=1}^{t-1} d_\tau (H_\tau) V_{t}^{f^*_{1:T}, d_{1:t}, d^*_{t+1:T}} (H_{t})\Bigg].
\end{align}
Since the $\prod_{\tau=1}^{t-1} d_\tau (H_\tau) $ term is independent of $d_{t}$ and always non-negative, \eqref{eq:d_star_T-1_as_function_of_A} reduces   pointwise in $h_{t}\in\mc{H}_{t}$ to
\begin{align}
    d_{t}^*(h_{t})
   &\in\argmin_{d\in\{0,1\}}V_t^{f^*_{1:T}, d_{1:t}, d^*_{t+1:T}}(h_{t}) \nonumber\\
    &= \argmin_{d  \in\{0,1\}}\big(1-d \big)\mc{R}_{t}(h_{t}; f^*_{t}) + d \E\Big[V_{t+1}^{\pi_{\casc}^*}(H_{t+1})\mid H_{t}=h_{t}\Big] \nonumber\\
    & =\Ind\Big\lbrace \mc{R}_{t}(h_{t}; f_{t}^*)> \E\Big[V_{t+1}^{\pi_{\casc}^*}(H_{t+1})\mid  H_{t}=h_{t} \Big]\Big\rbrace\,.
\end{align}
This proves \eqref{eq:d_star_target}, that is, 
  \eqref{eq:d_star_casc_V} for stage $t$.  Substituting this Bayes-optimal $d_t^*$ back into the recursion \eqref{eq:Vt_star} yields:
 \begin{align}
     V_t^{f^*_{1:T}, d_{1:t-1}, d^*_{t:T}}(H_t)&=\big(1-d_{t}^*(H_{t})\big)\mc{R}_{t}(H_{t}; f^*_{t}) + d_{t}^*(H_{t})\E\Big[V_{t+1}^{\pi^*_\casc}(H_{t+1})\mid H_{t}\Big]\stackrel{\text{(i)}}{=}V_t^{\pi^*_\casc}(H_t),\label{eq:Vt_conclusion}
 \end{align}
  for any arbitrary $d_{1:t-1}$, where (i) holds by definition of $V_t^{\pi^*_\casc}$. Thus, this establishes     \eqref{eq:V_t_star_hypothesis}  for stage $t$, completing the induction step.\\ 
\textbf{Conclusion:}
 By backward induction, we  obtain \eqref{eq:d_star_casc_V} for all $t\in[T]$. 

 \end{proof}
 \paragraph{Connection of $d_{1:T}^*$ to optimal stopping:} %The Bayes-optimal $d_{1:T}^*$ we prove  above  recovers the optimal stopping result of  \citep{shiryaev2008optimal}. 
  
 Identify the filtration   $\sigma(H_t)$, the adapted reward process $\{G_t\}_{t=1}^T$ with $G_t=-\mc{R}_t(H_t; f_t^*)$, and let $  \mathrm{T} = \{\tau: \tau \in [T], \{\tau=t\} \in \sigma(H_t) \}$ denote the set of stopping times. Then the deferral rule $d_{1:T}^*$ in \eqref{eq:d_star_casc_V}, which minimizes the risk $\mc{R}(f^*_{1:T}, d_{1:T})$, coincides with the optimal stopping rule for the finite-horizon problem  $\sup_{\tau \in \mathrm{T}}\E[G_\tau]$ \citep{shiryaev2008optimal}, which maximizes the reward, and our value function is its dynamic programming recursion. 

While Lemma \ref{lem:bayes_opt_casc_V} characterizes the Bayes-optimal solution under the target $0$-$1$-$c$ loss, the next lemma is in parallel and characterizes the Bayes-optimal solution under the  cascade hinge loss. 

\begin{lemma}[Bayes-optimal solution in $\wt{\Pi}_\casc$ for hinge loss]\label{lem:fstar_tilde_def}
   Assume $Y\perp Z_{1:t-1}\mid (X,Z_t)$ as in  \eqref{eq:conditional_indep}. There exists a Bayes-optimal cascade classifier  
   \begin{align*}
       \tpi^*_\casc=(\tf^*_{1:T}, \td_{1:T}^*)\in \argmin_{\tpi \in \wt{\Pi}_\casc}\mc{R}^\hinge_\casc(\tpi) ,
   \end{align*}  
   where, with $f^*_t$ defined as in \eqref{eq:f_star_multi_app}, 
\begin{align}\label{eq:f_star_tilde}
\left[\tf^*_t(x,z_t)\right]_k=\begin{cases}
    1& \text{for}\;\; k = f_t^*(x, z_t)\\
    %\argmax_{y\in[K]}p(y\mid x, z_t)\\
    -\tfrac{1}{K-1} & \text{otherwise}
\end{cases}\quad \forall  t\in[T]\,,
\end{align}
and, with $\mc{R}^\hinge_t(h_t; \tf_t)$ defined as in \eqref{eq:cond_Rt_casc_2},
\begin{align}\label{eq:d_star_tilde}
    \td_t^*(h_t) = \Ind\Big\lbrace \mc{R}_t^\hinge (h_t; \tf^*_t)>\E\Big[V^{\tpi^*_\casc}_{t+1}(H_{t+1})\mid H_t=h_t\Big]\Big\rbrace\,\quad \forall t<T\quad  \text{and}\quad \td_T^*\equiv 0\,.
\end{align}
\end{lemma}
\begin{proof}
We minimize $\mc{R}^\hinge_\casc(\tpi)$ in two stages: first over $\tf_{1:T}$ with  $\td_{1:T}$   held fixed, then over $\td_{1:T}$  with $\tf_{1:T}=\tf^*_{1:T}$ fixed. 

\paragraph{Part 1: Proof for the Bayes-optimal   $\tf_{1:T}^*$.} Since  $p(y\mid x,z_{1:t})=p(y\mid x, z_t)$, the result in \eqref{eq:f_star_tilde} follows from   \cite[Lemma~1]{tarigan2008}.  
%we know $\mc{R}_t^\hinge(h_t; \tf_t)$ 
% is minimized pointwise in $(x,z_t)$ by taking $\tf_t=\tf_t^*$ to satisfy the following: 

\paragraph{Part 2: Proof for the Bayes-optimal   $\td_{1:T}^*$.} We now prove \eqref{eq:d_star_tilde}.  Since $\tpi=(\tf_{1:T}, \td_{1:T})$, we   write $\mc{R}_\casc^\hinge(\tf_{1:T}, \td_{1:T})$ to mean $\mc{R}_\casc^\hinge(\tpi)$, and $ V_t^{\tf_{1:T}, \td_{1:T}}(H_t)$ to mean    $ V_t^{\tpi}(H_t)$.  By definition of $\tpi^*_\casc$ and using \eqref{eq:result1_new} which links  $\mc{R}_\casc^\hinge$ and $V_1^{\tpi}$,  we can write 
  \begin{align}
&  \td^*_{1:T}\in \underset{\td_{1:T}}{\arg\min}\;\mc{R}_\casc^\hinge(\tf_{1:T}^*, \td_{1:T}) =\underset{\td_{1:T}  }{\arg\min} \; \E\left[V_1^{\tf^*_{1:T}, \td_{1:T}}(H_1)\right],\label{eq:objective_for_td_star}
  \end{align}
  where each minimum is taken over $ \td_t:\mc{H}_t\to [0,1]$ for $t\in[T]$. The proof proceeds by backward  induction, following the same steps as those in Part 2 of the proof of Lemma~\ref{lem:bayes_opt_casc_V}, see \eqref{eq:d_star_formulation}--\eqref{eq:Vt_conclusion}. For all $t<T$, this yields
  \begin{align*}
      \td_t^*(h_t)&\in \argmin_{\td \in[0,1]}\big(1-\td\big)\mc{R}^\hinge_{t}(h_{t}; \tf^*_{t})+ \td\E\Big[V_{t+1}^{\tf_{1:T}^*,\td_{1:t}, \td^*_{t+1:T}}(H_{t+1})\mid H_{t}=h_{t}\Big]\\
      & =\argmin_{\td \in[0,1]}\big(1-\td\big)\mc{R}^\hinge_{t}(h_{t}; \tf^*_{t})+ \td\E\Big[V_{t+1}^{\tpi_\casc^*}(H_{t+1})\mid H_{t}=h_{t}\Big]\,.
      \end{align*}
Since this objective is linear  in $\td\in[0,1]$,   the minimum is attained at an endpoint, that is,
\begin{align*}
    \td_t^*(h_t)
=\Ind\Big\lbrace \mc{R}^\hinge_{t}(h_{t}; \tf_{t}^*)> \E\Big[V_{t+1}^{\tpi^*_\casc}(H_{t+1})\mid  H_{t}=h_{t} \Big]\Big\rbrace\,,
  \end{align*}
  yielding \eqref{eq:d_star_tilde}.
\end{proof}

\subsection{Connections Between $0$-$1$-$c$   and Hinge Quantities}\label{app:connections_01c_hinge_cascade}

  In this section we establish two useful  results linking hinge and $0$-$1$-$c$ quantities, which will help prove calibration in the next section.
  \begin{lemma}[Hinge and $0$-$1$-$c$ Bayes quantities differ by a factor of $\frac{K}{K-1}$]\label{lem:K/(K-1)_factor} Assume $Y\perp Z_{1:t-1}\mid (X,Z_t)$ as in  \eqref{eq:conditional_indep}. Let $\pi^*_\casc= (f_{1:T}^*, d^*_{1:T})$ and $ \tpi^*_\casc=(\tf^*_{1:T}, \td^*_{1:T}) $ denote the  Bayes-optimal cascade classifiers from Lemmas~\ref{lem:bayes_opt_casc_V}--\ref{lem:fstar_tilde_def} under the $0$-$1$-$c$ loss and the hinge loss, respectively.  Define the conditional Bayes risks 
  \begin{align}
  \mc{R}_t(h_t; f_t^*)&:=\E\Big[\Ind\big\lbrace f_t^*(X,Z_t)\ne Y\big\rbrace+c_t \mid H_t=h_t\Big],\\
      \mc{R}^\hinge_t(h_t; \tf_t^*)&:= \E\Big[\ell^\hinge\big( \tf_t^*(X,Z_t), Y; c_t\big) \mid H_t=h_t\Big]\,,
  \end{align}
  then 
\begin{align}
    \mc{R}_t^\hinge(h_t; \tf^*_t)&=\tfrac{K}{K-1}\mc{R}_t(h_t; f^*_t)\qquad \forall t\in[T]\label{eq:lem_Rt_link_Rt_hinge}.
    \end{align}
   Moreover, the stage-$t$ value functions of $\tpi_\casc^*$ and $\pi^*_\casc$ satisfy
    \begin{align}
   V_t^{\tpi_\casc^*}(h_t) & =\tfrac{K}{K-1} V_t^{ \pi_\casc^*}(h_t) \qquad \forall t\in[T].\label{eq:lem_Vt_link_Vt_hinge}
\end{align}
\end{lemma}
\begin{proof}
By Lemmas~\ref{lem:bayes_opt_casc_V}--\ref{lem:fstar_tilde_def} together with Lemma \ref{lem:auxiliary_ft_tilde_star}, we have  
    \begin{align*}
\mc{R}_t(h_t; f^*_t)
& = 1-\max_{y\in [K]}p(y\mid x,z_t)+c_t\,,\\
\mc{R}_t^{\hinge}(h_t; \tf_t^*)&=\tfrac{K}{K-1}\left(1-\max_{y\in[K]} p(y\mid x, z_t)+c_t\right)\,,
    \end{align*}
    which proves \eqref{eq:lem_Rt_link_Rt_hinge}. 
To prove \eqref{eq:lem_Vt_link_Vt_hinge}, we use backward induction. \\
    \textbf{Base case $t=T$:} we have 
    \begin{align*}
    V_T^{\tpi_\casc^*}(h_T)=\mc{R}_T^\hinge(h_T; \tf_T^*)\qquad \text{and}\qquad
        V_T^{\pi_\casc^*}(h_T)=\mc{R}_T(h_T; f_T^*),
    \end{align*}
    so applying \eqref{eq:lem_Rt_link_Rt_hinge} gives
     \eqref{eq:lem_Vt_link_Vt_hinge} for $t=T$.\\
     \textbf{Backward induction:} Assume  \eqref{eq:lem_Vt_link_Vt_hinge} holds for $t+1$, we prove it for $t$.  First, by Lemmas~\ref{lem:bayes_opt_casc_V}--\ref{lem:fstar_tilde_def}, since both Bayes-optimal deferral rules  $d_{1:T}^*$ and $\td_{1:T}^*$ are  $\in\{0,1\}$-valued (see \eqref{eq:d_star_casc_V} and \eqref{eq:d_star_tilde}), we have 
     \begin{align}
     V_t^{\pi^*_\casc}(h_t)&=\min\left\lbrace\mc{R}_t(h_t; f_t^*), \E\left[V_{t+1}^{\pi^*_\casc}(H_{t+1})\mid H_t=h_t\right]\right\rbrace,\label{eq:def_Vt_f_star}\\
     V_t^{\tpi^*_\casc}(h_t)&=    \min\left\lbrace\mc{R}_t^\hinge(h_t; \tf_t^*), \E\left[V_{t+1}^{\tpi^*_\casc}(H_{t+1})\mid H_t=h_t\right]\right\rbrace.\label{eq:def_Vt_ftilde_star}
     \end{align}
     Note by \eqref{eq:def_Vt_ftilde_star}, \eqref{eq:lem_Rt_link_Rt_hinge} and \eqref{eq:def_Vt_f_star}, we can write
     \begin{align*}
         V_t^{\tpi_\casc^*}(h_t)&\stackrel{\text{(i)}}{=}\min\left\lbrace\mc{R}_t^\hinge(h_t; \tf_t^*), \E\left[V_{t+1}^{\tpi^*_\casc}(H_{t+1})\mid H_t=h_t\right]\right\rbrace\\
         &\stackrel{\text{(ii)}}{=}\min\left\lbrace \tfrac{K}{K-1} \mc{R}_t(h_t; f_t^*), \E\left[\tfrac{K}{K-1}V_{t+1}^{\pi^*_\casc}(H_{t+1})\mid H_t=h_t\right]\right\rbrace\\
         &=\tfrac{K}{K-1}\min\left\lbrace  \mc{R}_t(h_t; f_t^*), \E\left[ V_{t+1}^{\pi^*_\casc}(H_{t+1})\mid H_t=h_t\right]\right\rbrace\\
         &\stackrel{\text{(iii)}}{=}\tfrac{K}{K-1} V_t^{\pi^*_\casc}(h_t)\,,
     \end{align*}
     where  step (i) follows from \eqref{eq:def_Vt_ftilde_star}, step (ii) follows from \eqref{eq:lem_Rt_link_Rt_hinge} and the induction hypothesis, and step (iii) applies \eqref{eq:def_Vt_f_star}. 
     This completes the induction.\\
     \textbf{Conclusion: } By backward induction, we establish \eqref{eq:lem_Vt_link_Vt_hinge} for all $t\in[T]$.
\end{proof}
    
The next lemma   is an algebraic inequality  linking the hard and soft deferral decisions, which will be useful in the proof of  the main result (Theorem~\ref{thm:calibration_casc}) in Section~\ref{app:proof_calibration_casc} to show that the suboptimality gap under a hard deferral rule is at most \emph{twice} that under its soft counterpart.
\begin{lemma}[Inequality linking soft and hard deferral decisions]\label{lem:small_ineq}
For $\td\in[0,1]$, let $d := \Ind\{\td \geq0.5\}$. For all $\mc{R}, \bar{V}\in \mathbb{R}$, it holds that
\begin{align}
    (1-d) \mc{R} + d\bar{V}- \min(\mc{R}, \bar{V}) \leq 2\Big( (1-\td) \mc{R} +\td \bar{V} - \min(\mc{R}, \bar{V})\Big) \,,\label{eq:small_ineq_new}
\end{align}
where the 2 on the RHS is the smallest valid constant.
\end{lemma}
\begin{proof}  
We first establish two inequalities linking $d$ and $\td$. When $\td\ge0.5$, we have $d=1$ and $2\td\ge1$, so $d\le 2\td$; when $\td<0.5$, $d=0$ so $d\le2\td$ is immediate. The same argument applied to $1-\td$, using $1-d=\Ind\{1-\td> 0.5\}$, gives $1-d\le 2(1-\td)$. 
Now we prove \eqref{eq:small_ineq_new} by considering two cases.\\
\textbf{Case 1. When $\mc{R}\le \bar{V}$:} we have 
\begin{align*}
  (1-d)\mc{R}+d\bar{V}-\min(\mc{R}, \bar{V})&=  (1-d) \mc{R} + d\bar{V}- \mc{R}%  = -d\mc{R} +d\bar{V}  
  = d( \bar{V} -\mc{R} )\\
    &\stackrel{\text{(i)}}{\leq} 2\td  (  \bar{V}  - \mc{R})=2\left(\td\bar{V} +(1-\td)\mc{R}  -\mc{R}\right)\\
    &=2\left( (1-\td) \mc{R} +\td \bar{V}- \min(\mc{R}, \bar{V})\right)
\end{align*}
where step (i) applies $d\le2\td$. \\
\textbf{Case 2. When $\mc{R}>\bar{V}$:} we have 
\begin{align*}
    (1-d) \mc{R} + d\bar{V}- \min(\mc{R}, \bar{V})
    & =(1-d)\mc{R}+d\bar{V}-\bar{V} =(1-d)(\mc{R}-\bar{V})\\
    &\stackrel{\text{(i)}}{\le} 2(1-\td)(\mc{R}-\bar{V})=2\left((1-\td)\mc{R}-\bar{V}+ \td\bar{V}\right) \\
    &= 2\Big((1-\td)\mc{R}+\td\bar{V}-\min(\mc{R}, \bar{V})\Big)
\end{align*}
where step (i) applies $1-d\le 2(1-\td)$. \\
\textbf{Tightness of \eqref{eq:small_ineq_new}:} Finally, the factor of 2 in \eqref{eq:small_ineq_new} is tight and cannot be improved: setting $\mc{R}=0$, $\bar{V}=1$, and $\tilde{d}= 0.5$ (so $d=1)$ gives $\text{LHS}=1$ and $\text{RHS}= 1$.
\end{proof}

\subsection{Proof of  Theorem~\ref{thm:calibration_casc}}\label{app:proof_calibration_casc}
In this section, we prove our main calibration result of Theorem~\ref{thm:calibration_casc}. Like our convention in the main text, we will consistently denote $0$-$1$-$c$ risk and its related quantities without superscripts and    hinge risks and their related quantities    with `hinge' superscripts, e.g., $\mc{R}_\casc$, $\Delta_t$ and $\mc{R}^\hinge_\casc$, $\Delta^\hinge_t$, respectively.

\excessriskcasc*

\begin{proof}In this proof, we write $\varphi(\tpi)=\pi=(f_{1:T}, d_{1:T})$. 
Let $\pi^*_\casc $  and $\tpi^*_\casc $ denote the Bayes solutions under the target $0$-$1$-$c$ loss and the cascade hinge loss, respectively. That is,
\begin{align*}
    \pi^*_\casc&= (f^*_{1:T}, d^*_{1:T})\in \argmin_{\pi \in \Pi_\casc}\mc{R}(\pi),
    \\
    \tpi^*_\casc&= (\tf^*_{1:T}, \td^*_{1:T})\in \argmin_{\tpi \in \wt{\Pi}_\casc}\mc{R}^\hinge_\casc(\tpi).
\end{align*}
Then our end result \eqref{eq:casc_calibration} can be rewritten into
\begin{align}\label{eq:calibration_C1_def}
    \mc{R} (\pi ) - \mc{R}(\pi^*_\casc)
\le C_1 
\Big(
\mc{R}^\hinge_\casc(\tpi) - \mc{R}^{\hinge}_\casc(\tpi^*_\casc)
\Big)\quad\text{with}\quad  C_1=2^{T-1}(K-1)\,.
\end{align}
By Lemma~\ref{lem:global_risk_to_V1} and Corollary~\ref{cor:bayes_to_value},  the risks on both sides of \eqref{eq:calibration_C1_def} equal  their expected  stage-1 value functions, further  reducing \eqref{eq:calibration_C1_def} to 
\begin{align}
    \E\left[V_1^{\pi}(H_1)\right]-\E\left[V_1^{\pi^*_\casc}(H_1)\right] \le C_1  \Big( \E\left[V_1^{\tpi}(H_1)\right] - \E\left[V_1^{\tpi^*_\casc}(H_1)\right]\Big).\label{eq:four_V1_inequality}
\end{align}
To prove \eqref{eq:four_V1_inequality},   it suffices in turn to prove for every $t\in[T]$ and pointwise in $h_t\in\mc{H}_t$, 
\begin{align}\label{eq:Delta_link_Delta}
   \underbrace{V_t^{\pi}(h_t)-V_t^{\pi^*_\casc}(h_t)}_{=:\Delta_t (h_t)}\le C_t \underbrace{\Big(V_t^{\tpi}(h_t) - V_t^{\tpi^*_\casc}(h_t)\Big)}_{= :\Delta_t^\hinge(h_t)} \quad 
\text{with}\quad C_t= 2^{T-t}(K-1)\,.
\end{align}
Here,   $\Delta_t (h_t)$ and $\Delta_t^\hinge(h_t) $ denote the suboptimality gaps in the stage-$t$ value functions associated with the $0$-$1$-$c$ risk and hinge risk, respectively. 
The form of $C_t$ arises from unrolling the recursion backward from the terminal stage, where each step of the induction introduces an additional factor of $2$, giving  $C_T= K-1$ up to $C_1=2^{T-1}(K-1)$. 
In the remainder of the proof, we prove \eqref{eq:Delta_link_Delta} using backward induction.
\paragraph{Step 1. Base case where $t=T$:}
By the definition \eqref{eq:def_VT} of stage-$T$  value functions, 
\eqref{eq:Delta_link_Delta} specialized to $t=T$ reduces to 
\begin{align}
 & \underbrace{\mc{R}_T(h_T; f_T) - \mc{R}_T(h_T; f_T^*)}_{=:\Delta_T(h_T)}\le 
   C_T\underbrace{\Big( \mc{R}_T^\hinge(h_T; \tf_T) - \mc{R}_T^\hinge(h_T; \tf_T^*)\Big)}_{=: \Delta_T^\hinge (h_T)}  \quad \text{with} \quad C_T= K-1,\label{eq:stage_T_calibration}\\
&\text{where}\quad    \Delta_T(h_T)\ge 0\,,\;\Delta_T^\hinge (h_T)\ge 0 \quad \text{for all}\quad h_T\in\mc{H}_T\,.\label{eq:stage_T_calibration_nonneg}
\end{align}
Eq.~\eqref{eq:stage_T_calibration} holds  by the single-stage calibration bound of  Lemma~\ref{lem:K_mistake} with $(x,z_t)$ replaced by $h_t$. In \eqref{eq:stage_T_calibration_nonneg}, $\Delta_T(h_T)\ge 0$ follows from  the Bayes optimality of $f_t^*$ for the 0-1-$c$ loss 
%, which gives   $\mc{R}_t(h_t; f_t) \ge \mc{R}_t(h_t; f_t^*)$  
(see explanation above \eqref{eq:Rt_ht_cascade_general_ft}), and $\Delta_T^\hinge(h_T)\ge 0 $
 follows from the Bayes optimality of  $\tf_T^*$  for the cascade hinge loss. More precisely, $\Delta_T^\hinge(h_T)\ge 0 $  because of  the proof of \cite[Lemma 1]{tarigan2008} and the conditional independence assumption \eqref{eq:conditional_indep}.

\paragraph{Step 2. Backward induction:} Next, for $t\le T-1$, we  show 
that  if  the suboptimality gaps at stage $t+1$ satisfy:
\begin{align}
  &\underbrace{V_{t+1}^{\pi}(h_{t+1})-V_{t+1}^{\pi^*_\casc}(h_{t+1})}_{=:\Delta_{t+1} (h_{t+1})} \leq C_{t+1} \underbrace{\Big(V_{t+1}^{\tpi}(h_{t+1})-V_{t+1}^{\tpi^*_\casc}(h_{t+1})\Big)}_{=: \Delta_{t+1}^{\hinge}(h_{t+1})} \quad  \text{with}\quad  C_{t+1}=2^{T-t-1}(K-1)\,,\label{eq:induction_hypothesis}\\
& \text{where}\quad 
   \Delta_{t+1}(h_{t+1})\ge 0\,, \;  \Delta_{t+1}^\hinge (h_{t+1})\ge 0 \quad \text{for all}\quad h_{t+1}\in\mc{H}_{t+1}\,,\label{eq:induction_hypothesis_nonneg}
\end{align}
then  the corresponding suboptimality gaps at stage $t$  satisfy
\begin{align}\label{eq:induction_hypothesis_t}
 &\underbrace{V_{t}^{\pi}(h_{t})-V_{t}^{\pi^*_\casc}(h_{t}) }_{=:\Delta_t(h_{t})}\leq C_{t} \underbrace{\Big(V_{t}^{\tpi}(h_{t})-V_{t}^{\tpi^*_\casc}(h_{t})\Big)}_{=:\Delta_t^{\hinge}(h_{t}) }
   \quad \text{with}\quad C_t=2C_{t+1},\\
  & \text{where}\quad \Delta_t(h_t)\ge 0 \,,\;\Delta_t^\hinge(h_t)\ge 0\quad \text{for all}\quad h_t\in\mc{H}_t\,.\label{eq:induction_hypothesis_t_nonneg}
\end{align}

\paragraph{Step 2.1. Decomposing  the     suboptimality gap under the $0$-$1$-$c$ loss:} Let $\textsf{m}$ be the cost achieved by an \emph{oracle}   deferral rule that picks  the cheaper routing option at  stage $t$, that is, 
\begin{align}\label{eq:def_m_oracle}
    \textsf{m}:=\min\Big\lbrace \mc{R}_t(h_t; f_t), \E\left[V^{\pi}_{t+1}(H_{t+1})\mid H_t=h_t\right]\Big\rbrace.
\end{align}  For the induction step, we first expand the  $\Delta_t(h_t)$ term on the  LHS of    \eqref{eq:induction_hypothesis_t} by adding  and subtracting  $\textsf{m}$ to obtain
\begin{align}
     \Delta_t (h_t) 
    = V_t^{\pi}(h_t) - V_t^{\pi^*_\casc}(h_t)  
      =  \underbrace{V_t^{\pi}(h_t)  - \textsf{m} }_{=:\Delta_{t,\select}(h_t)} 
      +\underbrace{\textsf{m}- V_t^{\pi^*_\casc}(h_t)}_{=:\Delta_{t, \internal}(h_t)} \,.\label{eq:Delta_ht_decompose_cascade}
\end{align}
Using the definition     of $\textsf{m}$ in \eqref{eq:def_m_oracle}   and  the recursive definition    of the value function in \eqref{eq:def_Vt_new}, we can rewrite the two terms on the RHS of  \eqref{eq:Delta_ht_decompose_cascade} as follows, making their non-negativity explicit:
\begin{align}
    &\Delta_{t, \select}(h_t)=\underbrace{(1 - d_t(h_t)) \mc{R}_t(h_t; f_t) + d_t(h_t) \E\left[V^{\pi}_{t+1}(H_{t+1})\mid H_t = h_t\right]}_{V_t^\pi(h_t)}\nonumber \\
       &\qquad\qquad\qquad\qquad\qquad\quad\qquad  - \underbrace{\min\Big\lbrace \mc{R}_t(h_t; f_t), \E\left[V^{\pi}_{t+1}(H_{t+1})\mid H_t=h_t\right]\Big\rbrace }_{\textsf{m}} \ge0\,, \label{eq:Delta_select_ht}\\
&\Delta_{t, \internal}(h_t)=\underbrace{\min\Big\lbrace \mc{R}_{t}(h_t; f_t), \E\left[V_{t+1}^{\pi}(H_{t+1})\mid H_t=h_t\right]\Big\rbrace}_{\textsf{m}} \nonumber\\
& \qquad\qquad\qquad\qquad\qquad\qquad\quad  -\underbrace{\min\Big\lbrace \mc{R}_{t}(h_t; f_t^*), \E\left[V_{t+1}^{\pi^*_\casc}(H_{t+1})\mid H_t=h_t\right]\Big\rbrace }_{V_t^{\pi^*_\casc}(h_t)}\ge0\,.\label{eq:Delta_internal_ht}
\end{align}
Here,  $\Delta_t (h_t)$ is decomposed  into two terms: $\Delta_{t, \select}(h_t)$   measures the suboptimality gap of   $d_t(h_t)$ relative to the  oracle deferral rule, while $\Delta_{t, \internal}(h_t)$ measures the suboptimality gap of using $\pi$ rather than  $\pi^*_\casc$ for the downstream stages, holding the deferral choice at $t$ as the oracle. 
We have   $\Delta_{t, \select}(h_t)\ge 0$ because  $  \min\{a,b\}\le (1-\lambda)a + \lambda b $ for every $\lambda \in [0,1]$. Moreover,   $\Delta_{t, \internal}(h_t)\ge 0 $ follows from the Bayes optimality of $f_t^*$, which gives   $\mc{R}_t(h_t; f_t) \ge \mc{R}_t(h_t; f_t^*)$, 
and  the induction hypothesis \eqref{eq:induction_hypothesis_nonneg}, which assumes $\Delta_{t+1}(h_{t+1})=V_{t+1}^{\pi}(h_{t+1})-V_{t+1}^{\pi^*_\casc}(h_{t+1})\ge 0$ for all $h_{t+1}\in \mc{H}_{t+1}$.
Since the sum of two non-negative terms is non-negative, we obtain  
\begin{align}\label{eq:Delta_t_nonneg}
    \Delta_t(h_t)\ge 0\,,
\end{align} which is the first part of the conclusion \eqref{eq:induction_hypothesis_t_nonneg} of the induction hypothesis. \\
\textbf{Step 2.2. Replacing hard with soft deferral rules:} Using   \eqref{eq:def_m_oracle}--\eqref{eq:Delta_internal_ht} and applying  Lemma~\ref{lem:small_ineq}, we can  replace  the hard deferral decision $d_t(h_t)$ in $\Delta_{t, \select}(h_t)$  with its soft counterpart $\td_t(h_t)$ at the cost of a factor of 2, yielding  
\begin{align}
   \Delta_t (h_t) &= \Delta_{t, \select}(h_t)+\Delta_{t, \internal}(h_t)\nonumber\\
     &\leq 2 \Big( (1 - \td_t(h_t)) \mc{R}_t(h_t; f_t) + \td_t(h_t) \E\left[V^{\pi}_{t+1}(H_{t+1})\mid H_t = h_t\right]-\textsf{m}  \Big)   +\Big(\textsf{m}    -V_t^{\pi^*_\casc}(h_t) \Big) \label{eq:Delta_01c_decompose_new_1}\\
     &= 2 \Big( (1 - \td_t(h_t)) \mc{R}_t(h_t; f_t)  + \td_t(h_t) \E\left[V^{\pi}_{t+1}(H_{t+1})\mid H_t = h_t\right]   \Big) 
    -\textsf{m}- V_t^{\pi^*_\casc}(h_t)\nonumber\\
    &= 2 \Big( (1 - \td_t(h_t)) \mc{R}_t(h_t; f_t)  + \td_t(h_t) \E\left[V^{\pi}_{t+1}(H_{t+1})\mid H_t = h_t\right]   \Big) 
    - \Delta_{t,\internal}(h_t)-2 V_t^{\pi^*_\casc}(h_t)\nonumber\\
     & \leq 2 \Big( (1 - \td_t(h_t)) \mc{R}_t(h_t; f_t) + \td_t(h_t) \E\left[V_{t+1}^{\pi}(H_{t+1})\mid H_t=h_t\right]  -V_t^{\pi^*_\casc}(h_t) \Big) ,\label{eq:Delta_01c_decompose_new}
\end{align}
where the equalities follow by rearranging terms and   the final inequality~follows~from~$\Delta_{t,\internal}(h_t)\!\ge\!0$.\\
\textbf{Step 2.3. Connecting   $\Delta_t(h_t)$ to $\Delta_t^\hinge(h_t)$:} To help connect   the $0$-$1$-$c$ suboptimality gap $\Delta_t(h_t)$  to its hinge counterpart $\Delta^\hinge_t(h_t)$,  we  rewrite the RHS of  \eqref{eq:Delta_01c_decompose_new} without the factor of 2,  by   adding and subtracting the same terms to obtain
\begin{align}
    \tfrac{1}{2}\Delta_t(h_t)&\le   (1 - \td_t(h_t)) \mc{R}_t(h_t; f_t) + \td_t(h_t) \E\left[V_{t+1}^{\pi}(H_{t+1})\mid H_t=h_t\right]  -V_t^{\pi^*_\casc}(h_t) \nonumber \\
    &=\underbrace{(1 - \td_t(h_t)) \big(\mc{R}_t(h_t; f_t) -\mc{R}_t(h_t; f_t^*)\big)}_{=:\textsf{A}\ge0} + (1-\td_t(h_t))\mc{R}_t(h_t; f_t^*) \nonumber\\
    & \quad +\underbrace{\td_t(h_t)\Big( \E\left[V_{t+1}^{\pi}(H_{t+1})\mid H_t=h_t\right] -\E\left[V_{t+1}^{\pi^*_\casc}(H_{t+1})\mid H_t=h_t\right] \Big)}_{=:\textsf{B}\ge0} \nonumber\\
    &\quad  +\td_t(h_t)\E\left[V_{t+1}^{\pi^*_\casc}(H_{t+1})\mid H_t=h_t\right] - V_t^{\pi^*_\casc}(h_t)\nonumber\\
    &= \textsf{A}+\textsf{B} +\underbrace{ (1-\td_t(h_t))\mc{R}_t(h_t; f_t^*)+\td_t(h_t)\E\left[V_{t+1}^{\pi^*_\casc}(H_{t+1})\mid H_t=h_t\right] - V_t^{\pi^*_\casc}(h_t)}_{=:\textsf{C}\ge0}\nonumber\\
    &= \textsf{A} +\textsf{B} +\textsf{C}\,.\label{eq:half_Delta_ht}
\end{align}
Here, $\textsf{A}\ge0$ because $f_t^*$ is Bayes-optimal. Term  $\textsf{B}\ge 0$ because  the  induction hypothesis \eqref{eq:induction_hypothesis_nonneg}  assumes $\Delta_{t+1}(h_{t+1})=V_{t+1}^{\pi}(h_{t+1})-V_{t+1}^{\pi^*_\casc}(h_{t+1})\ge 0$ for all $h_{t+1}\in \mc{H}_{t+1}$.
Term $\textsf{C}\ge0$ because, using  \eqref{eq:def_Vt_f_star} and the fact that 
$  \min\{a,b\}\le (1-\lambda)a + \lambda b $ for every $\lambda \in [0,1]$, we have
\begin{align}\label{eq:prove_C_non_negative}
    V_t^{\pi^*_\casc} (h_t)
    &=\min\left\lbrace\mc{R}_t(h_t; f_t^*), \E\left[V_{t+1}^{\pi^*_\casc}(H_{t+1})\mid H_t=h_t\right]\right\rbrace\nonumber\\
    &\le (1-\td_t(h_t))\mc{R}_t(h_t; f_t^*)+\td_t(h_t)\E\left[V_{t+1}^{\pi^*_\casc}(H_{t+1})\mid H_t=h_t\right]\,.
\end{align} 

We upper bound terms $\textsf{A}, \textsf{B}, \textsf{C}$  separately with their hinge counterparts. By single-stage calibration (Lemma~\ref{lem:K_mistake}, with $h_t$ in place of $x, z_t$), similar to \eqref{eq:stage_T_calibration}, we have 
\begin{align}
    \textsf{A}\le (K-1) \underbrace{(1-\td_t(h_t))\Big(\mc{R}_t^\hinge(h_t; \tf_t) - \mc{R}_t^\hinge(h_t; \tf^*_t)\Big)}_{=:\tilde{\textsf{A}}} = (K-1) \tilde{\textsf{A}}\,.\label{eq:A_tilde_A}
\end{align}
Since  $\tf_t^*$ is   Bayes-optimal   for the cascade hinge loss, $\mc{R}_t^\hinge(h_t; \tf_t) - \mc{R}_t^\hinge(h_t; \tf^*_t)\ge0$,  and so 
  $\tilde{\textsf{A}}\ge0$. 
By the induction hypothesis \eqref{eq:induction_hypothesis}, we can write 
\begin{align}
    \textsf{B}&= \td_t(h_t)\E\left[V_{t+1}^{\pi}(H_{t+1}) - V_{t+1}^{\pi^*_\casc}(H_{t+1})\mid H_t=h_t\right]\nonumber\\
    &\le\td_t(h_t)C_{t+1}\E\left[V_{t+1}^{\tpi}(H_{t+1}) - V_{t+1}^{\tpi^*_\casc}(H_{t+1})\mid H_t=h_t\right]\nonumber\\
    & = 
    %2^{T-t-1}(K-1) 
    C_{t+1}   \underbrace{\td_t(h_t)\E\left[V_{t+1}^{\tpi}(H_{t+1}) - V_{t+1}^{\tpi^*_\casc}(H_{t+1})\mid H_t=h_t\right]}_{=:\tilde{\textsf{B}}} =C_{t+1} \tilde{\textsf{B}}\,,\label{eq:B_tilde_B}
\end{align}
where $\tilde{\textsf{B}}\ge0$ because
the induction hypothesis \eqref{eq:induction_hypothesis_nonneg}   assumes $\Delta_{t+1}^\hinge(h_{t+1})=V_{t+1}^{\tpi}(h_{t+1})-V_{t+1}^{\tpi^*_\casc}(h_{t+1})\ge 0$ for all $h_{t+1}\in \mc{H}_{t+1}$.

Finally, since $\mc{R}_t(h_t;f_t^*)=\tfrac{K-1}{K}\mc{R}_t^\hinge(h_t; \tf_t^*)$ and $V_t^{ \pi_\casc^*}(h_t) =\tfrac{K-1}{K} V_t^{\tpi_\casc^*}(h_t)  $ by Lemma~\ref{lem:K/(K-1)_factor} for all $t\in[T]$, we can write 
\begin{align}
    \textsf{C}&=(1-\td_t(h_t))\tfrac{K-1}{K}\mc{R}_t^\hinge(h_t; \tf_t^*)+\td_t(h_t)\tfrac{K-1}{K}\E\left[V_{t+1}^{\tpi^*_\casc}(H_{t+1})\mid H_t=h_t\right]-\tfrac{K-1}{K}V_t^{\tpi_\casc^*}(h_t)\nonumber \\
    &=\tfrac{K-1}{K}\underbrace{\Big((1-\td_t(h_t))\mc{R}_t^\hinge(h_t; \tf_t^*)+\td_t(h_t) \E\left[V_{t+1}^{\tpi^*_\casc}(H_{t+1})\mid H_t=h_t\right]- V_t^{\tpi_\casc^*}(h_t)\Big)}_{=:\tilde{\textsf{C}}} = \tfrac{K-1}{K} \tilde{\textsf{C}}\,,\label{eq:C_tilde_C}
\end{align}
where $\tilde{\textsf{C}}\ge0$ following the same argument as that in \eqref{eq:prove_C_non_negative}  for $\textsf{C}\ge0$, using \eqref{eq:def_Vt_ftilde_star} instead of \eqref{eq:def_Vt_f_star}.\\
\textbf{Step 2.4. Assembling all terms:} Finally, recall from \eqref{eq:induction_hypothesis} that $ C_{t+1}=2^{T-t-1}(K-1) $. Since $t\le T-1$ implies $2^{T-t-1}\ge 1$,    the pre-multiplicative factors  $  K-1 ,  C_{t+1} $, and $\frac{K-1}{K}$ in   \eqref{eq:A_tilde_A}, \eqref{eq:B_tilde_B}, and \eqref{eq:C_tilde_C}, respectively,  are all dominated by   $ C_{t+1} $.  Together with  $\tilde{\textsf{A}}, \tilde{\textsf{B}}, \tilde{\textsf{C}}\ge 0$, we have 
\begin{align*}
    \textsf{A}\le C_{t+1}\tilde{\textsf{A}},\quad  \textsf{B}\le C_{t+1}\tilde{\textsf{B}}, \quad \text{and}\quad \textsf{C}\le  C_{t+1}\tilde{\textsf{C}}\,.
\end{align*}  Substituting this 
into \eqref{eq:half_Delta_ht} and rearranging yields
\begin{align}\label{eq:A_B_C}
     \Delta_t(h_t)  &\leq 2\big(\textsf{A} + \textsf{B} + \textsf{C}\big) 
     \le 2\cdot C_{t+1}\big(\tilde{\textsf{A}} + \tilde{\textsf{B}} + \tilde{\textsf{C}}\big) = C_t\big(\tilde{\textsf{A}} + \tilde{\textsf{B}} + \tilde{\textsf{C}}\big)\,.
     \end{align}
  By the definition of $\tilde{\textsf{A}}, \tilde{\textsf{B}}$,  and $\tilde{\textsf{C}}$ from \eqref{eq:A_tilde_A}, \eqref{eq:B_tilde_B}, and \eqref{eq:C_tilde_C}, we can write 
  %Rearranging and simplifying gives
\begin{align}
    \tilde{\textsf{A}} + \tilde{\textsf{B}} + \tilde{\textsf{C}}
   &=(1-\td_t(h_t))\Big(\mc{R}_t^\hinge(h_t; \tf_t) - \mc{R}_t^\hinge(h_t; \tf^*_t)\Big)\nonumber \\
   &\quad  + \td_t(h_t)\E\left[V_{t+1}^{\tpi}(H_{t+1}) - V_{t+1}^{\tpi^*_\casc}(H_{t+1})\mid H_t=h_t\right]\nonumber \\
   &\quad +\Big((1-\td_t(h_t))\mc{R}_t^\hinge(h_t; \tf_t^*)+\td_t(h_t) \E\left[V_{t+1}^{\tpi^*_\casc}(H_{t+1})\mid H_t=h_t\right]- V_t^{\tpi_\casc^*}(h_t)\Big)\nonumber \\
   &=(1-\td_t(h_t)) \mc{R}_t^\hinge(h_t; \tf_t)  + \td_t(h_t)\E\left[V_{t+1}^{\tpi}(H_{t+1})  \mid H_t=h_t\right]   - V_t^{\tpi_\casc^*}(h_t) \nonumber \\
   &\stackrel{\text{(i)}}{=} V_t^{\tpi}(h_t) - V_t^{\tpi_\casc^*}(h_t) \stackrel{\text{(ii)}}{=} \Delta_t^\hinge(h_t)\,,\label{eq:tilde_A_B_C}
\end{align}
  where   step (i) applies the definition of $V_t^{\tpi}(h_t)$ (see \eqref{eq:def_Vt_new}), and step (ii)  uses the definition of $\Delta_t^\hinge(h_t)$ in \eqref{eq:induction_hypothesis_t}. Combining \eqref{eq:A_B_C} and \eqref{eq:tilde_A_B_C} proves  \eqref{eq:induction_hypothesis_t}. Together with the result $\Delta_t(h_t)\ge 0$ established  in \eqref{eq:Delta_t_nonneg}, this implies   $\Delta_t^\hinge(h_t)\ge 0$, proving the second part of \eqref{eq:induction_hypothesis_t_nonneg}. Taken together, we have  completed the backward~induction~step. 
\paragraph{Step 3. Conclusion:} Applying this argument inductively down to $t=1$ yields  \eqref{eq:Delta_link_Delta} for all $t\in[T]$. Setting $t=1$ in \eqref{eq:Delta_link_Delta} and taking total expectation over $H_1$ proves \eqref{eq:calibration_C1_def}, which in turn proves   \eqref{eq:casc_calibration}.
\end{proof}

\subsection{Lower Bound}
We prove the following lower bound to show that  $2^{T-1}(K-1)$ is the smallest constant for which the upper bound in Theorem~\ref{thm:calibration_casc} holds \emph{uniformly} over all data distributions $P$
  in \eqref{eq:def_P} satisfying \eqref{eq:conditional_indep},  cost vectors $c=\big[
    \kappa_1, \kappa_1+\kappa_2, \dots, \sum_{t=1}^T\kappa_t\big]^\top$  with $\kappa_1=0$ and $\kappa_t\ge \kappa_{t-1}$, and  multi-model classifiers  $\tpi\in \wt{\Pi}_\casc$.

\begin{proposition}[Lower bound]\label{prop:LB}
Let $K\ge 2$ and $T\ge 2$. For every $C<2^{T-1}(K-1)$, there exists a distribution  $P$  satisfying \eqref{eq:conditional_indep},   an admissible  cost vector $c$,  and   $\tpi\in \wt{\Pi}_\casc$ such that 
    \begin{align}
        \mc{R}\big(\varphi(\tpi)\big) - \mc{R}^*_\casc > C\Big(\mc{R}^\hinge_\casc(\tpi) - \mc{R}^{\hinge*}_\casc\Big)\,.
    \end{align}
    
\end{proposition}

\begin{proof} We construct a bad instance of $P, c$ and $ \tpi$, and show that the excess risks of the bad instance  have a ratio that approaches $2^{T-1}(K-1)$.\\
\textbf{Step 1. Choose $P$ and $c$ for  the bad instance:} Set all incremental costs to zero, i.e., $\kappa_t=0$ for all $t\in[T]$, giving the all-zero cost vector $c= \bzero_T$. Let $X$ have any distribution on $\mc{X}$, and let $Z_1, Z_2, \dots, Z_T$ be constant.     Let $Y$ be independent of $X$ with 
\begin{align}\label{eq:def_py}
   p_y:= \Pr(Y=y)=\begin{cases}
        \tfrac{1}{K-1}  &\text{for}\quad  y\in [K-1]\,,\\
        0  &\text{for}\quad  y=K\,.
    \end{cases}
\end{align}
This ensures $Y$ to be independent of the full history $H_T=(X, Z_{1:T})$. In particular, \eqref{eq:conditional_indep} holds and $p(y\mid x, z_{1:t}) = p(y\mid x, z_t)=p_y$ for all $t\in [T]$.

 Using 
 the definition of $\ell^\hinge$ from \eqref{eq:stage_wise_surrogate},     for a score vector $v\in \mathbb{R}^K$ with $\sum_{k\in[K]} v_k=0$, we can write  that 
 \begin{align}\label{eq:ell_bar_def}
     \ell^\hinge(v, y; 0) = \sum_{k\in[K]\setminus y}
\underbrace{\Big[v_k + \tfrac{1}{K-1}\Big]_+}_{=: a_k}
+ \tfrac{K}{K-1}\cdot 0 = \sum_{k\in[K]\setminus y} a_k = :\bar{\ell}^\hinge (v, y)  \,.
 \end{align}
 Define $L(v)$ as follows, and expanding and rearranging gives
\begin{align}\label{eq:L(v)}
    L(v):=
    %\sum_{y\in [K]} p_y \ell_t^\hinge (v, y; 0) =
    \sum_{y\in[K]}p_y \bar{\ell}^\hinge (v, y) %&= \sum_{y\in [K]} p_y \Bigg(\sum_{k\in[K]\setminus y}\underbrace{\Big[v_k + \tfrac{1}{K-1}\Big]_+}_{=: a_k}+ \tfrac{K}{K-1}\cdot 0\Bigg) \nonumber\\
&= \sum_{y\in [K]} p_y  \sum_{k\in[K]\setminus y}
a_k 
\stackrel{\text{(i)}}{=}  \sum_{k\in[K] }  a_k\sum_{y\in [K]\setminus k } p_y   \nonumber\\
&= \sum_{k\in[K] } a_k(1-p_k) 
\stackrel{\text{(ii)}}{=} \sum_{k=1}^{K-1}a_k  \cdot \tfrac{K-2}{K-1} + a_K \,,
\end{align}
where step (i) exchanges summations, and  step  (ii) applies \eqref{eq:def_py}. Moreover, we can lower bound $L(v)$ by
\begin{align}\label{eq:L(v)_LB}
L(v) \stackrel{\text{(i)}}{\ge} \tfrac{K-2}{K-1}\sum_{k\in[K]}a_k 
\stackrel{\text{(ii)}}{\ge} \tfrac{K-2}{K-1}\sum_{k\in[K]}\Big[v_k + \tfrac{1}{K-1}\Big]
\stackrel{\text{(iii)}}{=}  \tfrac{K-2}{K-1} \cdot K\cdot  \tfrac{1}{K-1} = \tfrac{K(K-2)}{(K-1)^2}=:L^*\,,
\end{align}
where step (i)  uses $a_k\ge 0$ for all $k\in[K]$,   step (ii) uses $[x]_+\ge x$ for any $x\in \mathbb{R}$, and step (iii) applies $\sum_{k\in[K]}v_k=0$. 
Define a score vector $v^*\in \mathbb{R}^K$  by 
$v^*_1:=1$ and $v^*_k:=-\frac{1}{K-1}$ for all $2\le k\le K$, which satisfies $\sum_{k\in[K]}v_k^*=0$.  Plugging $v=v^*$ into \eqref{eq:L(v)} and applying $p_1=\tfrac{1}{K-1}$  from \eqref{eq:def_py} obtains
\begin{align}\label{eq:v_star_def}
   L(v^*) = %\tfrac{K}{K-1}\cdot (1-p_1) + \sum_{k=2}^K0\cdot (1-p_k)  =
   \tfrac{K}{K-1} \tfrac{K-2}{K-1} + 0
   %= \tfrac{K}{K-1}\left(1-\tfrac{1}{K-1}\right)
   =\tfrac{K(K-2)}{(K-1)^2}=L^*\,.
\end{align}
\textbf{Step 2. Compute   Bayes risks $\mc{R}^*$ and $\mc{R}^{\hinge*}_\casc$:} We now  show that  in the bad instance that we have set up, the  Bayes 0-1-$c$ risk  and the Bayes cascade hinge risk are $\mc{R}^*=\frac{K-2}{K-1}$ and $\mc{R}^{\hinge*}_\casc =L^*$, respectively.

First, using the definition of the 0-1-$c$ risk  $\mc{R}(\pi)$ in \eqref{eq:multi_01c}--\eqref{eq:def_R(f;c)},
% and \eqref{eq:def_py}
we can express $\mc{R}(\pi)$ in the bad instance  as  
\begin{align}\label{eq:R_pi_hard_case}
    \mc{R}(\pi) &= \E\Big[\Ind\{f_{r(H)} (X, Z_{r(H)}) \ne Y\} + c_{r(H)}\Big] \nonumber\\
    &=\Pr\Big(f_{r(H)} (X, Z_{r(H)}) \ne Y\Big)=1-\Pr\Big(Y=f_{r(H)} (X, Z_{r(H)})\Big)\,,
\end{align}
and, using \eqref{eq:def_py}, we can lower bound it by
\begin{align}
     \mc{R}(\pi) \ge 1- \max_{y\in [K]} p_y =1- \tfrac{1}{K-1}=\tfrac{K-2}{K-1}\,.
\end{align}
Note that equality holds for $f_t \equiv 1$ for all $t\in[T]$ and  any router $r$, attaining  the lower bound $\tfrac{K-2}{K-1}$\,. Hence,  the Bayes 0-1-$c$ risk is 
\begin{align}\label{eq:def_R_star_casc_LB}
\mc{R}^*_\casc=\tfrac{K-2}{K-1}\,.
\end{align} 

Second, using the definition of the hinge risk $\mc{R}^\hinge_\casc(\tpi)$ below  \eqref{eq:surrogate_hinge_casc}--\eqref{eq:r_tilde_t_def}, we can express $\mc{R}^\hinge_\casc(\tpi)$ in the bad instance as
\begin{align}
\mc{R}^\hinge_\casc(\tpi)
&=\E\left[\sum_{t=1}^T\tr_t(H_t)\ell^\hinge\big(\tf_t(X,Z_{t}),Y; c_t\big)\right]  \nonumber\\
& 
\stackrel{\text{(i)}}{=} \E\left[\sum_{t=1}^T\tr_t(H_t)\bar{\ell}^\hinge \big(\tf_t(X,Z_{t}),Y\big)\right]\nonumber\\
& = \sum_{t=1}^T\E\left[\tr_t(H_t)\E\left[\bar{\ell}^\hinge \big(\tf_t(X,Z_{t}),Y\big)\,\Big|\, H_t\right]\right]\nonumber\\
& \stackrel{\text{(ii)}}{=} \sum_{t=1}^T\E\left[\tr_t(H_t) \sum_{y\in[K]}p_y\bar{\ell}^\hinge\big( \tf_t(X,Z_{t}),y\big)\right]\nonumber\\
& \stackrel{\text{(iii)}}{=} \sum_{t=1}^T\E\left[\tr_t(H_t)  L\big(\tf_t(X,Z_{t})\big)\right]\label{eq:R_hinge_casc_hard_case}\,,
\end{align}
 where step (i) applies the definition of $\bar{\ell}^{\hinge}$ in \eqref{eq:ell_bar_def}, step (ii) uses $ \Pr(Y=y\mid H_t)=p_y$, and step (iii) uses the definition of $L$ in \eqref{eq:L(v)}. 
Using \eqref{eq:L(v)_LB} and $\sum_{t=1}^T \tr_t(H_t)=1$, we  can further lower bound $\mc{R}^\hinge_\casc(\tpi)$ by 
\begin{align}
   \mc{R}^{\hinge}_\casc (\tpi) \ge   \sum_{t=1}^T\E \left[\tr_t(H_t)   L^*\right] =  \E\left[\sum_{t=1}^T \tr_t(H_t)\right]\cdot L^* =L^* \,.
\end{align}
% step (iv) applies   \eqref{eq:L(v)_LB}, and step (v) uses $\sum_{t=1}^T \tr_t(H_t)=1$ for any $X, Z_1, \dots, Z_T$. 
Note that \eqref{eq:v_star_def} implies that equality holds for $\tf_t\equiv v^*$ for all $t\in[T]$ and for any soft router  $\tr_1, \dots, \tr_T$. Hence, the Bayes cascade hinge risk is 
 \begin{align}\label{eq:def_R_star_hinge_casc_LB}
\mc{R}^{\hinge*}_\casc=L^*. 
 \end{align}

 \paragraph{Step 3. Choose  $ \tpi\in \wt{\Pi}_\casc$ for  the bad instance:} Recall that $v^*\in \mathbb{R}^K$  is a score vector with  $v^*_1=1$ and $v^*_k=-\frac{1}{K-1}$ for all $2\le k\le K$. Fix $\epsilon\in [0,\tfrac{1}{2}]$ and $s\in (0, \frac{1}{K-1}]$. Define  a score vector  $v^{(s)}\in \mathbb{R}^K$  by   $v^{(s)}_k := -s$ for $k\in[K-1]$ and $v_K^{(s)} :=(K-1)s$,  satisfying  $\sum_{k=1}^K v_k^{(s)}=0$.  Then, define $\tpi =(\tf_{1:T}, \td_{1:T})$ as 
 %follows: for any $x\in\mc{X}, z_t\in \mc{Z}_t$, define 
 \begin{align}
\tf_t \equiv v^*\quad \text{for}\; t\le T-1, \quad \text{and}\quad \tf_T  \equiv v^{(s)}\,,
 \end{align}
 and  \begin{align}
     \td_t\equiv \tfrac{1}{2}+\epsilon \quad \text{for}\; t\le T-1, \quad \text{and}\quad \td_T\equiv 0\,.
 \end{align}
This constructed $\tpi$ belongs to $ \wt{\Pi}_\casc$ by definition of $\wt{\Pi}_\casc$. The probability of exiting  at stage $T$ is thus 
\begin{align}\label{eq:tilde_rT_LB}
    \tr_T\equiv (\tfrac{1}{2}+\epsilon)^{T-1}
\end{align} and, by the definition of $L(v)$ in \eqref{eq:L(v)}, we have 
\begin{align}
    L(v^{(s)}) &= \sum_{k=1}^{K-1} \Big[v_k^{(s)} + \tfrac{1}{K-1}\Big]_+\cdot \tfrac{K-2}{K-1} + \Big[v_K^{(s)} + \tfrac{1}{K-1}\Big]_+\nonumber\\
    & = \sum_{k=1}^{K-1} \Big[-s + \tfrac{1}{K-1}\Big]_+\cdot \tfrac{K-2}{K-1} + \Big[(K-1)s + \tfrac{1}{K-1}\Big]_+\nonumber\\
    & \stackrel{\text{(i)}}{=} \sum_{k=1}^{K-1} \Big[-s + \tfrac{1}{K-1}\Big]\cdot \tfrac{K-2}{K-1} +\Big[(K-1)s + \tfrac{1}{K-1}\Big] \nonumber\\&= \tfrac{K-2}{K-1} -s(K-2) + s(K-1) +\tfrac{1}{K-1} = s + 1\,,\label{eq:L_vs}
\end{align}
where step (i) holds because $s\in (0, \tfrac{1}{K-1}]$ by construction. \\
\textbf{Step 4. Compute the risks $\mc{R}(\varphi(\tpi))$ and $\mc{R}^\hinge_\casc(\tpi)$:} We are now ready to derive   the cascade hinge risk of the constructed  $\tpi$ and the 0-1-$c$ risk of its hard prediction $\varphi(\tpi)$. First,  
together with  the expression for $\mc{R}^\hinge_\casc(\tpi)$ in \eqref{eq:R_hinge_casc_hard_case}, we can write 
\begin{align}
\mc{R}^\hinge_\casc(\tpi)
&=\sum_{t=1}^T\E\left[\tr_t(H_t)  L\big(\tf_t(X,Z_{t})\big)\right] =
\sum_{t=1}^{T-1}\E\left[\tr_t(H_t)  L\big(v^*\big)\right] + \E\left[\tr_T(H_T) L(v^{(s)})\right] \nonumber\\
&\stackrel{\text{(i)}}{=} \Big(1-\E\big[\tr_T(H_T)\big]\Big)L^* + \E\big[\tr_T(H_T)\big] (s+1)\nonumber\\
&= L^*+ \E\big[\tr_T(H_T)\big](s+1-L^*)\stackrel{\text{(ii)}}{=}L^* + (\tfrac{1}{2}+\epsilon)^{T-1}(s+1-L^*)\,,\label{eq:R_casc_hinge_LB}
\end{align}
where step (i) uses $L(v^*)=L^*$ from \eqref{eq:v_star_def} and $L(v^{(s)}) =s+1$ from \eqref{eq:L_vs},  and step (ii) applies  \eqref{eq:tilde_rT_LB}. 

Second, under the soft-to-hard map $\varphi$, as defined in \eqref{eq:hard_ft_from_soft_ft} and \eqref{eq:hard_r_casc}, the hard classifiers are $f_t\equiv 1$ for $t\le T-1$ and $f_T\equiv K$ and, induced by the deferral rules,    the hard routing rule is $r(h)=T$ for all $h\in\mc{H}$.
%almost surely, 
This implies that 
\begin{align}\label{eq:R_hard_pi}
    \mc{R}(\varphi(\tpi)) =\Pr\Big(f_{r(H)}(X, Z_{r(H)})\ne Y\Big)= \Pr\Big(f_T(X, Z_T)\ne Y\Big) = \Pr\big(Y\ne K\big) = 1\,,
\end{align}
where the last step uses \eqref{eq:def_py}.\\
\textbf{Step 5. Compute  excess risks and prove lower bound:} The final step computes and links the excess risks. 
Combining  \eqref{eq:def_R_star_casc_LB} and \eqref{eq:R_hard_pi} gives the excess 0-1-$c$ risk:
\begin{align}\label{eq:excess_01c_LB}
    \mc{R}(\varphi(\tpi)) -\mc{R}^*_\casc = 1-\tfrac{K-2}{K-1} =\tfrac{1}{K-1}\,.
\end{align}
Combining \eqref{eq:def_R_star_hinge_casc_LB} and \eqref{eq:R_casc_hinge_LB} gives the excess cascade hinge risk:
\begin{align}\label{eq:excess_hinge_LB}
  \mc{R}^\hinge_\casc(\tpi) -\mc{R}^{\hinge*}_\casc
  &= L^* + (\tfrac{1}{2}+\epsilon)^{T-1}(s+1-L^*)- L^*
  =(\tfrac{1}{2}+\epsilon)^{T-1}(s+1-L^*)\nonumber\\
  & =(\tfrac{1}{2}+\epsilon)^{T-1}\Big[s+1-\tfrac{K(K-2)}{(K-1)^2}\Big] = (\tfrac{1}{2}+\epsilon)^{T-1}\Big[s+\tfrac{1}{(K-1)^2}\Big]\,,
\end{align}
where the last step applies \eqref{eq:L(v)_LB}. Finally, taking ratio of \eqref{eq:excess_01c_LB} and \eqref{eq:excess_hinge_LB} gives
\begin{align}
    \frac{ \mc{R}(\varphi(\tpi)) -\mc{R}^*_\casc}{\mc{R}^\hinge_\casc(\tpi) -\mc{R}^{\hinge*}_\casc} = \frac{\tfrac{1}{K-1}}{(\tfrac{1}{2}+\epsilon)^{T-1}\Big[s+\tfrac{1}{(K-1)^2}\Big]} = \frac{2^{T-1}(K-1)}{(1+2\epsilon)^{T-1} \Big[s(K-1)^2+ 1 \Big]} =: \textsf{R}\,.\label{eq:ratio_LB}
\end{align}
As $\epsilon, s\to 0$, the denominator tends to 1, and so $\textsf{R}\to 2^{T-1}(K-1)$, implying that $2^{T-1}(K-1)$ is the constant in the lower bound. More concretely,  for every $C<2^{T-1}(K-1)$, we can choose $\epsilon$ and $s$ small enough that $\textsf{R}>C$. For example, setting \begin{align*}
\epsilon=0\quad \text{and}\quad      s=\min\left\lbrace \tfrac{1}{K-1}\,, \; \frac{2^{T-1} (K-1)/C-1}{2(K-1)^2}\right\rbrace>0\,,
\end{align*} 
and using the shorthand $\textsf{r}:=2^{T-1} (K-1)/C>1$ gives
\begin{align*}
  s(K-1)^2+1 \le  \frac{\textsf{r}-1}{2} +1= \frac{\textsf{r}+1}{2}    < \frac{\textsf{r}+\textsf{r}}{2}   =\textsf{r}\,.
\end{align*}
This implies that 
\begin{align}
  \textsf{R}=  \frac{2^{T-1}(K-1)}{1\cdot   \Big[s(K-1)^2+ 1 \Big]} > \frac{2^{T-1}(K-1)}{  \textsf{r}}  =  C \,.
\end{align}
  This proves the proposition.
\end{proof}

\subsection{An Upper Bound in Favorable Instances} 

Proposition~\ref{prop:LB} shows that the constant in Theorem~\ref{thm:calibration_casc} cannot be improved in bad problem instances, where every soft deferral decision is maximally indecisive, that is, $\td_t\equiv 0.5$. We now  show that the   constant in Theorem~\ref{thm:calibration_casc} is much smaller in more favorable problem instances, where the soft deferral decisions stay away from 0.5. 
To begin, we prove a refined, margin-dependent 
 version of Lemma~\ref{lem:small_ineq} to connect soft and hard deferral decisions. 
\begin{lemma}[Margin-dependent   inequality linking soft and hard deferral decisions]\label{lem:refined}
   For $\td\in[0,1]$, let $d:=\Ind\{\td\ge 0.5\}$ and $m:=|\td-0.5|$. For all $\mc{R}, \bar{V}\in \mathbb{R}$, 
    \begin{align}\label{eq:relate_soft_hard_r}
        (1-d)\mc{R} + d\bar{V} -\min( \mc{R}, \bar{V})\;\le\; \tfrac{2}{1+2m} \Big[(1-\td) \mc{R} +\td\bar{V} -\min(\mc{R}, \bar{V})\Big]\,,
    \end{align}
    where $\tfrac{2}{1+2m}\in[1,2]$ is the smallest constant valid at the given $m$. 
\end{lemma}
\begin{proof}
   Let $[x]_+:=\max(x,0)$. Using $m:=|\td-0.5|$, we can define and write
   \begin{align}
       q:=\max(\td, 1-\td)=\tfrac{1}{2}+m\,, \quad \tfrac{1}{q}=\tfrac{2}{1+2m}.
   \end{align}
For brevity, we use shorthand $H:=(1-d)\mc{R} + d\bar{V} -\min( \mc{R}, \bar{V})$ and  $S:=(1-\td) \mc{R} +\td\bar{V} -\min(\mc{R}, \bar{V})$, so the desired bound   in \eqref{eq:relate_soft_hard_r} can be equivalently written as 
\begin{align}\label{eq:H_S/q}
    H\le \tfrac{1}{q}S\,,
\end{align}
which we prove below. 
To establish  tightness of \eqref{eq:relate_soft_hard_r}, it therefore suffices to construct a case in which  \eqref{eq:H_S/q} holds with equality.

\textbf{Step 1:} We first rewrite $H$ and $S$ as follows, replacing the hard-to-handle  $\min(\mc{R}, \bar{V})$ terms  with terms involving $[\cdot ]_+$.
%We first prove that 
To rewrite $H$, we consider two cases. If $\mc{R}\le \bar{V}$, then $\min(\mc{R}, \bar{V})=\mc{R}$, which implies that  
\begin{align}
   (1-d)(\mc{R} -\bar{V})_++d(\bar{V} - \mc{R})_+&=(1-d)\cdot 0 + d (\bar{V}-\mc{R})\nonumber\\
   &=(1-d)\mc{R} + d\bar{V} -\min(\mc{R}, \bar{V})=H\,.
\end{align}
Similarly, if $\mc{R}>\bar{V}$, then $\min(\mc{R}, \bar{V})=\bar{V}$, which implies that 
\begin{align}
    (1-d)(\mc{R} -\bar{V})_++d(\bar{V} - \mc{R})_+&=(1-d)(\mc{R}-\bar{V}) + d\cdot 0\nonumber\\
    &=(1-d)\mc{R} + d\bar{V} -\min(\mc{R}, \bar{V})=H\,.
\end{align}
Together, we have shown
\begin{align}\label{eq:rewrite_H}
    (1-d)(\mc{R} -\bar{V})_++d(\bar{V} - \mc{R})_+=H\,.
\end{align}
Applying the same argument to $S$ obtains
\begin{align}\label{eq:rewrite_S}
    (1-\td) (\mc{R}-\bar{V})_++\td(\bar{V} - \mc{R})_+=S\,.
\end{align}

\textbf{Step 2:} Using \eqref{eq:rewrite_H} and \eqref{eq:rewrite_S}, we can now prove \eqref{eq:H_S/q}. Again, consider two cases: If $\td< 0.5$, then $d=0$ and $q=\max(\td, 1-\td)=1-\td$. Hence, substituting this into \eqref{eq:rewrite_S} and using the fact that $[\cdot]_+\ge0$ gives 
\begin{align}
    S=(1-\td) (\mc{R}-\bar{V})_++\td(\bar{V} - \mc{R})_+&=q(\mc{R}-\bar{V})_++\td(\bar{V} - \mc{R})_+\nonumber\\
    &\ge q(\mc{R}-\bar{V})_+\nonumber\\
    &=q\Big( (1-d)(\mc{R}-\bar{V})_+ + d(\bar{V}-\mc{R})_+\Big)=qH\,,\label{eq:S_qH_1}
\end{align}
where the last step applies \eqref{eq:rewrite_H}. When  $\mc{R}>\bar{V}$, \eqref{eq:S_qH_1} holds with equality. 

Analogously, if $\td\ge 0.5$,  then $d=1$ and $q=\td$, so substituting this into \eqref{eq:rewrite_S} and using  $[\cdot]_+\ge0$ gives 
\begin{align}
    S=(1-\td) (\mc{R}-\bar{V})_++\td(\bar{V} - \mc{R})_+&=(1-\td) (\mc{R}-\bar{V})_++q(\bar{V} - \mc{R})_+\nonumber\\
    &\ge q(\bar{V} - \mc{R})_+\nonumber\\
    & = q\Big(d(\bar{V} - \mc{R})_++ (1-d)(\mc{R}-\bar{V})_+\Big) = qH\,,\label{eq:S_qH_2}
\end{align}
where the last step applies \eqref{eq:rewrite_H}. When $\mc{R}<\bar{V}$, \eqref{eq:S_qH_2} holds with equality. 

Taken together, \eqref{eq:S_qH_1} and \eqref{eq:S_qH_2} establish  \eqref{eq:H_S/q} and show that the inequality is tight. This concludes the proof.

\end{proof}
Note that setting $m=0$ in \eqref{eq:relate_soft_hard_r} recovers \eqref{eq:small_ineq_new} in Lemma~\ref{lem:small_ineq}. 

\begin{proposition}[Calibration for decisive soft deferral rules]\label{prop:tighter_UB}
    Assume  $Y\perp Z_{1:t-1}\mid (X, Z_t)$ as in \eqref{eq:conditional_indep}, and let   $
\mc{R}^*_\casc$
%:=\min_{ \pi \in \Pi_{\casc}} \mc{R}(\pi)$ 
and  $\mc{R}^{\hinge*}_\casc$ be as in Theorem~\ref{thm:calibration_casc}. 
% :=
% \min_{\tpi \in \wt{\Pi}_{\casc}} \mc{R}^\hinge_\casc(\tpi)$.
Let $m\in[0, 0.5]$. Then, for any $\tpi \in \wt{\Pi}_\casc$ whose deferral rules satisfy $|\td_t(H_t) -0.5|\ge m$ almost surely for all $t\in[T-1]$,  
\begin{align}
\mc{R}\big(\varphi(\tpi)\big) \!- \!\mc{R}^*_\casc
\le
 \big(\tfrac{2}{1+2m}\big)^{T-1}(K\!-\!1) 
\Big(
\mc{R}^{\hinge}_\casc(\tpi)\! -\! \mc{R}^{\hinge*}_\casc
\Big).  \label{eq:tighter_UB}
\end{align}
\end{proposition}
\begin{proof}
 Let $\rho \coloneqq \frac{2}{1+2m} \in [1, 2]$. 
The proof follows that of Theorem~\ref{thm:calibration_casc} in Appendix~\ref{app:proof_calibration_casc},
with  the constants  $C_{t+1}$ and $ C_t$  in \eqref{eq:induction_hypothesis} and \eqref{eq:induction_hypothesis_t} changed  into  $C_{t+1}=\rho^{T-t-1}(K-1)$  and $C_t=\rho C_{t+1}$, and   one subsequent  change in \textbf{Step 2.2}.  The rest of the proof stays unchanged. 

In \textbf{Step 2.2}, we apply Lemma~\ref{lem:refined} in place of
Lemma~\ref{lem:small_ineq}.
By Lemma~\ref{lem:refined}, the factor of 2 in \eqref{eq:Delta_01c_decompose_new_1} 
becomes $2/(1+2|\td_t(h_t)-0.5|)$, giving the  following counterpart of \eqref{eq:Delta_01c_decompose_new_1}: 
\begin{align}
    \Delta_t(h_t)
    &\le \frac{2}{1+2|\td_t(h_t)-0.5|} \underbrace{\Big((1-\td_t(h_t))\mc{R}_t(h_t; f_t) + \td_t(h_t) \E\big[V_{t+1}^\pi (H_{t+1}) \mid H_t=h_t\big] -\textsf{m}\Big)}_{=:S_t(h_t)} \nonumber\\
    & \quad +\Big(\textsf{m}    -V_t^{\pi^*_\casc}(h_t) \Big)\,,
\end{align}
where recall from \eqref{eq:def_m_oracle} that $\textsf{m}:=\min\Big\lbrace \mc{R}_t(h_t; f_t), \E\left[V^{\pi}_{t+1}(H_{t+1})\mid H_t=h_t\right]\Big\rbrace$ and from \eqref{eq:Delta_ht_decompose_cascade} and \eqref{eq:Delta_internal_ht} that $\Delta_{t, \internal}(h_t)=\textsf{m} -V_t^{\pi^*_\casc}(h_t)\ge 0$. Note that  $S_t(h_t)\ge 0$ since  $  \min\{a,b\}\le (1-\lambda)a + \lambda b $ for every $\lambda \in [0,1]$. Thus, together with $2/(1+2|\td_t(h_t)-0.5|) \le \rho$ 
for  almost every $h_t$ by
assumption,  we can   write 
\begin{align}
     \Delta_t(h_t) 
     & \le \rho \Big((1-\td_t(h_t))\mc{R}_t(h_t; f_t) + \td_t(h_t) \E\big[V_{t+1}^\pi (H_{t+1}) \mid H_t=h_t\big] - \textsf{m}\Big)+\Big(\textsf{m}    -V_t^{\pi^*_\casc}(h_t) \Big)\nonumber\\
     & = \rho \Big((1-\td_t(h_t))\mc{R}_t(h_t; f_t) + \td_t(h_t) \E\big[V_{t+1}^\pi (H_{t+1}) \mid H_t=h_t\big]\Big) - (\rho-1) \textsf{m}   -V_t^{\pi^*_\casc}(h_t) \nonumber\\
     & = \rho \Big((1-\td_t(h_t))\mc{R}_t(h_t; f_t) + \td_t(h_t) \E\big[V_{t+1}^\pi (H_{t+1}) \mid H_t=h_t\big]\Big)\nonumber\\
     & \quad -(\rho-1)\Big(\Delta_{t, \internal}(h_t) +V_t^{\pi^*_\casc} (h_t)\Big) - V_t^{\pi^*_\casc} (h_t)\nonumber\\
     & = \rho \Big((1-\td_t(h_t))\mc{R}_t(h_t; f_t) + \td_t(h_t) \E\big[V_{t+1}^\pi (H_{t+1}) \mid H_t=h_t\big]\Big)  -(\rho-1) \Delta_{t, \internal}(h_t)\nonumber\\
     & \quad  -\rho V_t^{\pi^*_\casc} (h_t)  \nonumber\\
     & \stackrel{\text{(i)}}{\le} \rho \Big((1-\td_t(h_t))\mc{R}_t(h_t; f_t) + \td_t(h_t) \E\big[V_{t+1}^\pi (H_{t+1}) \mid H_t=h_t\big]\Big)   -\rho V_t^{\pi^*_\casc} (h_t) \nonumber\\
     & = \rho \Big((1-\td_t(h_t))\mc{R}_t(h_t; f_t) + \td_t(h_t) \E\big[V_{t+1}^\pi (H_{t+1}) \mid H_t=h_t\big] -V_t^{\pi^*_\casc} (h_t) \Big)\,, \label{eq:Delta_01c_decompose_new_decisive}
\end{align}
%, which does not depend on $h_t$.
where the inequality (i) follows from  $\rho - 1 \ge 0$ and $\Delta_{t, \internal}(h_t)\ge 0$. The RHS of \eqref{eq:Delta_01c_decompose_new_decisive} matches the RHS of \eqref{eq:Delta_01c_decompose_new} with  $\rho$ in place of 2.
  \textbf{Step~2.3} and~\textbf{Step~2.4} then go through
with $\rho$ in place of $2$. The  domination  argument in
\textbf{Step~2.4}   still holds:   since $\rho \ge 1$ implies $\rho^{T-t-1}\ge 1$  for $t\le T-1$, the constants $K-1, C_{t+1}=\rho^{T-t-1}(K-1)$ and $\frac{K-1}{K} $ are all dominated by $C_{t+1}$. This completes the backward induction in \textbf{Step 2} with   
stage-wise calibration constants   
$C_{t+1} = \rho^{T-t-1}(K-1)$ and $C_t=\rho C_{t+1}$.

The rest of the proof follows, yielding $C_1 = \rho^{T-1}(K-1)$ which proves \eqref{eq:tighter_UB}.
\end{proof}

\clearpage

\section{Synthetic Experiments } 
\subsection{Details for the Synthetic Experiments } 
\label{app:synthetic_details}
The class-probability functions $\boldsymbol{\theta}(x,z_t)$ are defined piecewise over the input space $\mc{X}$. The goal is to induce regions where deferring to later classifiers improves accuracy enough to offset the routing cost. The parameters of the posterior distribution $P_{Y|X,Z_2, Z_3} = \mathrm{Categorical}\big(\boldsymbol{\theta}(x,z_2,z_3)\big)$ are given by;
\begin{align}
\theta_y(x, z_2, z_3)
&= \frac{\exp\!\left(\ell_y(x, z_2, z_3)\right)}
{\sum_{k=1}^3 \exp\!\left(\ell_k(x, z_2, z_3)\right)}, \\[0.75em]
\ell_1(x, z_2, z_3)
&= 1.2 - 4|x|, \\
\ell_2(x, z_2, z_3)
&= -0.2 + 2.2x + 3 g(x) z_2, \\
\ell_3(x, z_2, z_3)
&= -0.2 - 2.2x + 3 g(x) z_3, \\
g(x)
&= \sigma\!\left(10(|x| - 0.15)\right), \\
\sigma(u)
&= \frac{1}{1 + e^{-u}} .
\end{align}
The posterior distributions $P_{Y|X}$, $P_{Y|X,Z_2}$ and $P_{Y|X,Z_3}$ needed to compute the Bayes-optimal risks are estimated using $5000$ Monte Carlo samples.

\paragraph{Model architecture.}

All models ($f_1(x), f_2(x,z_2), f_3(x,z_3), r(x), d_1(x), d_2(x,z_2)$) are multilayer perceptrons with ReLU activations, $3$ hidden layers of width $128$, and batch normalization. We observe that including batch normalization speeds up training but does not change performance, so its use is optional. Each classifier outputs $K-1$ scores, with the final class score determined by the constraint that the scores sum to zero, as required by the multiclass hinge formulation.
\paragraph{Training.}
We use $100,000$ training, $1000$ validation, and $2000$ test samples. Models are trained using the AdamW  optimizer with a learning rate of $10^{-3}$, a batch size of $512$, and a maximum of $50$ epochs, with early stopping based on validation performance with a patience of two epochs. 
\clearpage
\subsection{Empirical Risk Minimization in Synthetic Dataset for Various Costs}\label{app:empirical_risk_vs_bayes}

We report the empirical risk of the learned models alongside the optimal empirical risk throughout training in Fig.~\ref{fig:emp_risk}, for the synthetic experiment presented in Section~\ref{sec:synthetic}. We present both the cascade and the flat setting, for costs  $\kappa_1=0, \kappa_2 = \kappa_3 = \kappa = 0.03$ and $\kappa_1=0, \kappa_2=\kappa_3 = \kappa=0.06$. 
We observe that training with the surrogate losses in both cases attains the optimal risk on the test set. 

\begin{figure}[h!]
    \centering
    \includegraphics[width=0.24\textwidth]{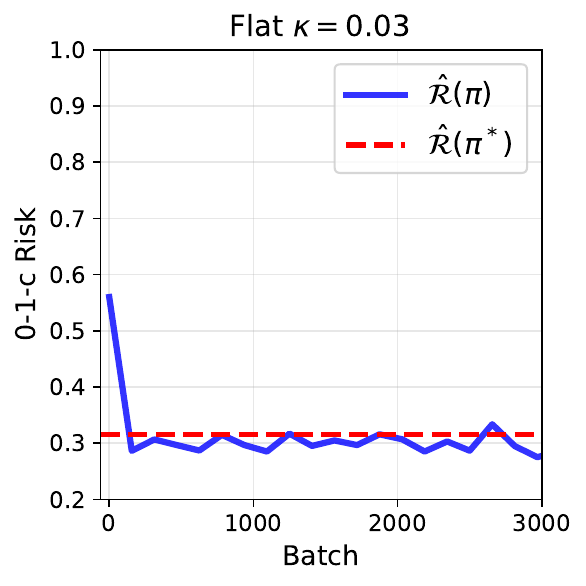}
    \includegraphics[width=0.24\textwidth]{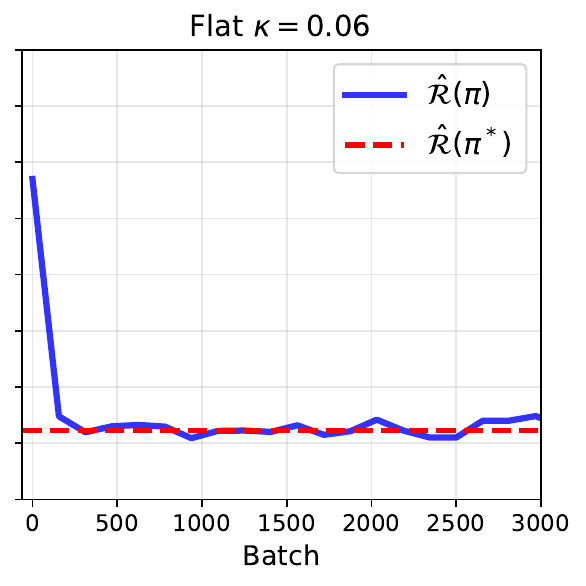}
    \includegraphics[width=0.24\textwidth]{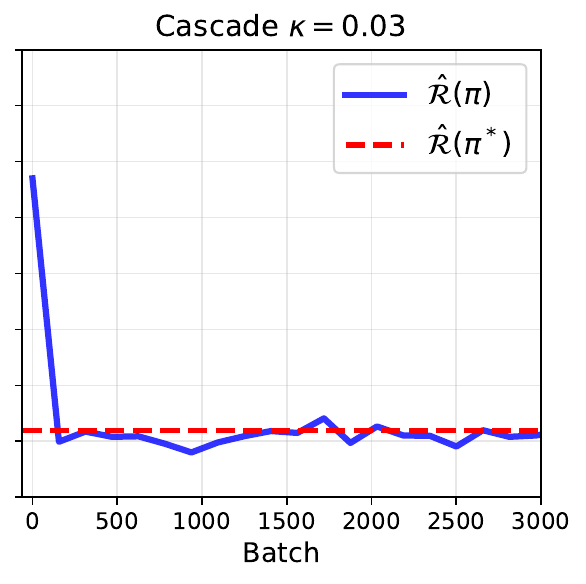}
    \includegraphics[width=0.24\textwidth]{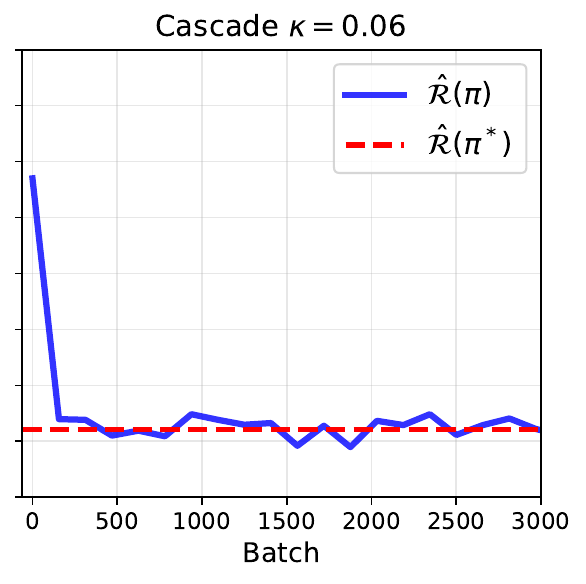}
    \caption{Empirical risk of the learned model trained with surrogate loss in the flat case (two left figures) and the cascade case (two right figures), with varying costs per classifier. In all cases, the empirical risk evaluated on the test set of the learned model converges to the empirical risk of the Bayes-optimal solution. 
    }
    \label{fig:emp_risk}
\end{figure}

\subsection{Accuracy of Classifiers and Router in Synthetic Dataset}\label{app:training_trajectories}
In this experiment, we examine how the accuracy of the router and classifiers evolves during training in the synthetic experiment with flat routers presented in Section~\ref{sec:synthetic}. 
As shown in Fig.~\ref{fig:route_vs_class}, the accuracy of the classifiers quickly rises and plateaus, while the routing accuracy takes longer to converge. 
\begin{figure}[h]
    \centering
    \includegraphics[width=0.6\linewidth]{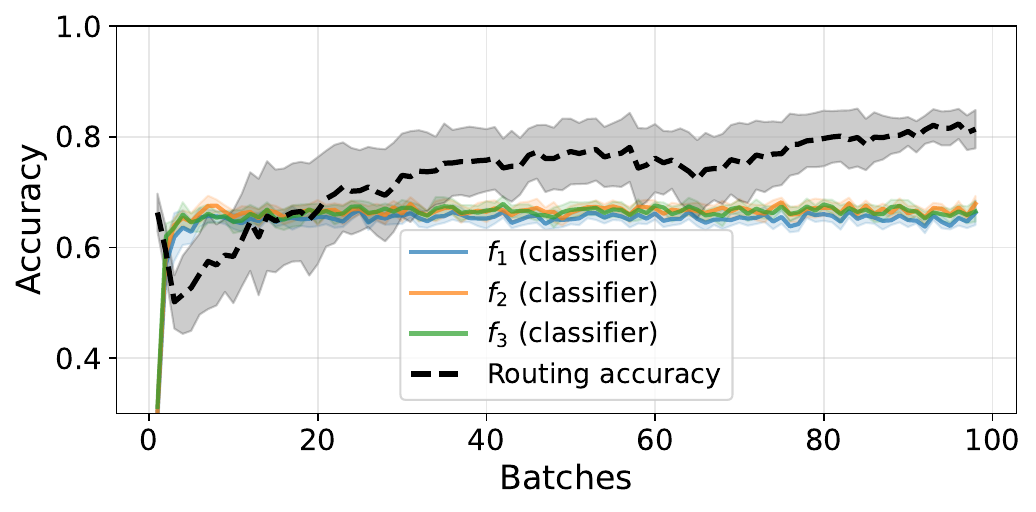}
    \caption{Classifier and routing-function accuracy during training in the flat router setting with $\kappa = 0.005$. The router is slower to learn compared to the classifiers.}
    \label{fig:route_vs_class}
\end{figure}

\clearpage
\section{Post-Training PPI Correction}
\subsection{Constructing Confidence Intervals Using Unlabeled Data with PPI}
\label{app:PPI_CI_construction}

In this section, we detail  the CI construction in  \eqref{eq:PPI_CI_result} and discuss the result in \eqref{eq:PP_equals_oracle}. For a fixed $x\in \mc{X}$, the CI for $\delta_t(x)$ in \eqref{eq:PPI_CI_result} is centered around  $\delta_t^\PP(x)$ with half-width $w_{t,N_b, l_b}(x)$, where $N_b$ and $l_b$ refer to the number of unlabeled and labeled samples w.r.t. the considered bin, respectively.  We construct $\delta_t^\PP(x)$ and $w_{t,N_b, l_b}(x)$ as follows. Let
\begin{align}
&\Delta_t(x, z_t,z_{r_\Flat(x)}, y)= 
\Ind\big\{f_{r_\Flat(x)}^\dagger(x, z_{r_\Flat(x)})\ne y\big\} 
%\nonumber\\
%&\quad \qquad \qquad \quad
-
\Ind\big\{f_t^\dagger(x, z_t)\ne y\big\}+c_{r_\Flat(x)}-c_t ,
\end{align}
then  we define 
\begin{align}
\delta_t^\PP(x)\coloneqq & \frac{1}{N_b}\sum_{i=1}^{N_b} 
\underbrace{\Delta_t\big(x,   z^{(i)}_t, z_{r_{\Flat}(x)}^{(i)}, f^\dagger_T(x, z_{T}^{(i)})\big)}_{\text{Term 1}} \nonumber\\
        & - \frac{1}{l_b}\sum_{j=1}^{l_b}  \underbrace{ 
    \Big[ 
        \Delta_t\big(x,z_t^{(j)}, z_{r_\Flat(x)}^{(j)}, f^\dagger_T(x,z_{T}^{(j)})\big)   -   \Delta_t\big(x,z^{(j)}_t,z^{(j)
        }_{r_\Flat(x)},y^{(j)}\big) \Big]
}_{\text{Term 2}},
\end{align}
where the  sum over Term 1 is evaluated using  unlabeled data, while the  sum over Term 2, called the \emph{rectifier} \citep{Angelopoulos2023PPI},  is evaluated using    labeled data. We use the most expensive classifier $f^\dagger_T$  to generate predictions on the unlabeled samples, since by design it is generally the most powerful classifier among $f_{1:T}^\dagger$.

The resulting CI has half-width
$w_{t,N_b,l_b}(x) \coloneqq 1.96\sqrt{\hat{\sigma}^2_{f}/N_b +  \hat{\sigma}^2_{f-Y}/l_b}$, 
where $\hat{\sigma}_f^2$ and $\hat{\sigma}^2_{f-Y}$ are the empirical variances of the summands of Term 1 and Term 2, respectively. When the imputing classifier $f^{\dagger}_T$ is accurate, the ${\hat{\sigma}_{f-Y}^2}/{l_b}$ term is small. The unlabeled-data term ${\hat{\sigma}_f^2}/{N_b}$ decreases  as the number of unlabeled samples  $N_b$ grows. In this case, taken together, the CIs are typically narrow and do not  straddle zero, yielding the guarantee in \eqref{eq:no_worse_than} which says $r^\PP_\Flat(x)$ yields a lower 0-1-$c$ risk than the baseline $r_\Flat(x)$ does.  When, in addition,   the CI for $\delta_{r^\dagger_\Flat(x)}(x)$ is sufficiently small to separate it from the next runner-up, we obtain the   guarantee in \eqref{eq:PP_equals_oracle} which says $r_\Flat^\PP(x)= r^\dagger_\Flat(x)$.

\subsection{Additional PPI Experiments: Impact of Labeled Data Budget on Performance}\label{sec:app_ppi}

We provide additional figures illustrating the behavior of the PPI algorithm for the setting presented in Section~\ref{sec:PPI_experiments} as the amount of training and PPI labeled data varies, ranging from $n,l=100$ to $n,l=3000$ samples (Fig.~\ref{fig:add_ppi} from (a) to (d)), while the amount of unlabeled data remains fixed.
As expected, with fewer labeled samples (100 and 200 shown in Fig.~\ref{fig:add_ppi} (a) and (b), respectively), neither the trained routing model nor the classifiers can reach Bayes-optimality. The trained routing model $ r_{\Flat} $ exhibits large errors w.r.t $r^{\dagger}_{\Flat}$, as indicated by the larger gray shaded regions. The error of the trained classifiers is reflected in the associated oracle router $r^{\dagger}_{\Flat}$, which is also far from the Bayes-optimal router $r^*_{\Flat}$, as indicated by the blue shaded regions ($20.6\%$ and $12.5\%$ disagreement). In this regime, although the PPI correction yields larger relative improvements (w.r.t. $r^{\dagger}_{\Flat}$), 
the overall routing remains farther from being Bayes-optimal due to the poor quality of the initial classifiers.

As the number of labeled samples increases (1000 and 3000 shown in Fig.~\ref{fig:add_ppi} (c) and (d), respectively), the oracle router more closely approximates the Bayes-optimal router, and the trained routing model also more closely approximates the oracle. At $n=1000$, the trained router still fails to capture the region around $ x = -0.60$, where the oracle routes samples to the first classifier. With the PPI correction applied, part of this region is now correctly routed to $ f_1$. At $n=3000$, there is sufficient data to learn both the classifiers and the router; the oracle router is very close to the Bayes-optimal router, the trained router is also very close to the oracle, and the PPI correction leaves the routing unchanged.

We also repeat the experiment but reuse the training data as labeled PPI data, instead of separating the labeled set into two independent partitions (Fig.~\ref{fig:add_ppi_reused}). In that setting, the theoretical guarantee of Theorem~\ref{thm:PPI} no longer holds as the classifiers are trained on the same data used for the correction, and are therefore not independent anymore. In practice, however, this violation has little impact, and performance is only slightly lower than with separate labeled sets (Fig.~\ref{fig:add_ppi}).

In summary, the PPI correction improves routing when the trained router is suboptimal, and its ability to recover the Bayes-optimal routing model depends on the quality of the underlying trained classifiers. The same trend holds when the training data is reused for the PPI correction.  When sufficient labeled data is available to train near-optimal classifiers, the correction more closely approaches the Bayes-optimal routing behavior. 

\begin{figure}[h!]
    \centering
    \begin{minipage}[t]{0.48\textwidth}
        \centering
        \includegraphics[width=\linewidth]{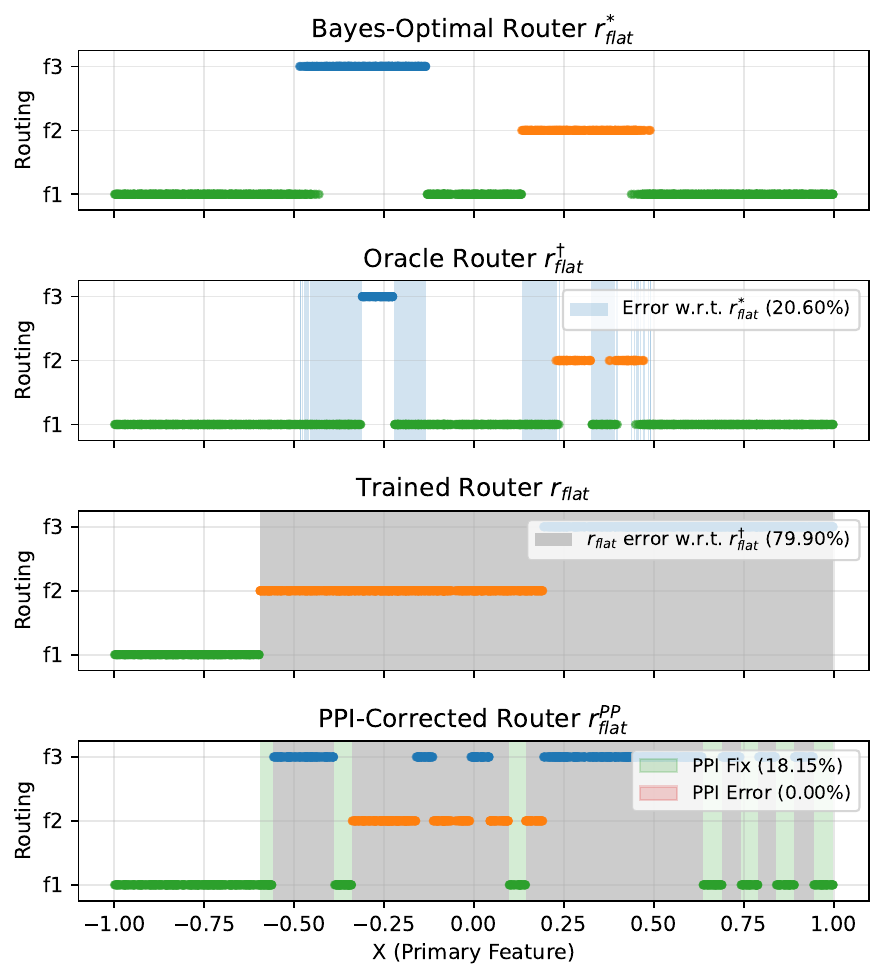}
        \vspace{0.3em}
        {\small (a) 100 training samples}
    \end{minipage}
    \hfill
    \begin{minipage}[t]{0.48\textwidth}
        \centering
        \includegraphics[width=\linewidth]{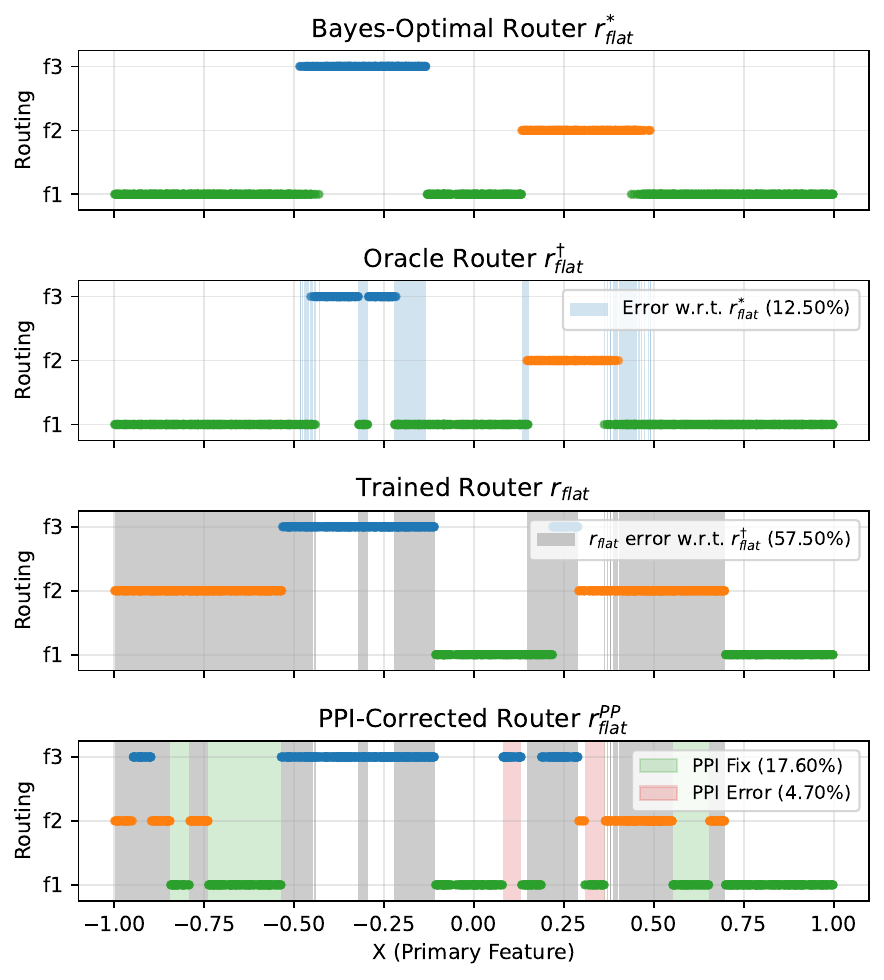}
        \vspace{0.3em}
        {\small (b) 200 training samples}
    \end{minipage}

    \vspace{0.8em}

    \begin{minipage}[t]{0.48\textwidth}
        \centering
        \includegraphics[width=\linewidth]{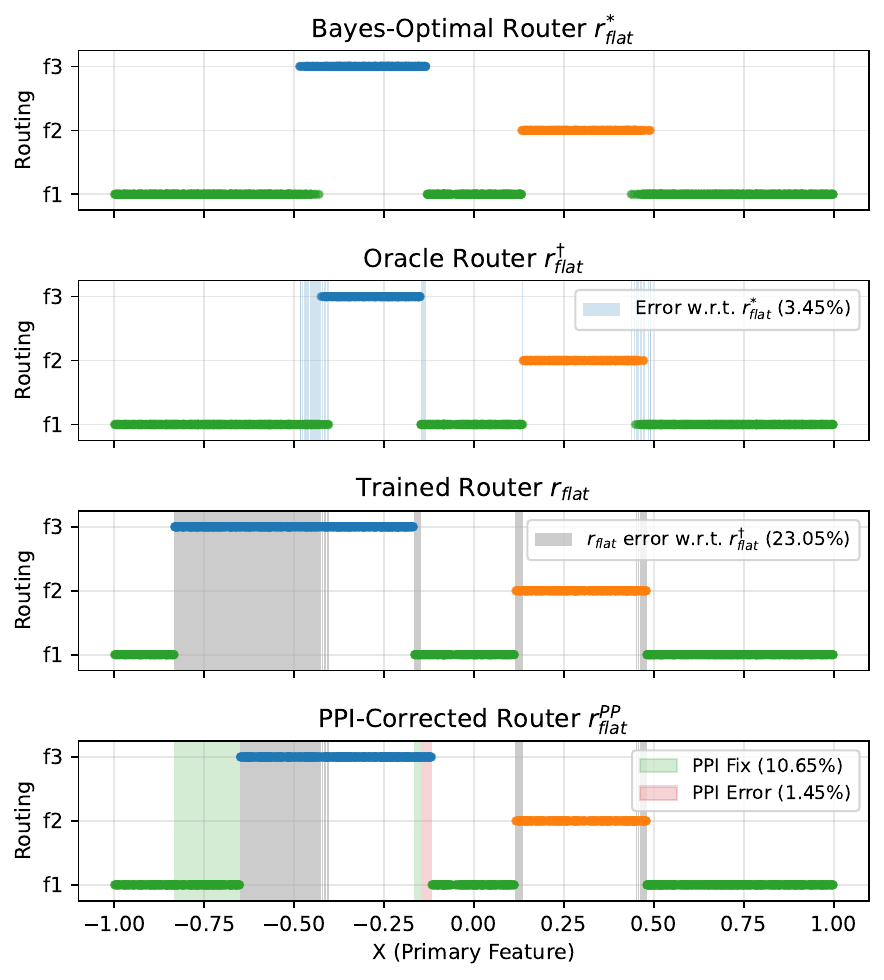}
        \vspace{0.3em}
        {\small (c) 1000 training samples}
    \end{minipage}
    \hfill
    \begin{minipage}[t]{0.48\textwidth}
        \centering
        \includegraphics[width=\linewidth]{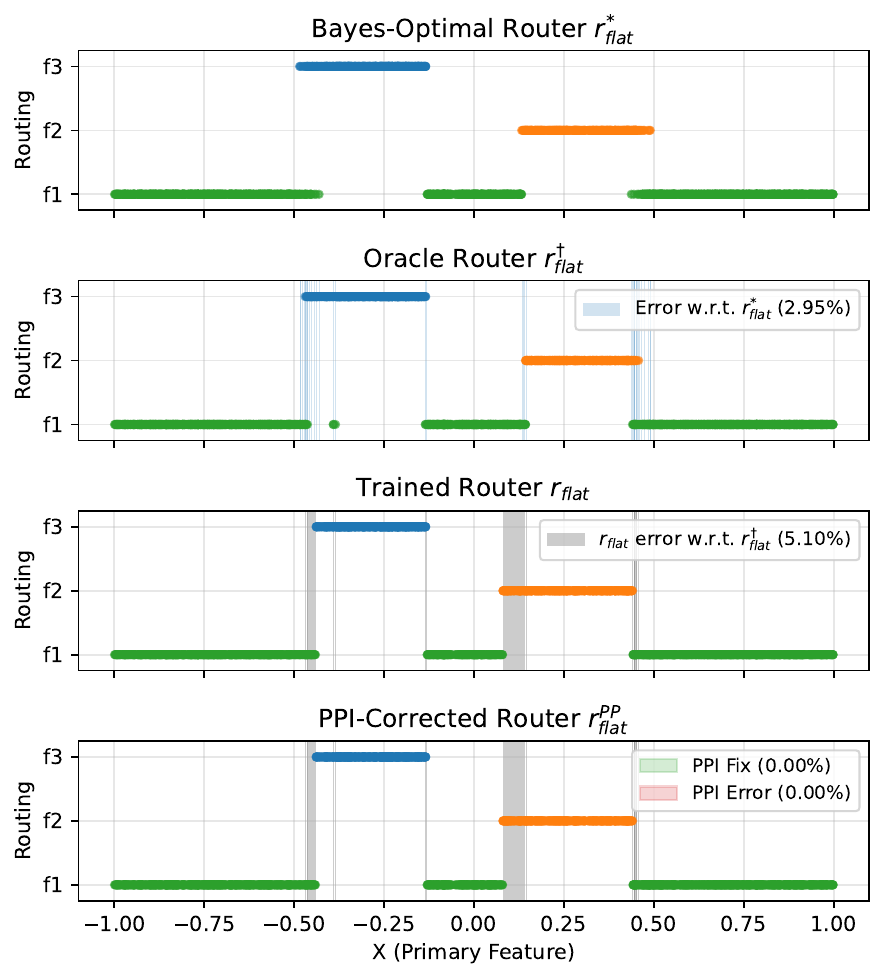}
        \vspace{0.3em}
        {\small (d) 3000 training samples}
    \end{minipage}
    \caption{Comparison of the Bayes-optimal routing function $r_\Flat^*$, the  router learned on labeled data $r_{\Flat}$, and the PPI-corrected router $r^{\PP}_{\Flat}$ , for increasing numbers $n=l$ of labeled training and PPI samples. The error of the oracle w.r.t. the Bayes-optimal $r^*_\Flat$ is shown in blue shaded regions, the error of the trained router w.r.t. the oracle router is shown in gray shaded regions, and the  fixes and errors introduced by PPI are computed w.r.t. the oracle router and shown by the green and red shaded regions. Note that these are computed pointwise for every input $x$, whereas the PPI correction is applied at the bin level, so the errors and fixes also reflect the effect of binning.}
    \label{fig:add_ppi}
\end{figure}

\begin{figure}[h!]
    \centering
    \begin{minipage}[t]{0.48\textwidth}
        \centering
        \includegraphics[width=\linewidth]{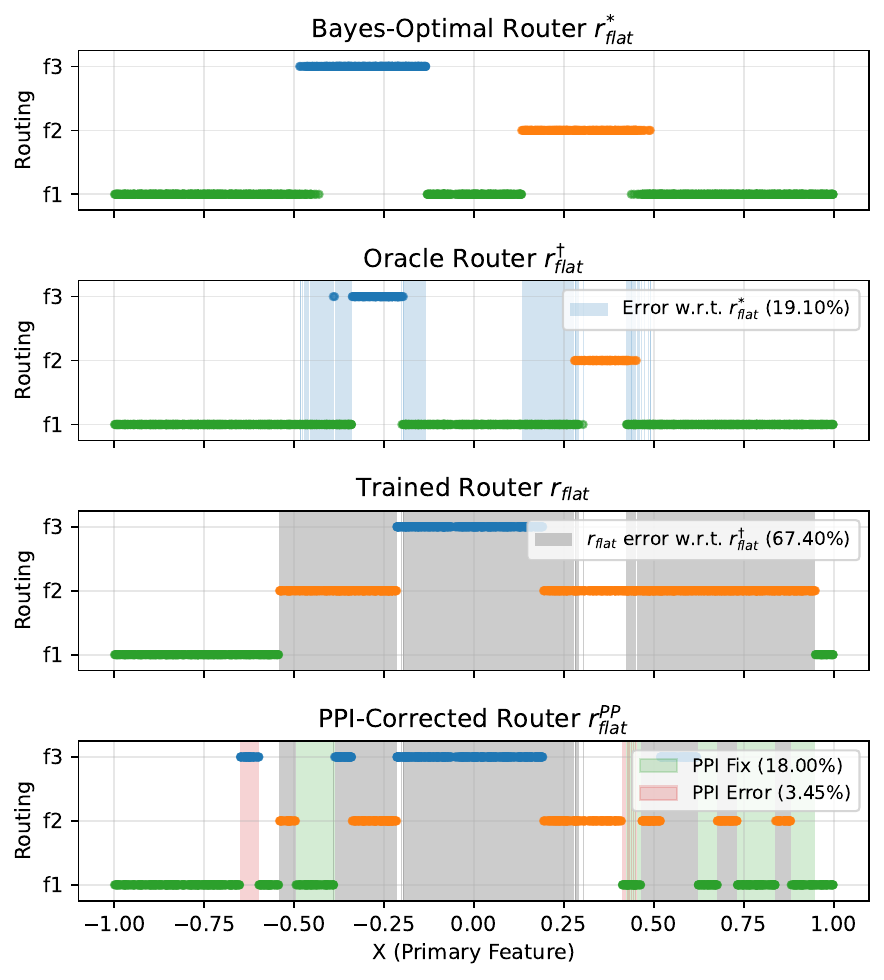}
        \vspace{0.3em}
        {\small (a) 100 training samples}
    \end{minipage}
    \hfill
    \begin{minipage}[t]{0.48\textwidth}
        \centering
        \includegraphics[width=\linewidth]{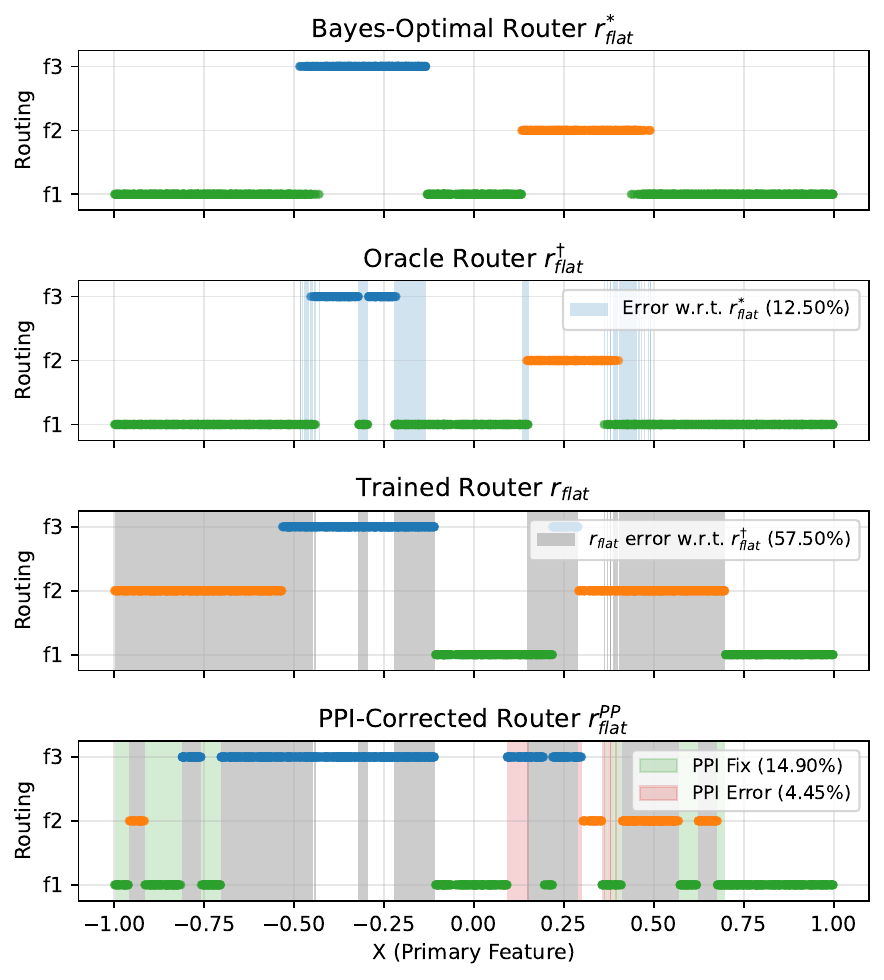}
        \vspace{0.3em}
        {\small (b) 200 training samples}
    \end{minipage}

    \vspace{0.8em}

    \begin{minipage}[t]{0.48\textwidth}
        \centering
        \includegraphics[width=\linewidth]{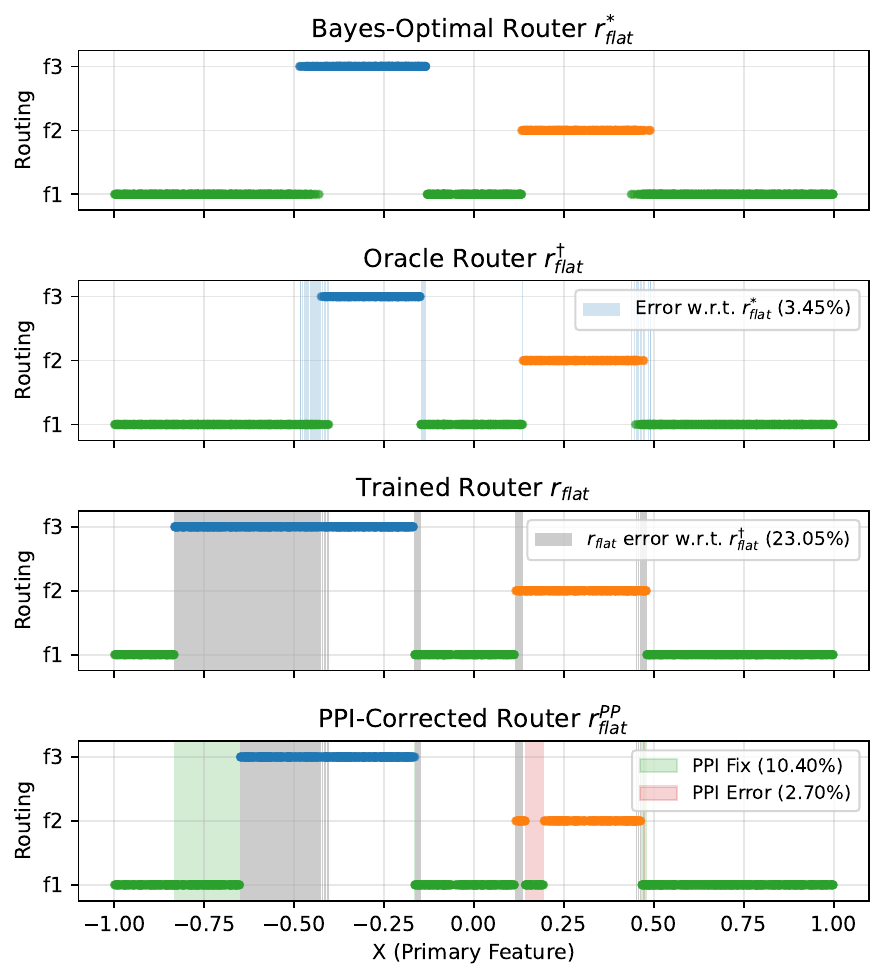}
        \vspace{0.3em}
        {\small (c) 1000 training samples}
    \end{minipage}
    \hfill
    \begin{minipage}[t]{0.48\textwidth}
        \centering
        \includegraphics[width=\linewidth]{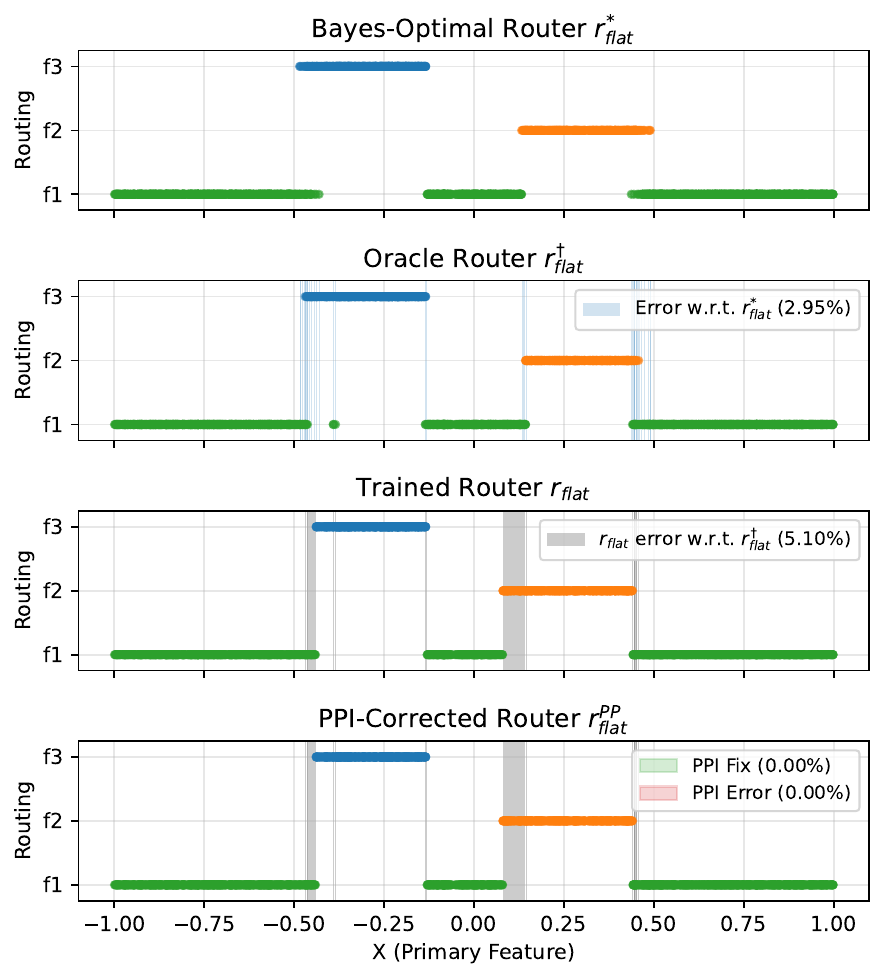}
        \vspace{0.3em}
        {\small (d) 3000 training samples}
    \end{minipage}

    \caption{Comparison of the Bayes-optimal routing function $r_\Flat^*$, the  router learned on labeled data $r_{\Flat}$, and the PPI-corrected router $r^{\PP}_{\Flat}$, for increasing numbers $n$ of labeled samples, when the same labeled data is used both for training and for the PPI correction.}
    \label{fig:add_ppi_reused}
\end{figure}

\clearpage
\section{Details for LLM Experiments}\label{app:training_llms}

\subsection{Datasets}\label{app:dataset}
We evaluate the utility of our proposed training losses on two multiple-choice question-answering benchmarks that span  commonsense reasoning and mathematical problem solving. 

\paragraph{CommonsenseQA.} CommonsenseQA~\citep{talmor2018commonsenseqa} is a dataset designed to test common sense knowledge and reasoning. Each question is a 5-way multiple choice problem ($K=5$), where the goal is to identify semantically related concepts. The dataset contains 12,102 questions covering diverse common sense relations such as causes, effects, and properties of everyday concepts. Questions require reasoning beyond surface-level pattern matching, as distractors are carefully designed to be semantically similar to the correct answer.

\paragraph{GSM8K.} GSM8K~\citep{cobbe2021gsm8k} is a dataset of 8,792 grade school math word problems designed to evaluate multi-step mathematical reasoning with final answers being numerical values.

Since GSM8K is originally a free-form numerical answer dataset, we convert it to a 4-way multiple choice format ($K=4$) to match our classification framework. For each question, we generate distractors using the following procedure:
\begin{enumerate}
    \item \textbf{Distractor generation from weak model:} We sample 20 chain-of-thought responses from a weak model (Qwen-1.5B) for each question. We collect all incorrect numerical answers produced by the model and normalize them (handling commas, decimal formatting, currency symbols).
    \item \textbf{Frequency-based distractor selection:} We select the 3 most frequently occurring wrong answers as distractors, ensuring each distractor differs from the correct answer and other distractors to avoid duplicates.
    \item \textbf{Synthetic fallback:} If the weak model does not produce enough distinct wrong answers, we supplement the answer pool with synthetic distractors based on common error patterns (e.g., off-by-one errors, order-of-magnitude mistakes, and  arithmetic sign errors).
    \item \textbf{Random ordering:} The correct answer position is randomized across the four labels A to D with a deterministic seed per question to ensure reproducibility.
\end{enumerate}
This conversion procedure yields realistic distractors that reflect actual model errors rather than arbitrary numerical alternatives, making the resulting multiple-choice problems more representative of model failure modes.

\paragraph{Data splitting.} For both datasets, we partition the data into training, validation, and test sets using an 80-10-10 split with a fixed random seed for reproducibility. The validation set is used for hyperparameter selection and early stopping, while the test set is held out for final evaluation.

\subsection{Model and Training Details}\label{app:model_training_llm}

\paragraph{Setup and feature extraction.}
We mine LLM input-output pairs from natural language prompts and their natural language responses from a community repository~\citep{rumiantsev2025fevalttc} covering our multiple choice datasets (this repository also contains ground truth answers).
To convert the natural language prompts and responses to fixed size vectors we use a sentence embedding language model.\footnote{We use the \texttt{all-mpnet-base-v2} model from Sentence-Transformers from Hugging Face.} 
Specifically, each sample is represented by a prompt embedding $x \in \mathbb{R}^{768}$, obtained from the sentence embedding of the prompt text  and answer embeddings $z_t \in \mathbb{R}^{768}$ for $t \in [T]$.  Note that this means that the LLMs as well as the sentence embedding model are used to generate features $(x, z_t)$ for our classifiers $f_t$ which we describe in the next paragraph (the LLMs are not the classifiers themselves). For $t=1$, $z_1$ is the embedding of the answer produced by the first LLM, so $z_1$ carries information in these experiments. The prompt embedding $x$ is shared by all classifiers and each $z_t$ is specific to the $t$-th LLM, so $(x,z_t)$ for $t>1$ does not contain $z_1$. The flat router observes only $x$.

\paragraph{Architecture of classifier $\tf_t$.}
We describe the architecture of the classifiers used in the LLM experiments in Eq.~\eqref{eq:ft_network}. The classifier is a feedforward neural network that takes as inputs $x, z_t$ and outputs $K$ scores that sum to zero:
\begin{align}
 \tf_t(x,z_t)
&=\Theta_3^{(t)}\mathrm{ReLU}\left(\mathrm{BN}\left(\Theta_2^{(t)}\mathrm{ReLU}\left(
\mathrm{BN}\left(
\Theta_1^{(t)}[x;z_t]\right)\right)\right)\right).\label{eq:ft_network}
\end{align}
Here, $\mathrm{BN}$ denotes batch normalization, 
and $[\cdot;\cdot]$ denotes concatenation. The superscript $(t)$ denotes that parameter matrices $\Theta^{(t)}_k$ belong to the classifier network $\tf_t$. The classifier parameters have dimensions  $\Theta_1^{(t)} \in \mathbb{R}^{128 \times 1536}$, $\Theta_2^{(t)} \in \mathbb{R}^{128 \times 128}$, and $\Theta_3^{(t)} \in \mathbb{R}^{(K-1) \times 128}$.
The classifier outputs $K-1$ scores, with the final class score determined by the constraint that the scores sum to zero.
During training, batch normalization normalizes each hidden layer with the mean and variance of the current mini-batch, so the output of a network for one sample also depends on the other samples in its mini-batch (during training only, not inference). While this slightly deviates from $\wt\Pi_\Flat$ and $\wt\Pi_\casc$ definition as additional components such as normalization or regularization were not included in the optimization that our theory considers, it is important to verify that our proposed approach works in a standard neural network setting with the usual components such as batch normalization.

\paragraph{Architecture of flat router.}
The soft router $\tilde r(x)$ outputs a probability distribution over the classifiers. It is implemented as a feedforward neural network:
\begin{align}
\tilde r(x)
&=\mathrm{Softmax}\left(\mathrm{BN}\left(\Theta_2^{(r)}\mathrm{ReLU}\left(\mathrm{BN}\left(\Theta_1^{(r)} x\right)\right)\right)\right).\label{eq:router}
\end{align}
The superscript $(r)$ in $\Theta^{(r)}_k$ denotes that we refer to the parameters of the router. The router parameters have dimensions $\Theta_1^{(r)} \in \mathbb{R}^{128 \times 768}$ and $\Theta_2^{(r)} \in \mathbb{R}^{T \times 128}$. 

\paragraph{Architecture of cascade router.} 
The cascade consists of similar stage-specific classifiers but with deferral routing modules implemented as feedforward neural networks. 

The soft deferral rule implemented in the deferral routing modules $\tilde d_t(x,z_{1:t})$ takes the full history available up to stage $t$, namely the question embedding and all answer embeddings observed so far, and outputs a deferral probability. A deferral module is defined as
\begin{align}
\tilde d_t(x,z_{1:t})
&=\sigma\left(\mathrm{BN}\left(\Phi_2^{(t)}\mathrm{ReLU}\left(\mathrm{BN}\left(\Phi_1^{(t)}[x;z_{1:t}]\right)\right)\right)\right),\label{eq:cascade_router}
\end{align}
where $\sigma$ denotes the sigmoid function. 
The deferral module at stage $t$ has input dimension $768+768t=768(1+t)$, so it is parameterized by $\Phi_1^{(t)} \in \mathbb{R}^{128 \times 768(1+t)}$ and $\Phi_2^{(t)} \in \mathbb{R}^{1 \times 128}$. There are only $T-1$ deferral modules, since the final stage always exits if reached.

\paragraph{Batch Normalization.} Batch normalization is optional in our setup. As in the synthetic experiments, we observe that it only speeds up training and does not change performance significantly.

\paragraph{Training.}
Both the flat and the cascade architectures are trained using the AdamW optimizer with learning rate $10^{-3}$ and mini-batch size $32$. 
Training runs for a maximum of 1000 epochs for the cascade model and 100 epochs for the router model, with early stopping using a patience of 20 epochs based on validation loss. We employ the proposed hinge-based surrogate loss that balances classification accuracy with routing costs. Training takes approximately 3 hours per dataset with the majority of the time spent on generating embeddings from model answers using an Nvidia GeForce RTX 3060 12GB card.

\paragraph{Evaluation protocol.}
We report performance using 10 bootstrap trials on the held-out test set. For each trial, we resample the test set with replacement and compute the accuracy of each classifier on all test samples and on the samples routed to it, which for the cascade are the samples that exit at it. We also compute the fraction of samples routed to each classifier, or deferred past each stage of the cascade. This bootstrap procedure provides confidence intervals that account for finite sample variability.

\paragraph{Algorithm implementation.} We implement our method in PyTorch and make our code available:\\
\url{https://github.com/networkslab/Consistent-Classifier-Routing}.

\clearpage

\section{Additional LLM Experiments}\label{app:additional_experiments}

\subsection{LLM Experiments with Additional Models and Datasets}
\label{app:add_llm_cascade}
In this appendix we extend the experiments of Sec.~\ref{sec:LLM_experiments} in two directions. We first train a cascade router with the LLaMA classifiers on GSM8K, and we then repeat the flat-router experiment of Sec.~\ref{sec:LLM_experiments} with Qwen classifiers on CommonsenseQA, which changes both the model family and the dataset.

\paragraph{Cascade router with LLaMA classifiers.} We use the three cost configurations of Sec.~\ref{sec:LLM_experiments}, so the per-classifier costs range from $[0, 0.01, 0.02]$ in Configuration 1 to $[0, 0.1, 0.2]$ in Configuration 3. A prompt pays for every classifier it passes through, so an answer from LLaMA 405B costs $\kappa_1+\kappa_2+\kappa_3$, from 0.03 in Configuration 1 to 0.3 in Configuration 3. Fig.~\ref{fig:llama_cascade} shows how the trained cascade spends this budget. In its bottom row, the fraction of prompts deferred past LLaMA 1B falls from about 66\% in Configuration 1 to about 36\% in Configuration 3, so the share answered by LLaMA 1B grows from about one third to about two thirds. The fraction deferred past LLaMA 70B, which is the share answered by LLaMA 405B, falls from about 60\% to about 10\%. The top row shows that LLaMA 70B and LLaMA 405B answer at least 90\% of the prompts they receive correctly in every configuration, close to their accuracy on all prompts. For LLaMA 1B, the median accuracy on the prompts it answers stays within 4 percentage points of its accuracy on all prompts, so the first deferral module does not single out prompts that LLaMA 1B answers better than average. In contrast, the flat router of Fig.~\ref{fig:cost_comp} sends only about 9\% of the prompts to LLaMA 1B in Configuration 1, and LLaMA 1B answers about 66\% of them correctly against about 49\% over all prompts.

\begin{figure}[h!]
    \centering
    \includegraphics[width=0.6\linewidth]{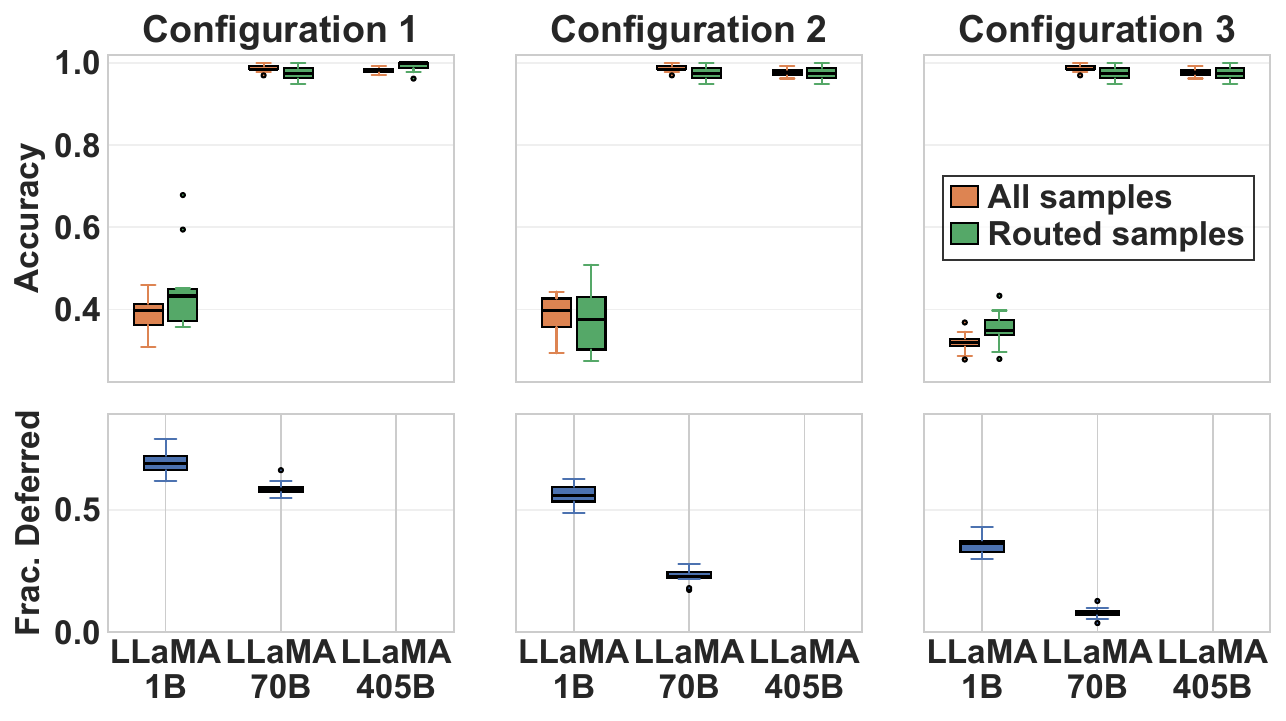}
    \caption{
    Classifier accuracy and deferral fractions of the cascade router with LLaMA classifiers on GSM8K across three cost configurations (costs increasing from left panel to right panel). The top row shows the accuracy of each classifier on all test prompts and on the prompts it answers, and the bottom row shows the fraction of prompts deferred at each classifier. 
    }
    \label{fig:llama_cascade}
\end{figure}

\newpage

\paragraph{Flat router with Qwen classifiers.} We repeat the flat-router experiment of Sec.~\ref{sec:LLM_experiments} with Qwen classifiers on CommonsenseQA and the same three cost configurations, which changes both the model family and the dataset. Fig.~\ref{fig:qwen_router} shows the results. In the bottom row, the share of prompts routed to Qwen 1.5B stays near 30\% in Configurations 1 and 2 and rises to 74\% in Configuration 3, while the share routed to Qwen 32B falls from 38\% to 9\%. In the top row, Qwen 32B is the most accurate classifier on all prompts at 91\%, while Qwen 72B reaches 83\%. In Configuration 1, the prompts routed to Qwen 1.5B are answered correctly 83\% of the time against 76\% over all prompts, so the router keeps easier prompts for the cheapest classifier, as the flat router of Fig.~\ref{fig:cost_comp} does for LLaMA 1B. In Configuration 3, Qwen 1.5B answers most prompts, and its accuracy on them falls to about its accuracy on all prompts. These results suggest that the behavior of our training loss is not specific to one model family.

\begin{figure}[h!]
    \centering
    \includegraphics[width=0.6\linewidth]{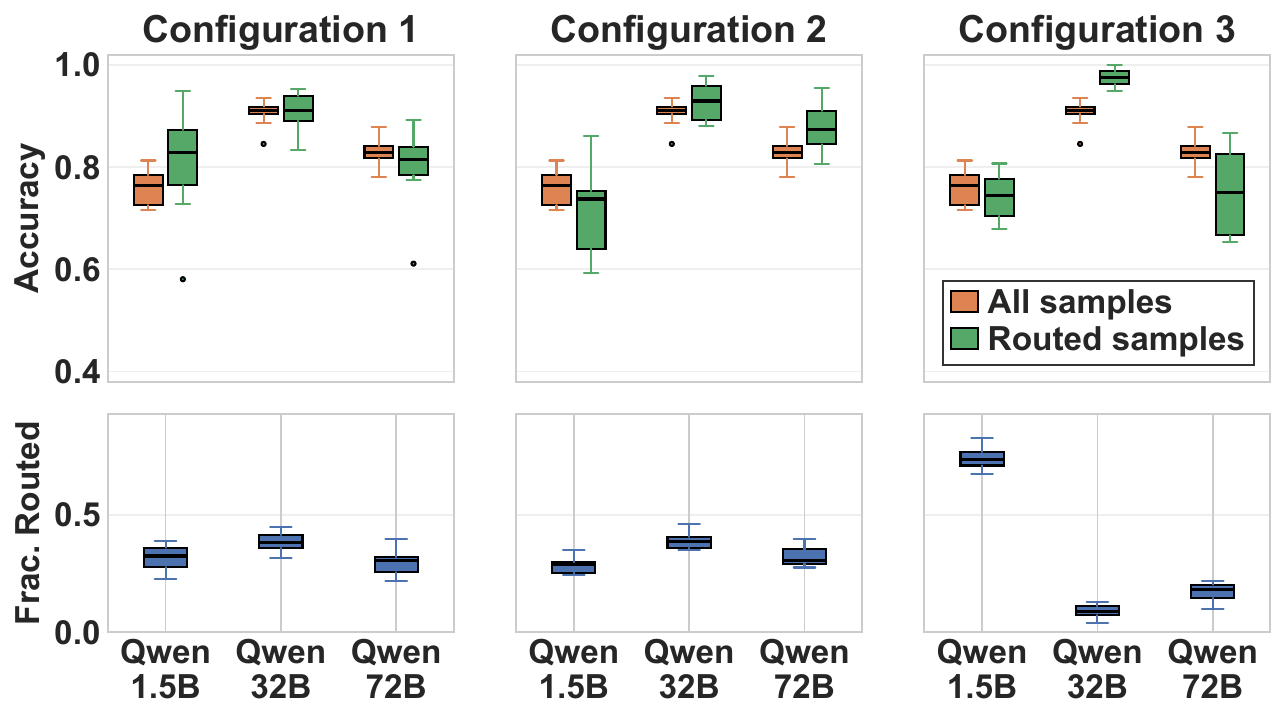}
    \caption{Classifier accuracy and routing fractions of the flat router with Qwen classifiers on CommonsenseQA across three cost configurations (costs increasing from left panel to right panel). The top row shows the accuracy of each classifier on all test prompts and on the prompts routed to it, and the bottom row shows the fraction of prompts routed to each classifier. 
    }
    \label{fig:qwen_router}
\end{figure}

\clearpage

\subsection{Comparison to Supervised LLM Cascade Baselines}\label{app:casc_baselines}

Our work is primarily theoretical, but in LLM multiple-choice question answering it can be compared directly with applied cascade methods. In the cascade router setting, we compare it with FrugalGPT~\citep{chen2024frugalgpt} and TREACLE~\citep{zhang2024treacle} in one configuration, with Qwen classifiers on CommonsenseQA and per-classifier costs $\kappa=[0, 0.05, 0.06]$ in order of model size. We choose these two methods because both learn when to escalate along a cascade of fixed LLMs under a cost budget, which makes them the separate-training counterpart of our setting. FrugalGPT uses a cascade of LLM APIs within a budget constraint~\citep{chen2024frugalgpt}, and TREACLE casts the choice of LLM and prompt as budget-constrained policy optimization~\citep{zhang2024treacle}. Other LLM cascade routing methods do not fit this setup at all for structural reasons. C3PO~\citep{valkanas2025cpo} does not admit direct comparison as it is an unsupervised technique, whereas the methods we compare to all require direct supervision. 
RouteLLM~\citep{ong2025routellm} learns from human preference data to route between exactly two LLMs, so it is a two-model router rather than a cascade of classifiers. RouterDC~\citep{chen2024routerdc} trains a query-only flat router to select one LLM per query (so it does not support a cascade router setup). It keeps the LLMs fixed and does not train classifiers on their outputs or optimize a $0$-$1$-$c$ objective, so it addresses a different problem. 
Our method and Joint CE use the Qwen cascade router, which is replicated for FrugalGPT and TREACLE cascades that query the Qwen models from the cheapest to the most expensive using the same costs.
Both FrugalGPT and TREACLE keep the classifiers fixed and train only the routing rule, while our method trains the classifiers and the router jointly. Note that the asymmetry in optimization objectives is deliberate, as this indicates the value added by joint training versus frozen classifiers. Joint classifier - router training methods are expected to perform better and our results confirm this.

Table~\ref{tab:cost_decomposition} reports three quantities on the test set for each method. The error rate is the fraction of test prompts answered incorrectly, and the average cost is the mean over the test set of the cost term $c_{r(h)}$ in~\eqref{eq:multi_01c}. The test empirical risk is the sum of the error rate and average cost, which is the empirical $0$-$1$-$c$ risk of Section~\ref{sec:loss_func}, the risk our guarantees concern. Joint Cross Entropy (CE) trains the classifiers and the router jointly with the ordinary cross-entropy loss, as described in Appendix~\ref{app:adap_baselines}. All methods use the same train-test split. Our method attains the lowest test empirical risk.

\begin{table}[h]
\centering
\caption{Error rate, average cost and test empirical risk of each method with Qwen classifiers on the test set of CommonsenseQA. FrugalGPT and TREACLE are run with a target cost budget of 0.03, which acts as a soft constraint, so their average cost can exceed it slightly. Our method attains the lowest test empirical risk. Bold indicates statistically significant difference between the two best models via pairwise Wilcoxon test with 10 trials (our method achieves statistical significance at $p<0.05$ level).
}
\vspace{1em}
\begin{tabular}{lccc}
\toprule
\textbf{Method} & \textbf{Error Rate} & \textbf{Avg Cost} & \textbf{Test Empirical Risk} \\
\midrule
Hinge (ours)          & 0.1626 & 0.0344 & \textbf{0.1970} \\
Joint CE              & 0.1870 & 0.0334 & 0.2204 \\
% Regol et al. 2025     & 0.2276 & 0.0317 & 0.2593 \\
FrugalGPT~\citep{chen2024frugalgpt}             & 0.2195 & 0.0351 & 0.2546 \\
TREACLE~\citep{zhang2024treacle}  &    0.1951   &   0.0328   &   0.2279\\
\bottomrule
\end{tabular}

\label{tab:cost_decomposition}
\end{table}
The results show that joint model-router training under our proposed hinge-based loss provides meaningful improvements in empirical risk, while incurring similar cost to the other methods. 

\newpage

\subsection{Comparison to Adaptive Inference Baselines}\label{app:adap_baselines}

Besides the empirical works which we compared to in Appendix~\ref{app:casc_baselines}, one might wonder how some naive baselines (e.g. joint training with cross entropy) or theoretical techniques from the broader literature would perform in the cascade router setting. In this appendix, we set up a simple experiment and show that our proposed loss leads to better convergence and lower overall~empirical~risk.

We compare our proposed loss against adaptive-inference baselines on CommonsenseQA using the Qwen family. Similarly to the other experiments, the LLM weights are frozen and we learn a classifier based on prompt embeddings $ x_i$ and answer embeddings $ z_i$. The learned classifier associated with stage $t$ is $f_t(x,z_t)$. For baselines that use probabilistic predictions, we write
\begin{align}
\hat p_t(k \mid x,z_t),
\qquad k\in[K],
\end{align}
for the class probabilities produced by the corresponding cross-entropy-trained classifier at stage $t$. The baselines are grouped into two categories: separate-training and joint-training baselines.

\paragraph{Separate-training baselines.}
In these baselines, the classifiers $ f_1,f_2,f_3$ are trained independently from the routing or deferral rule.
\begin{itemize}[leftmargin=*]
    \item \textbf{Confidence Threshold (CT):} For each stage $t$, a threshold $\tau_t$ is tuned on a validation set to minimize empirical $0$-$1$-$c$ risk. The method exits at stage $t$ when
    \begin{align}
    \max_{k\in[K]}\hat p_t(k \mid x,z_t) \ge \tau_t,
    \end{align}
    and otherwise defers to the next stage.

    \item \textbf{Confidence Threshold with cost adjustment (CT-c):} This method uses the same thresholding rule as CT, but sets a cost-adjusted threshold using the estimated validation accuracy of each classifier and its cost. We follow the construction of Regol et al. (2025)~\citep{regol2025acquisition}.

    \item \textbf{Soft Deferral:} A probabilistic routing baseline that defers at stage $t$ with probability
    \begin{align}
    1-\max_{k\in[K]}\hat p_t(k \mid x,z_t),
    \end{align}
    so lower-confidence predictions are more likely to be deferred.

    \item \textbf{~\citep{Mozannar2020}:} The standard learning to defer baseline, which learns a classifier and a deferral rule that decides whether to predict or defer to a fixed expert. Its classifier and deferral rule are trained together, while the expert, here the classifier of the most expensive LLM, is trained separately. 
\end{itemize}

\paragraph{Joint-training baselines.}
In these baselines, classifiers and routing decisions are trained jointly.
\begin{itemize}[leftmargin=*]
    \item \textbf{Joint cross-entropy training (Joint CE):} Joint training of the classifiers and the deferral modules with the loss $\sum_{t=1}^T \tr_t(x,z_{1:t})\big(\mathrm{CE}(\tf_t(x,z_t),y)+\sum_{t'\leq t}\kappa_t'\big)$, where $\tr_t(x,z_{1:t})$ is the probability of exiting at stage $t$ and $\sum_{t'\leq t}\kappa_t'$ is the cumulative cost of reaching stage $t$, as in~\eqref{eq:surrogate_hinge_casc}. Here $\mathrm{CE}$ is the cross-entropy of the softmax of the scores $\tf_t(x,z_t)$. This loss replaces each hinge term of~\eqref{eq:surrogate_hinge_casc} by the cross-entropy and keeps the cost term, so the deferral modules are penalized for deferring to expensive classifiers.
    \item \textbf{~\citep{regol2025acquisition}:} Joint training using the two-classifier hinge surrogate proposed by Regol et al. (2025), which corresponds to the two-model version of our proposed hinge loss.
\end{itemize}

Since ~\citet{Mozannar2020} and ~\citet{regol2025acquisition} are formulated for two classifiers, we instantiate them using the cheapest and most expensive LLMs. This difference in the number of classifiers reflects the setting our contribution targets as neither of these two methods extends to more than two classifiers, which is the gap our proposed loss fills. Running our method with $T=2$ would not give a more informative comparison, since for two classifiers the loss of Regol et al. (2025) is a special case of our hinge loss and the two would produce identical results. A difference can therefore only appear for $T>2$, and Appendix~\ref{app:ablation} shows that our loss continues to reduce the risk as $T$ grows to $5$. All other baselines use three LLMs, same as ours. The router is a cascade router, so a sample that exits at stage $t$ pays the cumulative cost $\kappa_1+\dots+\kappa_t$. The per-classifier costs are $\kappa=[0, 0.05, 0.06]$ in order of model size, and every method is evaluated by its test $0$-$1$-$c$ risk under these costs. Fig.~\ref{fig:risk} shows that training with our proposed loss achieves the lowest test-set risk among all baselines. A pairwise Wilcoxon test with 10 trials of our method with the best baseline (at convergence) shows that the difference is statistically significant at the $p<0.05$ level. Additionally, the joint training methods (second column of legend of Fig.~\ref{fig:risk}) outperform the separate-training methods, showing the value of training the classifiers together with the router.

\begin{figure}[h]
    \centering
    \includegraphics[width=0.6\linewidth]{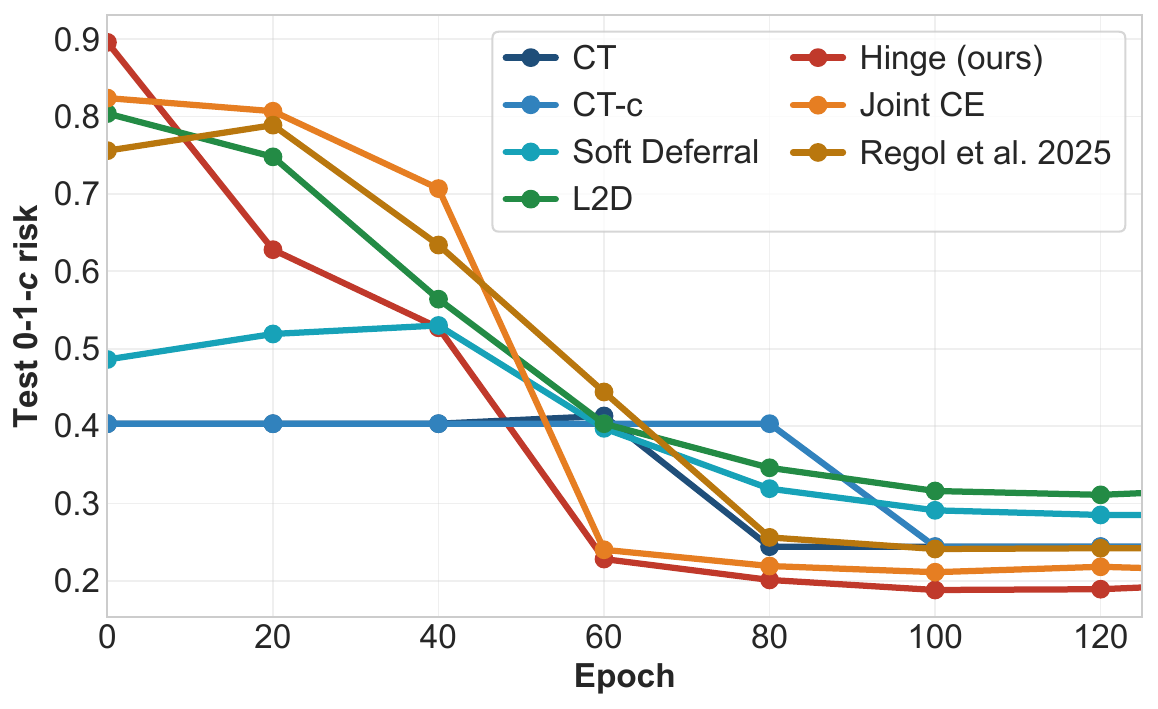}
    \caption{Test $0$-$1$-$c$ risk during training. Our method is compared with the separate-training baselines (CT, CT-c, Soft Deferral and L2D from Mozannar and Sontag 2020) and the joint-training baselines (Joint CE and Regol et al. 2025), and it reaches the lowest risk at the end of training.}
    \label{fig:risk}
\end{figure}

\clearpage

\subsection{Scaling with the Number of Classifiers}\label{app:ablation}

In our experiments, the number of classifiers has so far been fixed at $T=3$. To evaluate whether our approach can scale beyond this setting, we consider up to  $T=5$ classifiers in the CommonsenseQA experiment (see Fig.~\ref{fig:ablation}). The experiment starts from a single-model setting ($T=1$), where no routing rule is required, and progressively increases the number of models to five. We use a cascade router, and the per-classifier costs increase with model size, with $\kappa$ equal to 0 for Qwen-1.5B, 0.02 for LLaMA-8B, 0.05 for Qwen-32B, 0.06 for Qwen-72B and 0.15 for LLaMA-405B. 
The results show a consistent reduction in risk as additional classifiers are added, with diminishing returns beyond three classifiers. This aligns with the intuition that adding more classifiers provides additional operating points along the cost-accuracy trade-off curve.

\begin{figure}[h!]
    \centering
    \includegraphics[width=0.6\linewidth]{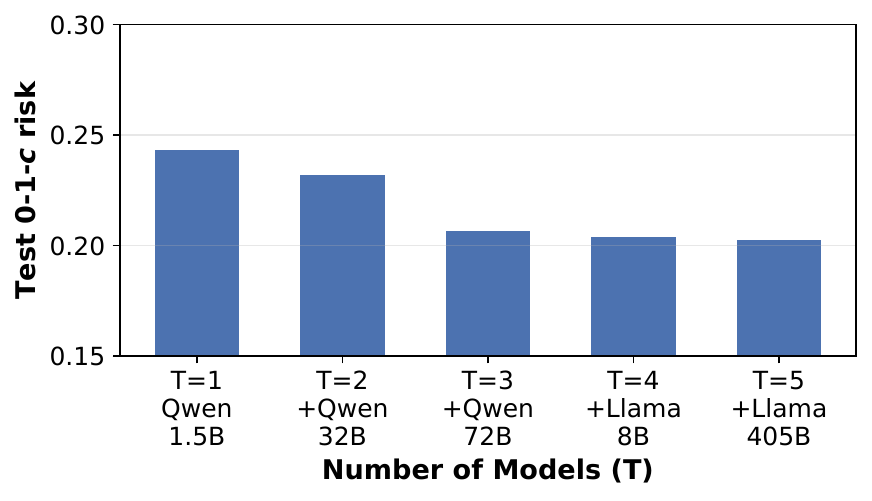}
    \caption{Test $0$-$1$-$c$ empirical risk on CommonsenseQA as a function of the number of  classifiers $T$, jointly-trained using our surrogate hinge loss. Starting from a naive single-model setting with the Qwen-1.5B model ($T=1$), where no routing rule needs to be learned,  we incrementally add more models: Qwen-32B ($T=2$), Qwen-72B ($T=3$), LLaMA-8B ($T=4$), and LLaMA-405B ($T=5$) alongside a trainable routing rule. 
    Risk decreases monotonically from $T=1$ to $T=5$ (with diminishing returns beyond $T=3$), indicating that the proposed framework scales well with the number of models and consistently benefits from additional routing options.
    }
    \label{fig:ablation}
\end{figure}
\clearpage
\newpage
\section{Additional Feature Variables to Represent Varying Model Capacity}\label{app:discussion_Z}
In the main text, we note that the additional feature variables $Z_t$ can represent hidden representations extracted by different architectures with varying modeling capacity. In general, hidden representations produced by larger models are more computationally expensive to obtain, but they are also often more informative. We therefore view these representations as  extra ``predictive signal'' available to the classifier $f_t$ and to the subsequent cascade routing rules. 

This framing can be related to the recently introduced concepts of computationally constrained information~\citep{xu2020theory} and epiplexity~\citep{finzi2026entropyEpiplexity}. These formulations provide a statistical abstraction of computational constraints, in which increased model capacity enables access to additional structured information in the data. 

The standard approach to model classifiers of varying performance is to consider  different hypothesis classes $\mathcal{F}_t, \mathcal{F}_{t+1}$, where a smaller model is constrained to $\mathcal{F}_t$ and a larger model to $\mathcal{F}_{t+1}$. This formulation typically relies on the implicit assumption that the optimal function in $\mathcal{F}_{t}$  is worse than that in  $\mathcal{F}_{t+1}$, w.r.t. some expected loss $\mathcal{L}$: 

\begin{align}
    & \mathcal{L}(f^*_t) \geq  \mathcal{L}(f^*_{t+1}) ,\\
    & \text{where}\quad  f^*_t \in \argmin_{f \in \mathcal{F}_t}  \mathcal{L}(f) \quad \text{and} \quad f^*_{t+1} \in \argmin_{f \in \mathcal{F}_{t+1}}  \mathcal{L}(f).
\end{align}
Instead, we adopt a computationally constrained information-based perspective: we introduce random variables $X$ and $Z_{1:T}$, where $X$ represents the base information accessible to every model, and $(X,Z_{t})$ represents the richer information available to more $t$-th model, which is richer for more powerful ones. As such, we model the random variables $Z_{1:T}$ to be dependent. This modeling choice allows us to abstract away from the particular implementation of the architectures generating these hidden representations. However, for $Z_{1:T}$ to be interpreted as the hidden representations, the generating architectures must remain fixed during training, since the distribution of $Z_{1:T}$ is assumed to remain unchanged.

\newpage
\section{Limitations and Future Work}\label{app:limitation}

In this section we discuss the limitations of our work in more depth. First, while our empirical results suggest that the proposed losses are  effective on real datasets, the main contribution of this work remains theoretical. Our primary goal  is to provide a theoretical formulation of  joint  training and to propose a practical loss that is  Bayes-consistent with respect to the targeted cost-sensitive zero-one loss. In practice, computational constraints can make full joint training infeasible (for instance, because it requires fitting all classifiers into memory), and simpler training schemes could still be preferable. Nevertheless even in these settings, our analysis can serve as a principled starting point for developing more scalable  methods.

A second limitation of our approach is its reliance on the hinge loss. While this loss can be optimized using standard gradient-based techniques, it is less commonly used than smooth surrogate losses such as the logistic loss or the exponential loss. The hinge loss exhibits different optimization properties, such as non-differentiability at zero and flat regions, which can make optimization more challenging in practice. 
We use the hinge loss because it yields a sharp linear transfer rate from the excess surrogate risk to the excess target risk, as our calibration bounds establish with Theorems~\ref{thm:calibration} and~\ref{thm:calibration_casc}. A linear transfer rate is the strongest transfer rate one can hope for~\citep{Frongillo2021}. The most natural family of functions that would preserve this behavior is the class of polyhedral losses, which are piecewise-linear convex losses, of which the multiclass hinge loss is one instance. These losses are known to admit linear transfer rates in standard classification settings~\citep{Frongillo2021, Finocchiaro2024}. We conjecture that our results extend to this broader family.  

Smooth surrogates could also be used and might make optimization easier in practice. The tradeoff is that they generally yield weaker guarantees. It is known that for sufficiently non-polyhedral surrogates  the square-root transfer rate is unavoidable~\citep{Frongillo2021}, so replacing our hinge version with a smooth loss would in general forfeit the linear rate our analysis provides. 
Difficulties in developing a Bayes-consistent cross-entropy surrogate loss for our setting have been discussed by~\citet{regol2025acquisition}, but further work in this direction could improve the practical applicability of the proposed surrogate losses.

Finally, our formulation relies on several assumptions that can be restrictive in some settings. First, we assume that the costs can be known before evaluating the classifiers. This is incompatible with autoregressive applications where the output length can vary, which would generally result in a varying cost that we cannot know in advance. Second, our analysis relies on defining a sequence of random variables $Z_t$ that are progressively made accessible to the classifiers, which explains why costlier classifiers can achieve higher predictive performance.  In the main text, we note that the $Z_t$ can represent hidden representations computed by progressively larger-capacity models. 
This abstraction enables our statistical analysis, but it assumes that the distribution of those $Z_t$ remains fixed during training. As a result, our framework directly applies to adaptive inference settings in which backbone networks generating the hidden representations are fixed during optimization, and only the intermediate classifiers are updated. It does not directly extend to settings where the backbone networks are also fine-tuned, since in that case the $Z_t$ could no longer correspond to the hidden representations, as their distribution would evolve during training. This is a limitation of our current setup, as it does not cover the general adaptive inference setting. Extending the framework to this more general setting is an interesting direction for future work.

\end{document}